\documentclass{article}

     \usepackage[preprint,nonatbib]{neurips_2020}

\usepackage[utf8]{inputenc} %
\usepackage[T1]{fontenc}    %
\usepackage{hyperref}       %
\usepackage{url}            %
\usepackage{booktabs}       %
\usepackage{amsfonts}       %
\usepackage{nicefrac}       %
\usepackage{microtype}      %

\usepackage[vlined,ruled,noend]{algorithm2e}
\usepackage{amsmath}
\usepackage{amssymb}
\usepackage{amsthm}
\usepackage{graphicx}
\usepackage{subcaption}
\usepackage{mathtools}

\newtheorem{theorem}{Theorem}
\newcommand{\T}{\mathrm{T}}
\newcommand{\F}{\mathrm{F}}

\title{Logical Neural Networks}

\author{%
    Ryan Riegel \\
    IBM Research - Watson
    \And
    Alexander Gray \\
    IBM Research - Watson
    \And
    Francois Luus \\
    IBM Research - Africa
    \AND
    Naweed Khan \\
    IBM Research - Africa
    \And
    Ndivhuwo Makondo \\
    IBM Research - Africa
    \And
    Ismail Yunus Akhalwaya \\
    IBM Research - Africa
    \AND
    Haifeng Qian \\
    IBM Research - Watson
    \And
    Ronald Fagin \\
    IBM Research - Almaden
    \And
    Francisco Barahona \\
    IBM Research - Watson
    \AND
    Udit Sharma \\
    IBM Research - India
    \And
    Shajith Ikbal \\
    IBM Research - India
    \And
    Hima Karanam \\
    IBM Research - India
    \AND
    Sumit Neelam \\
    IBM Research - India
    \And
    Ankita Likhyani \\
    IBM Research - India
    \And
    Santosh Srivastava \\
    IBM Research - India
}

\begin{document}

\maketitle

\begin{abstract}
We propose a novel framework seamlessly providing key properties of both neural nets (learning) and symbolic logic (knowledge and reasoning).  Every neuron has a meaning as a component of a
formula in a weighted real-valued logic, yielding a highly intepretable disentangled representation.  Inference is omnidirectional rather than focused on predefined target variables, and
corresponds to logical reasoning, including classical first-order logic theorem proving as a special case.  The model is end-to-end differentiable, and learning minimizes a novel loss
function capturing logical contradiction, yielding resilience to inconsistent knowledge.  It also enables the open-world assumption by maintaining bounds on truth values which can have probabilistic semantics,
yielding resilience to incomplete knowledge.
\end{abstract}

\section{Introduction and related work}

We present \emph{Logical Neural Networks} (LNNs), a neuro-symbolic framework designed to simultaneously provide key properties of both neural nets (NNs) {\em(learning)} and symbolic logic {\em(knowledge and reasoning)} -- toward direct interpretability, utilization of rich domain knowledge realistically, and the general problem-solving ability of a full theorem prover. The central idea is to create a 1-to-1 correspondence between neurons and the elements of logical formulae, 
using the observation that the weights of neurons can be constrained to act as, e.g.~AND or OR gates.
While the view of neurons as logical gates was inherent in the seminal \cite{McCullochPitts43}, it appears to have remained relatively unexploited since; an exception is the idea of converting logical statements into NN forms \cite{pinkas1994translate,garcez1999connectionist}; the best known of these is \cite{towell1994knowledge}, where the final model's neurons do not necessarily retain logical gate behaviors.  In LNNs, no conversions are needed because they are identical.
Inputs include a propositional or first-order logic (FOL) knowledge base (KB), including the usual training data (feature-value pairs) as a special case, and which variables should be predicted from which.

\textbf{Per-neuron interoperability, via full logical expressivity.}  Many approaches are based on Markov random fields (MRFs), such as Markov logic networks
\cite{richardson2006markov}, probabilistic soft logic \cite{bach2017hinge}, and those based on ILP/SLPs e.g.~\cite{cohen2016tensorlog}, where each logical clause has a weight; the clauses are atomic i.e.~ their internal logical structure is not represented.  
Obtaining probabilities from them requires an unwieldy satisfiability problem to be solved, e.g.~via MCMC in \cite{richardson2006markov}.  LNN inference is deterministic/repeatable and provably convergent in finite steps.  
In LNNs, every neuron represents an element in a clause, and is either a concept (e.g.~"cat") or a logical connective (e.g.~AND, OR), with weights on the connecting edges.  Thus each neuron represents 1) a \emph{meaning}, raising the level of interpretability versus previous approaches, 2) a way to identify the importance of relationships
between variables and 3) more parameters, defining a richer model space for potentially more accurate prediction.  The network structure is thus \emph{compositional} and \emph{modular}, e.g.~able to represent that one clause may be a sub-clause of another.  The representation is \emph{disentangled}, versus approaches such as \cite{serafini2016logic,rocktaschel2017end} that use a vector representation, sacrificing interpretability of the network.  Where many approaches only allow the representational power of propositional logic or Horn clauses, LNNs allow full function-free \emph{first-order logic} with \emph{real values} $0 \leq x \leq 1$, and classical 0/1 logic as a special case.

\textbf{Tolerance to incomplete knowledge, via truth bounds.}  The line of approach embodied by MRFs make a closed-world assumption, i.e.~that if a statement doesn't appear in the KB, it is
false.  LNN does not require complete specification of all variables' exact degree of truth, more generally maintaining upper and lower bounds for each variable -- allowing the
\emph{open-world assumption} that complete knowledge of the world is not realistic in general.  Bounds also contain more interpretable information than single values, and we show that they can represent \emph{probabilistic semantics}.

\textbf{Many-task generality, via omnidirectional inference.}  LNN neurons express bidirectional relationships with each  neighbor, allowing inference in any direction.   This allows task generality versus typical single-task NNs, and allows full-fledged \emph{theorem proving}.  MRF approaches that hide the internal logical
structure of clauses cannot draw the same conclusions that a theorem prover can.  Many/most neuro-symbolic approaches, e.g.~those based on embeddings, are arguably only ``logic-like'' and typically do not
demonstrate reliably precise  deduction. We show evidence of such capability on sample tasks used to test state-of-the-art theorem provers.

\section{Overview}

A logical neural network (LNN) is a form of recurrent neural network with a 1-to-1 correspondence to a set of logical formulae in any of various systems of \emph{weighted, real-valued logic}, in which evaluation performs logical inference.
Key innovations that set LNNs aside from other neural networks are
    1) neural activation functions \textbf{constrained} to implement the truth functions of the logical operations they represent, i.e.~\(\land, \lor, \neg, \rightarrow\), and, in FOL, \(\forall\) and \(\exists\),
    2) results expressed in terms of \textbf{bounds} on truth values so as to distinguish known, approximately known, unknown, and contradictory states,
    and 3) \textbf{bidirectional inference} permitting, e.g., \(x \rightarrow y\) to be evaluated as usual in addition to being able to prove \(y\) given \(x\) or, just as well, \(\neg x\) given \(\neg y\).
The nature of the modeled system of logic depends on the family of activation functions chosen for the network's neurons, which implement the logic's various atoms
and operations.
In particular, it is possible to constrain the network to behave exactly classically when provided classical input.
Computation is characterized by tightening \emph{bounds} on truth values at neurons pertaining to subformulae in \emph{upward} and \emph{downward passes} over the represented formulae's syntax trees.
Bounds tightening is monotonic; accordingly, computation cannot oscillate and necessarily converges for propositional logic.
Because of the network's modular construction, it is possible to partition and/or compose networks, inject formulae serving as logical constraints or queries, and control which parts of the network (or individual neurons) are trained or evaluated.

Inputs are initial truth value bounds for each of the neurons in the network; in particular, neurons pertaining to predicate atoms may be populated with truth values taken from KB data. Additional inputs may take the form of injected formulae representing a query or specific inference problem.
Outputs are typically the final computed truth value bounds at one or more neurons pertaining to specific atoms or formulae of interest.
In other problem contexts, the outputs of interest may instead be the neural parameters themselves --- serving as a form of inductive logic programming (ILP) --- after learning with a given loss function and input training data set.

\section{Model structure}
\label{sec:model}

\begin{figure}[tbp]%
    \centering%
    \begin{subfigure}[b]{.4\textwidth}%
    \centering%
    \captionsetup{width=.9\textwidth}%
    \includegraphics[width=.9\textwidth,trim={11cm 6cm 11cm 5cm},clip]{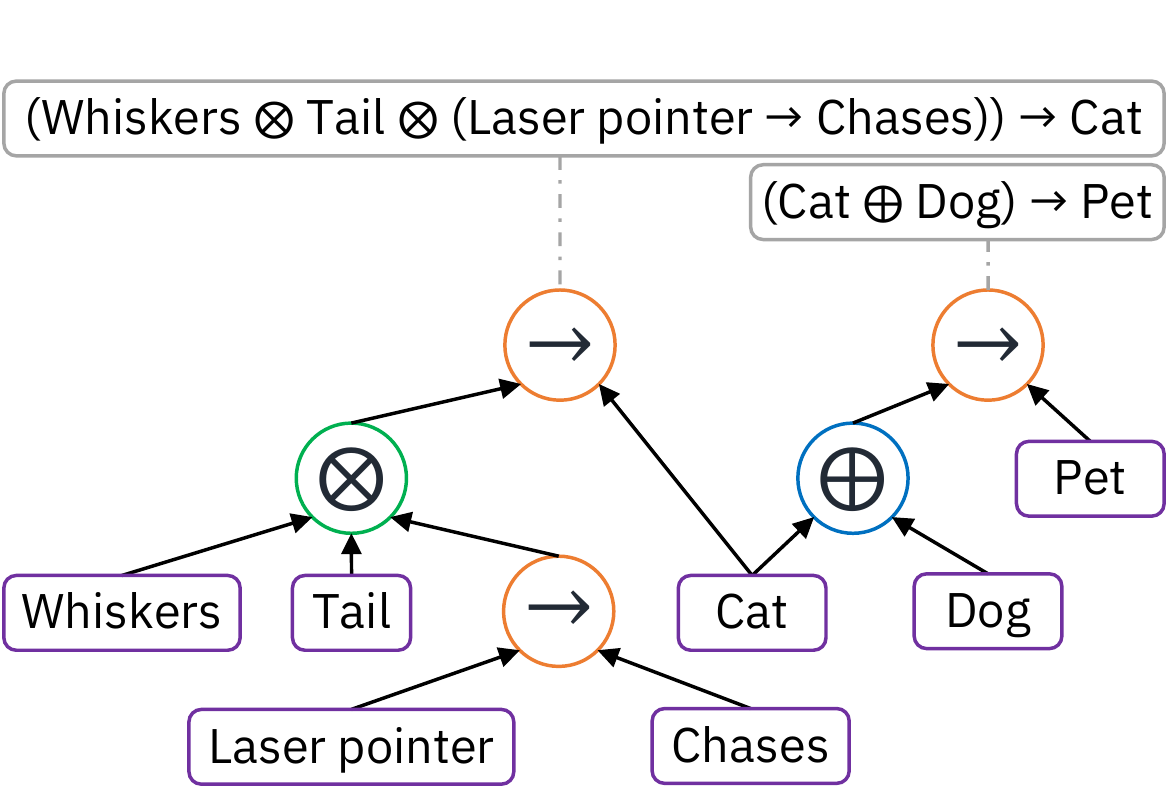}%
    \caption{%
      The LNN graph structure reflects the formulae it represents.
    }%
    \label{fig:structure}%
    \end{subfigure}%
    \begin{subfigure}[b]{.25\textwidth}%
    \centering%
    \captionsetup{width=.9\textwidth}%
    \includegraphics[width=.9\textwidth,trim={12.1cm 4.5cm 11.5cm 4.8cm},clip]{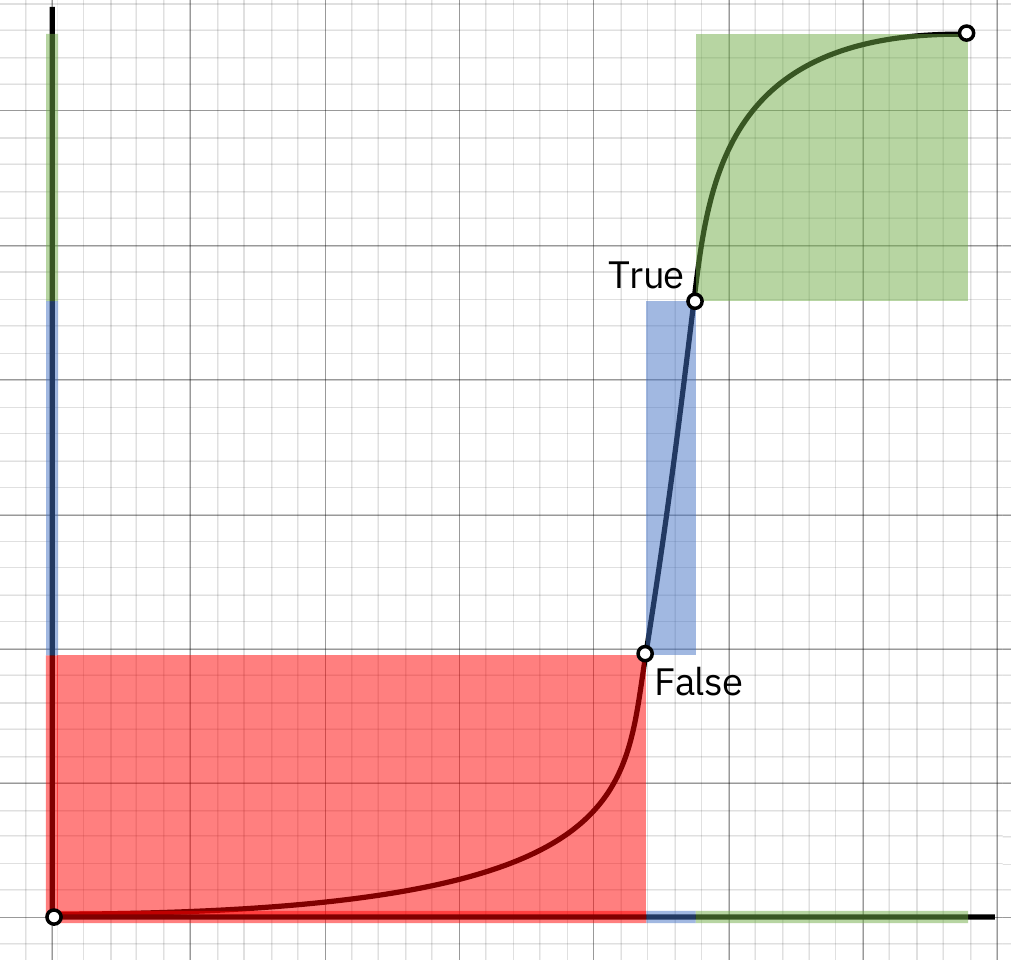}%
    \caption{%
      A logistic activation function 
    }%
    \label{fig:logistic}%
    \end{subfigure}%
    \begin{subfigure}[b]{.25\textwidth}%
    \centering%
    \captionsetup{width=.9\textwidth}%
    \includegraphics[width=.9\textwidth,trim={12.1cm 4.5cm 11.5cm 4.8cm},clip]{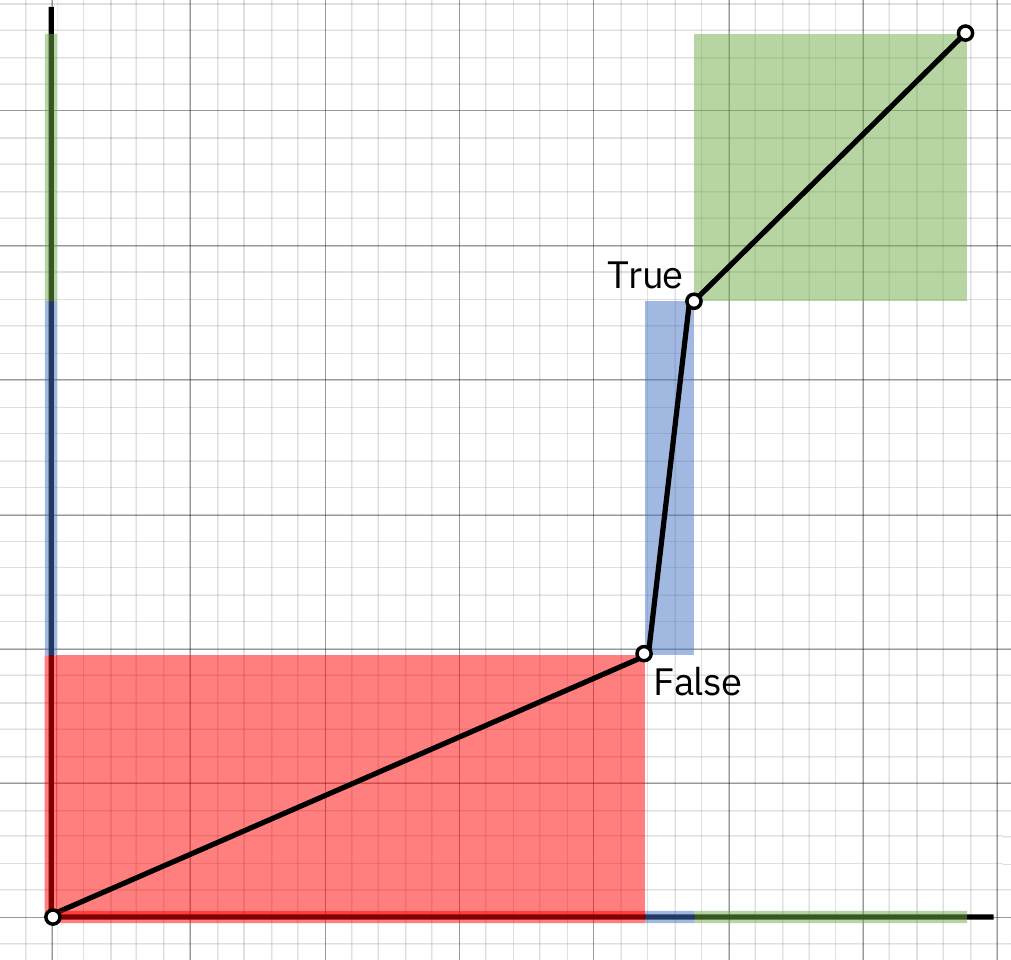}%
    \caption{%
      A linearly interpolated activation function
    }%
    \label{fig:tailored}%
    \end{subfigure}%
    \caption{Neurons (\ref{fig:structure}) with alternative activation functions (\ref{fig:logistic},~\ref{fig:tailored}) configured to match the truth functions of their corresponding operations, with established regions of unambiguously \texttt{True}, unambiguously \texttt{False}, and intermediate truth.
    }%
\end{figure}

In general, LNNs are described in terms of FOL, but it is useful to discuss LNNs restricted to the scope of propositional logic.\footnote{%
  First-order LNN expands to include neurons for predicate and quantifier symbols, with each formula grounding treated as a proposition.
  Each neuron keeps a table that maps a set of $n$-dimensional grounding tuples to truth value bounds, where $n$ is the number of unique variables in the underlying subformula.
  Quantifiers~\(\forall\) and \(\exists\) are modeled as pass-through nodes aggregating bounds over one of the $n$ dimensions via min and max, respectively.
  Inference proceeds as described for propositional LNN in Section~\ref{sec:inference}, with each grounding treated independently and special handling for \(\forall\) and \(\exists\).
  This method is similar to approaches that reduce inference in classical FOL to propositional logic \cite{modernAI}.
  Additional details are given in section \ref{sec:fol_lnn} of the supplementary material.
}
Structurally, an LNN is a graph made up of the syntax trees of all represented formulae connected to each other via neurons added for each proposition.
Specifically, as shown in Figure~\ref{fig:structure}, there exists one neuron for each logical operation occurring in each formula and, in addition, one neuron for each unique proposition occurring in any formula.
All neurons return pairs of values in the range \([0,1]\) representing lower and upper bounds on the truth values of their corresponding subformulae and propositions.
To aid interpretability of bounds, we define a threshold of truth~\(\frac{1}{2} < \alpha \leq 1\) such that a continuous truth value is considered \texttt{True} if it is greater than \(\alpha\) and \texttt{False} if it is less than \(1 - \alpha\).
Bound values identify one of four primary states that a neuron can be in, whereas secondary states offer a more-true-than-not or more-false-than-not interpretation.

\begin{table}[h]%
    \label{tab:bound-states}%
    \centering%
    \begin{small}%
    \caption{Primary truth value bound states}%
    \begin{tabular}{ c c c c c } %
    \toprule%
    Bounds & Unknown & True & False & Contradiction \\
    \midrule%
    Upper & \([ \alpha, 1 ]\) & \([ \alpha, 1 ]\) & \([ 0, 1 - \alpha ]\) &  Lower > Upper \\
    Lower & \([ 0, 1 - \alpha ]\) & \([ \alpha, 1 ]\) & \([ 0, 1 - \alpha ]\) & ~ \\

    \bottomrule%
    \end{tabular}%
    \end{small}%
\end{table}

Neurons corresponding to logical connectives accept as input the output of neurons corresponding to their operands and have activation functions configured to match the connectives' truth functions.
Neurons corresponding to propositions accept as input the output of neurons established as \emph{proofs} of bounds on the propositions' truth values and have activation functions configured to aggregate the tightest such bounds.
Proofs for propositions may be established explicitly, e.g.~as the heads of Horn clauses, though Section~\ref{sec:inference} shows how bidirectional inference permits every occurrence of each proposition in each formula to be used as a potential proof.
Negation is modeled as a pass-through node with no parameters, canonically performing \(\neg x = 1 - x\).

\subsection{Activation functions for connectives}
\label{sec:model-connectives}

Many candidate neural activation functions can accommodate the classical truth functions of logical connectives, each varying in how it handles inputs strictly between 0 and 1.
For instance, \(\min\{x, y\}\) is a suitable activation function for real-valued conjunction~\(x \otimes y\), but so are \(x \cdot y\) and \(\max\{0, x + y - 1\}\).
The choice of activation function corresponds to the implemented real-valued logic; G\"odel, product, and \L ukasiewicz logic are common examples. %

In addition to matching their corresponding connectives' truth functions, as described in Section~\ref{sec:learning}, LNN requires monotonicity.  Specifically, the activation functions for conjunction and disjunction must increase monotonically with respect to each operand, and the activation function for implication must decrease monotonically with respect to the antecedent and increase monotonically with respect to the consequent.
In addition, it is useful though not required for \(\otimes\) and real-valued disjunction~\(x \oplus y\) to be related via the De Morgan laws, for both operations to be commutative and associative, and for real-valued implication~\(x \rightarrow y\) to be the \emph{residuum} of \(\otimes\), or specifically \((x \rightarrow y) = \max \{z \mid y \geq (x \otimes z)\}\).
These properties do not guarantee that \((x \rightarrow y) = (\neg x \oplus y)\), that \(((x \rightarrow 0) \rightarrow 0) = x\), or that \((x \otimes x) = (x \oplus x) = x\), though they may individually hold for certain activation functions.

Gradient-based neural learning also requires differentiable parameters that can be tuned to improve model performance.
We introduce the concept of \emph{importance weighting}, whereby neural inputs with larger weight have more influence on neural output.
This can take many forms, as thoroughly explored in \cite{fagin2000formula}, though this paper focuses on an easily computable weighting scheme based on nonlinear functions applied to dot products.

\subsection{Weighted nonlinear logic}
\label{sec:weighted}

We introduce weighted generalizations of the standard real-valued logics in section \ref{sup:weighted_real_value_logics} of the supplementary material.
This completes the mapping of NNs to weighted real-valued logics.
Here is shown a weighted generalization of \L ukasiewicz-like logics.
The \(n\)-ary \emph{weighted nonlinear conjunctions}, used for logical AND, are given
\begin{align}
    \textstyle {}^{\beta}(\bigotimes_{i \in I} x_i^{\otimes w_i}) & \textstyle = f(\beta - \sum_{i \in I} w_i (1 - x_i))
\intertext{%
for \(f : \mathbb{R} \to [0,1]\) with \(f(1 - x) = 1 - f(x)\), input set~\(I\), bias term~\(\beta \geq 0\), weights~\(w_i \geq 0\), and inputs \(x_i \in [0,1]\).
The \(n\)-ary \emph{weighted nonlinear disjunctions}, used for logical OR, are then
}
    \textstyle {}^{\beta}(\bigoplus_{i \in I} x_i^{\oplus w_i}) & \textstyle = f(1 - \beta + \sum_{i \in I} w_i x_i).
\end{align}

Observe that \(\beta\) and the various \(w_i\) establish a hyperplane with respect to the inputs~\(x_i\), though clamped to the \([0,1]\) range.
For \(f(x) = \max\{0, \min\{1, x\}\}\), the resulting activation functions are similar to the rectified linear unit (ReLU) from neural network literature and to the \L ukasiewicz t- and s-norms.
Alternate choices of \(f\) include the logistic function as shown in Figure~\ref{fig:logistic} and a linearly interpolated \emph{tailored activation function} as shown in Figure~\ref{fig:tailored}, as further explored in Section~\ref{sec:learning}.

Bias term~\(\beta\) permits classically equivalent formulae~\(x \rightarrow y\), \(\neg y \rightarrow \neg x\), and \(\neg x \oplus y\) to be made equivalent in weighted nonlinear logic by adjusting \(\beta\).
The \emph{weighted nonlinear residuum}, used for logical implication,
is given by the residuum\footnote{%
  This solution assumes \(f(x) = \max\{0, \min\{1, x\}\}\), but for simplicity may be used with any \(f\).
}
of \(\otimes\),
\begin{equation}
    {}^{\beta}(x^{\otimes w_x} \rightarrow y^{\oplus w_y}) = f(1 - \beta + w_x (1 - x) + w_y y).
\end{equation}
Note the use of \(\otimes\) in the antecedent weight but \(\oplus\) in the consequent weight, which is meant to indicate the antecedent has AND-like weighting (scaling its distance from 1) while the consequent has OR-like weighting (scaling its distance from 0).
This residuum is most disjunction-like when \(\beta = 1\), most \((x \rightarrow y)\)-like when \(\beta = w_y\), and most \((\neg y \rightarrow \neg x)\)-like when \(\beta = w_x\).

\subsection{Activation functions for atoms}

Neurons pertaining to atoms require activation functions that aggregate truth values found for the various computations identified as proofs of the atom.
For example, each of \((x_1 \otimes x_2 \otimes x_3) \rightarrow y\), \((x_1 \otimes x_4) \rightarrow y\), and \((x_2 \otimes x_4) \rightarrow \neg y\) may constitute proofs (or disproofs) of \(y\).
A typical means of aggregating proven truth values is to return the maximum proven lower bound and minimum proven upper bound.
On the other hand, it may be desirable to employ importance weighting via weighted nonlinear logic in aggregation as well, substituting OR for max in the lower bound computation and AND for min in the upper bound.
A natural choice of weights for such aggregations is the weights of the atoms as they occur in the formulae serving as their proofs.
For example, if \(x_1^{\otimes 3} \rightarrow y^{\oplus 2}\) and \(x_2^{\otimes 1} \rightarrow y^{\oplus .5}\) are proofs of \(y\), then aggregation would use weights 2 and .5, respectively.

\section{Inference}
\label{sec:inference}

Inference refers to the entire process of computing truth value bounds for (sub)formulae and atoms based on initial knowledge, ultimately resulting in predictions made at neurons pertaining to queried formulae or other results of interest.
LNN achieves this with multiple passes over the represented formulae, propagating tightened truth value bounds from neuron to neuron until computation necessarily converges. %
The previous section already suggests the \emph{upward pass} of inference, whereby formulae compute their truth value bounds based on bounds available for their subformulae.
This section further describes the \emph{downward pass}, which permits prior belief in the truth or falsity of formulae to inform truth value bound for propositions or predicates used in said formulae.

\subsection{Bidirectional inference}
\label{sec:Bidirectional-inference}

In addition to the usual evaluation of connectives, LNN infers truth values bounds for each of a connective's inputs based on bounds on its output and other inputs.
Depending on the type of connective involved, such computations correspond to the familiar inference rules of classical logic: \\
\begin{minipage}{.45\textwidth}%
\begin{alignat*}{4}
         x,{} &&       x \rightarrow y &   && \vdash{} &      y \tag{\emph{modus ponens}} \\
    \neg y,{} &&       x \rightarrow y &   && \vdash{} & \neg x \tag{\emph{modus tollens}}
\end{alignat*}%
\end{minipage}%
\begin{minipage}{.55\textwidth}%
\begin{alignat*}{4}
         x,{} &&       \neg (x \land y & ) && \vdash{} & \neg y \tag{conjunctive syllogism} \\
    \neg x,{} &&              x \lor y &   && \vdash{} &      y \tag{disjunctive syllogism}
\end{alignat*}%
\end{minipage}%

The precise nature of these computations depends on the selected family of activation functions.
For example, as noted in Section~\ref{sec:model-connectives}, if implication is defined as the residuum, then \emph{modus ponens} is performed via the logic's t-norm, i.e.~AND.
The remaining inference rules follow a similar pattern as prescribed by the functional or logical inverses of their upward computations.

The application of these inference rules immediately suggests a means of generating proofs for atoms based on the formulae they occur in.
As further discussed in Section~\ref{subsec:recurrent_algorithm}, given
truth value bounds for a formula,
it is possible to apply
inference rules to obtain truth value bounds for each of its leaves.

\subsection{Inference rules in weighted nonlinear logic}
\label{sec:inference_rules}

In the following, \(L\) and \(U\) denote lower and upper bounds found for neurons corresponding to the formulae indicated in their subscripts, e.g.~\(L_{x \rightarrow y}\) is the lower-bound truth value for the formula~\(x \rightarrow y\) as a whole while \(U_x\) is the upper bound for just \(x\).
The bounds computations for \(\neg\) are trivial:
\begin{align*}
    L_{\neg x} & \geq \neg U_x = 1 - U_x, & L_x & \geq \neg U_{\neg x} = 1 - U_{\neg x}, \\
    U_{\neg x} & \leq \neg L_x = 1 - L_x, & U_x & \leq \neg L_{\neg x} = 1 - L_{\neg x}.
\end{align*}
The use of inequalities here acknowledges that tighter bounds for each value may be available from other sources.
For instance, both \(y\) and \(x \rightarrow \neg y\) can yield \(L_{\neg y}\); the tighter of the two would apply.

Observing that, in weighted nonlinear logic, \({}^{\beta}(x^{\otimes w_x} \rightarrow y^{\oplus w_y}) = {}^{\beta}((1 - x)^{\oplus w_x} \oplus y^{\oplus w_y})\) and \({}^{\beta}(\bigotimes_{i \in I} x_i^{\otimes w_i}) = 1 - {}^{\beta}(\bigoplus_{i \in I} (1 - x_i)^{\oplus w_i})\), it is only necessary to define one set of inference rule computations to cover all connectives.
The upward bounds computations for \({}^{\beta}(\bigoplus_{i \in I} x_i^{\oplus w_i})\) are
\begin{align*}
    L_{\bigoplus_i x_i} & \textstyle \geq {}^{\beta}(\bigoplus_{i \in I} L_{x_i}^{\oplus w_i}), & U_{\bigoplus_i x_i} & \textstyle \leq {}^{\beta}(\bigoplus_{i \in I} U_{x_i}^{\oplus w_i})
\end{align*}
while the downward bounds computations for disjunction are
\begin{alignat}{3}
    L_{x_i} & \textstyle \geq {}^{\beta/w_i}((\bigotimes_{j \neq i} (1 - U_{x_j})^{\otimes w_j/w_i}) \otimes L_{\bigoplus_i x_i}^{\otimes 1/w_i}) & \qquad\text{if } L_{\bigoplus_i x_i} & > 1 - \alpha, & \quad\text{else } 0 \label{eq:or_downward_0}\\
    U_{x_i} & \textstyle \leq {}^{\beta/w_i}((\bigotimes_{j \neq i} (1 - L_{x_j})^{\otimes w_j/w_i}) \otimes U_{\bigoplus_i x_i}^{\otimes 1/w_i}) & \qquad\text{if } U_{\bigoplus_i x_i} & < \alpha, & \quad\text{else } 1 \label{eq:or_downward_1}
\end{alignat}
where \(\alpha\) is threshold determined by \(f\) to address potential divergent behavior at \(L_{\bigoplus_i x_i} \leq 1 - \alpha\) and \(U_{\bigoplus_i x_i} \geq \alpha\).
To understand why this occurs, observe that, for \(f(x) = \max\{0, \min\{1, x\}\}\), i.e.~the ReLU case, \(x \oplus y\) can return 1 for many different values of \(x\) and \(y\); specifically, whenever \(w_x x + w_y y \geq \beta\).
Accordingly, if \(U_{\bigoplus_i x_i} = 1\), we cannot infer an upper bound for any \(x_i\).

\subsection{The Upward--Downward algorithm}
\label{subsec:recurrent_algorithm}

\begin{algorithm*}[tbp]
\caption{Upward pass to infer formula truth value bounds from subformulae bounds \label{algo-upwardpass}}%
\small%
\SetKwComment{tcp}{\#{ }}{}%
\SetAlgoLined\SetArgSty{}%
\SetFuncArgSty{}
\SetKwProg{Fn}{function}{\string:}{}%
\SetKwFunction{FRecurs}{upwardPass}%
\Fn{\FRecurs{formula $z$}}{
\For(\tcp*[f]{propagate bounds upward from leaves}){operand $x_i$ of $z$, $i \in I$}{
\FuncSty{upwardPass}$(x_i)$
}
\If(\tcp*[f]{negation}){$z = \neg x$}{
\FuncSty{aggregate}$(z, (1 - U_x,\;1 - L_x))$
}
\ElseIf(\tcp*[f]{multi-input disjunction}){$z={}^{\beta}(\bigoplus_{i \in I} x_i^{\oplus w_i})$}{
\FuncSty{aggregate}$(z, ({}^{\beta}(\bigoplus_{i \in I} L_{x_i}^{\oplus w_i}),\;{}^{\beta}(\bigoplus_{i \in I} U_{x_i}^{\oplus w_i}))$
}
\tcp{Other operations are handled as combinations of the above}
}
\SetKwFunction{FRecurs}{aggregate}%
\Fn{\FRecurs{formula $z$, $(L_z', U_z')$}}{
$(L_z, U_z) := (\max\{L_z,L_z'\},\; \min\{U_z,U_z'\})$\tcp*[f]{tighten existing bounds}
}
\end{algorithm*}

\begin{algorithm*}[tbp]
\caption{Downward pass to infer subformula truth value bounds from formula bounds \label{algo-downwardpass}}%
\small%
\SetKwComment{tcp}{\#{ }}{}%
\SetAlgoLined\SetArgSty{}%
\SetFuncArgSty{}
\SetKwProg{Fn}{function}{\string:}{}%
\SetKwFunction{FRecurs}{downwardPass}%
\Fn{\FRecurs{formula $z$}}{
\For{operand $x_j$ of $z$, $j \in I$}{
\If(\tcp*[f]{negation}){$z=\neg x$}{
\FuncSty{aggregate}$(x, (1 - U_z,\;1 - L_z))$
}
\ElseIf(\tcp*[f]{multi-input disjunction}){$z={}^{\beta}(\bigoplus_{i \in I} x_i^{\oplus w_i})$}{
$\begin{array}{r@{}lll}
L_{x_j}' & {}:= {}^{\beta/w_j}((\bigotimes_{i \neq j} (1 - U_{x_i})^{\otimes w_i/w_j}) \otimes L_{\bigoplus_i x_i}^{\otimes 1/w_j}) & \text{if } L_{\bigoplus_i x_i} > 1 - \alpha, & \text{else } 0 \\
U_{x_j}' & {}:= {}^{\beta/w_j}((\bigotimes_{i \neq j} (1 - L_{x_i})^{\otimes w_i/w_j}) \otimes U_{\bigoplus_i x_i}^{\otimes 1/w_j}) & \text{if } U_{\bigoplus_i x_i} < \alpha, & \text{else } 1
\end{array}$
\FuncSty{aggregate}$(x_j, (L_{x_j}',\;U_{x_j}'))$
}
\FuncSty{downwardPass}$(x_j)$\tcp*[f]{propagate bounds downward to leaves}
}
}
\end{algorithm*}

\begin{algorithm*}[tbp]
\caption{Recurrent inference procedure with recursive directional graph traversal \label{algo-inference}}%
\small%
\SetKwComment{tcp}{\#{ }}{}%
\SetAlgoLined\SetArgSty{}%
\SetFuncArgSty{}
\SetKwProg{Fn}{function}{\string:}{}%
\SetKwFunction{FRecurs}{inference}%
\Fn{\FRecurs{formulae $\mathbf{z}$}}{
\While(\tcp*[f]{loop until convergence}){$\sum(|\delta L_z| + |\delta U_z|) > \epsilon$}{
\For(\tcp*[f]{visit all formula roots in sequence}){$r\in$ \FuncSty{roots}($\mathbf{z}$)}{
\FuncSty{upwardPass}$(r)$\tcp*[f]{leaves-to-root traversal}

\FuncSty{downwardPass}$(r)$\tcp*[f]{root-to-leaves traversal}
}
}
}
\end{algorithm*}

Inference propagates truth value bounds from neuron to neuron along in alternating upwards and downwards passes over the syntax trees of the represented formulae.
The upward pass, shown in Algorithm~\ref{algo-upwardpass}, uses truth value bounds available at atoms to compute bounds at each subformula according to the normal evaluation of connectives based on their operands.
The downward pass, shown in Algorithm~\ref{algo-downwardpass}, uses truth value bounds known for formulae and previously computed at other subformulae to tighten bounds at each subformula (and ultimately each atom) according to the inference rules given above.
This process repeats, as shown in Algorithm~\ref{algo-inference}, until convergence, as proved in section \ref{sec:recurrent_algorithm} of the supplementary material:

\begin{theorem}
    Given monotonic \(\neg\), \(\oplus\), and \(f\), Algorithm~\ref{algo-inference} converges to within \(\epsilon\) in finite time.
\end{theorem}

\section{LNN bounds as probability bounds}
\label{sec:probability_bounds}

This section presents a variant of LNN where lower and upper bound truth values at each subformula serve as bounds on the probability that the subformula is \texttt{True} in classical logic.
This is achieved by using different activation functions for lower and upper bound computations.
Again observing that implication and conjunction may be defined in terms of negation and disjunction, these are
\begin{align*}
    L_{\bigoplus_i x_i} & \textstyle \geq \max_{i \in I} L_{x_i} & L_{x_i} & \textstyle \geq \max\{0, L_{\bigoplus_i x_i} - \sum_{j \neq i} U_{x_j}\}, \\
    U_{\bigoplus_i x_i} & \textstyle \leq \min\{1, \sum_{i \in I} U_{x_i}\} & U_{x_i} & \leq U_{\bigoplus_i x_i}
\end{align*}
with negation unchanged from above.
Let $A$ be the set of atomic formulae and let $g : A \rightarrow \{\T, \F\}$ be an interpretation.
Further define $g(\sigma)$ for any formula $\sigma$ on $A$ to be the truth value of $\sigma$ under the truth-value assignments by $g$ to the atomic formulae.
Let $\Lambda$ be the set of all interpretations.

A sentence is an expression of the form $(\sigma, l, u)$ where $\sigma$ is a formula and $l, u \in [0,1]$.
A theory is a set of sentences $\Gamma = \{(\sigma_1, l_1, u_1), \cdots, (\sigma_k, l_k, u_k)\}$.
Define $S_\sigma \triangleq \{g \mid g(\sigma) = \T\}$ for any formula $\sigma$.
A model is a probability function $p(\cdot)$ over $\Lambda$.
We say that $p(\cdot)$ is a model of $\Gamma$ and write $p(\cdot) \models \Gamma$ if and only if $l_i \leq p(S_{\sigma_i}) \leq u_i$ for $i = 1, \cdots, k$.
Let $P_\Gamma$ denote the set of all models of $\Gamma$.

Initial knowledge is specified by a set of formulas $V_0$ and two functions $L_0 : V_0 \rightarrow [0, 1]$ and $U_0 : V_0 \rightarrow [0, 1]$.
We may then state the following theorem, proved in section~\ref{sec:th2_proof} of the supplementary material:

\begin{theorem}\label{th:1}
  Let $L_\sigma$ and $U_\sigma$ denote the lower and upper bounds computed by LNN for formula $\sigma$.
  Define $\Gamma_0 = \{(v, L_0(v), U_0(v))\mid v \in V_0\}$.
  If $P_{\Gamma_0} \neq \emptyset$, the following inequalities hold:
\begin{align*}
  L_{\sigma} & \textstyle \leq \inf_{p \in P_{\Gamma_0}} p(S_{\sigma}) & %
  U_{\sigma} & \textstyle \geq \sup_{p \in P_{\Gamma_0}} p(S_{\sigma}) %
\end{align*}
\end{theorem}

\section{Learning}
\label{sec:learning}

A core strength of the LNN model is its differentiability, permitting the optimization via back-propagation of parameters including operand importance weights, formula truth value bounds, and/or the truth value bounds of atoms.
Loss functions for LNN may exploit its logical interpretability, in particular by penalizing contradiction, which can then be used to enforce even complicated logical requirements.
An important consideration, however, is whether it is desired to preserve neurons' fidelity to their corresponding logical connectives, especially when presented with classical inputs. %

Weighted nonlinear logic behaves classically for classical inputs
when optimized as
\begin{align}
    \textstyle \min_{B, W}\quad & \textstyle E(B, W) + \mathrlap{\sum_{k \in N} \max\{0, L_{B,W,k} - U_{B,W,k}\}} \notag \\
    \text{s.t.}\quad & \forall k \in N,\; i \in I_k, & \alpha \cdot w_{ik} - \beta_k + 1 & \geq \alpha, & w_{ik} & \geq 0 \label{eqn:or-1} \\
    & \forall k \in N,                      & \textstyle \sum_{i \in I_k} (1 - \alpha) \cdot w_{ik} - \beta_k + 1 & \leq 1 - \alpha, & \beta_k & \geq 0 \label{eqn:or-0}
\end{align}
for loss function~\(E\), bias vector~\(B\), weight matrix~\(W\), (disjunction) neuron index set~\(N\), and inferred lower and upper bounds~\(L_{B,W,k}\) and \(U_{B,W,k}\) at each neuron.
Intuitively, \eqref{eqn:or-1} requires disjunctions to return \texttt{True} if \emph{any} of their inputs are true, even if their other inputs are 0, i.e.~maximally false, while \eqref{eqn:or-0} requires them to return \texttt{False} if \emph{all} of their inputs are false.
Loss function~\(E\) often embodies typical NN learning objectives such as mean-square error; in addition, \emph{contradiction loss}~\(\sum_{k \in N} \max\{0, L_{B,W,k} - U_{B,W,k}\}\) penalizes the sum total contradiction observed in the system.

Given the above linear constraints, methods such as Frank--Wolfe \cite{frank1956algorithm,jaggi2013revisiting} may be used to optimize \(B\) and \(W\).
It is easy to see, however, that weights~\(w_{ik}\) cannot be made equal to 0, nor can constraints be relaxed to permit nonclassical behavior.
This may be corrected via the introduction of slack variables, though the following presents a means of sidestepping this issue while also improving gradients.

\subsection{Tailored activation functions}

For disjunction with \(\beta = 1\), the \emph{tailored activation function}~\(f_{\mathbf{w}}\), shown in Figure~\ref{fig:tailored}, is a linear interpolation between four critical points --- \((0, 0)\), \((x_{\F}, 1 - \alpha)\), \((x_{\T}, \alpha)\), and \((x_{\max}, 1)\), establishing regions of unambiguous \texttt{True}, intermediate, and \texttt{False} truth values, respectively --- given

\begin{equation}
\label{tailored:definition}
    f_{\mathbf{w}}(x) = \left\{\begin{array}{l@{\qquad}l}
        x \cdot (1 - \alpha) / x_{\F} & \text{if } 0 \leq x \leq x_{\F}, \\
        (x - x_{\F}) \cdot (2 \alpha - 1) / (x_{\T} - x_{\F}) + 1 - \alpha & \text{if } x_{\F} < x < x_{\T}, \\
        (x - x_{\T}) \cdot (1 - \alpha) / (x_{\max} - x_{\T}) + \alpha & \text{if } x_{\T} \leq x \leq x_{\max},
    \end{array}\right.
\end{equation}

\begin{align*}
    x_{\F} & \textstyle = \sum_{i \in I} w_i \cdot (1 - \alpha), & x_{\T} & \textstyle = w_{\max} \cdot \alpha, & x_{\max} & \textstyle = \sum_{i \in I} w_i.
\end{align*}

By construction, this guarantees classical inputs produce classical results without the need for constraints.
In addition, because \(x_{\T}\) is defined in terms of \(w_{\max}\),
weights may drop to 0 without significantly impacting \(f_{\mathbf{w}}\).
By the nature of monotonic linear interpolation, gradients are large and reliable everywhere.
Lastly, the tailored activation function establishes \(\alpha\) as a means of controlling the system's classicality, with smaller values being more classical.

\section{Empirical evaluation}
\textbf{Smokers and friends.}
LTN experiment $\mathcal{K}_{\text{exp2}}$ \cite{serafini2016logic} has plausible universally quantified axioms for a small universe (a-h) with initial facts for smokes $S(x)$, cancer $C(x)$, and friends $F(x,y)$ (open-world). We repeat the experiment including axioms induced by MLN~\cite{richardson2006markov} (total 8 axioms) on this data, and measure total network contradiction and LNN truth bounds for axioms in Table~\ref{tab:smokers1}. 
MLN axiom weight 6.88 for $\exists y F(x,y)$ states that a world with $n$ friendless people is $e^{6.88n}$ times less probable than a world where all have friends, other things being equal~\cite{richardson2006markov}, but LNN sets full truth as it does not conflict. MLN assigns high log-probability to the last axiom even though it produces contradictions, while LNN sufficiently relaxes the axiom to remove conflicts.
LTN assigns high satisfiability of 96 to $\neg S(x)\kern-0.2em\lor\kern-0.2em\neg F(x,y)\kern-0.2em\lor\kern-0.2em S(y)$, despite evidence against the axiom whereas LNN correctly adjusts bounds to remove contradiction. LNN also infers all logical consequences, in particular friendship symmetry that LTN is unable to produce.
Learning of neuron weights can possibly reduce bounds relaxation in $P_1^{8}$ when comparing to $P_2^{8}$ for the last two high-conflict axioms.
\begin{table}[tbp]
    \centering
    \begin{small}
    \caption{Learnt LNN neuron weights (from $P_1^{8}$) and axiom lower bounds (as \%) for LTN experiment $\mathcal{K}_{\text{exp2}}$ (universe: a-h) \cite{serafini2016logic}, compares LTN degree of satisfiability (as \% for 5 axioms) to LNN $P_1^{5}, P_2^{5}$ and repeats in $P_1^{8}, P_2^{8}$ for 8 axioms including 3 induced by MLN~\cite{richardson2006markov}, with corresponding MLN log-probability weights, followed by axiom-wise contradiction counts. Loss function $(1+\mathtt{contradiction})/(1+\mathtt{factalign}+\mathtt{tightbounds})$ (normalized) component values after training show complete removal of contradictions by relaxing facts and inferences.
    Gradient descent ($\alpha=1, \beta=1, w_{\max}=1$ with weight normalization and gradient-transparent clamping) adjusts operand weights ($P_1$), initial axiom/fact bounds ($P_1$, $P_2$). Every training epoch performs inference initialized with updated bounds until convergence after the parameter update.
    }
    \begin{tabular}{@{ }l@{}c@{ }@{ }c@{ }c@{ }@{ }c@{ }c@{ }@{ }c@{ }c@{ }}
        \toprule
        \multicolumn{2}{@{}r@{}}{}&\multicolumn{4}{@{ }c@{ }}{\bf{LNN} $[L, U]$} \textit{100 epochs, lr: 0.1$\rightarrow$0}\\
        \cmidrule{3-4}\cmidrule{5-6}
        \textbf{Smokers and friends} [$\mathcal{K}_{\text{exp2}}$ (a-h)] & \textbf{LTN} & ${P_1^{5}}$ & ${P_2^5}$ & ${P_1^8}$ & ${P_2^8}$ & \textbf{MLN} & $L\kern-0.3em>\kern-0.3emU$\\
        \midrule
        $\exists y F(x,y)$ & 100 & 100 & 100 & 100 & 100 & 6.88 & 0\\
        $\neg F(x,x)$ & 98 & [83, 98] & [83, 98] & [56, 98] & [80, 98] & 0.26 & 0\\
        $\neg F(x,y)^{0.96}\kern-0.2em\lor\kern-0.2em F(y,x)^1$ & 90 & [96, 97] & [97, 100] & [51, 95] & [82, 97] & - & 0\\
        $\neg S(x)^{0.98}\kern-0.2em\lor\kern-0.2em\neg F(x,y)^1\kern-0.2em\lor\kern-0.2em S(y)^{0.97}$ & 96 & [65, 100] & [65, 100] & [65, 100] & [66, 100] & 3.53 & 2\\
        $\neg S(x)^1\kern-0.2em\lor\kern-0.2em C(x)^{0.98}$ & 77 & [57, 100] & [58, 100] & [50, 100] & [60, 100] & -1.35 & 2\\\cmidrule{1-4}
        $\neg F(x,y)^{0.97}\kern-0.2em\lor\kern-0.2em \neg S(y)^1\kern-0.2em\lor\kern-0.2em F(y,x)^{0.96}\kern-0.2em\lor\kern-0.2em S(x)^1$ &  &  &  & 100 & 100 & 6.87 & 0\\
        $\neg F(w,x)^1\kern-0.2em\lor\kern-0.2em\neg F(w,y)^1\kern-0.2em\lor\kern-0.2em \neg F(z,x)^1\kern-0.2em\lor\kern-0.2em C(z)^{0.97}$ &  &  &  & [73, 100] & [70, 100] & 4.33 & 51\\
        $\neg F(w,x)^1\kern-0.2em\lor\kern-0.2em\neg F(y,w)^1\kern-0.2em\lor\kern-0.2em \neg F(z,y)^1\kern-0.2em\lor\kern-0.2em \neg S(y)^{0.99}$ &  &  &  & [80, 100] & [77, 100] & 9.68 & 66\\
        \midrule
        Contradiction (remaining) & & 0 & 0 & 0 & 0 & & \textit{121}\\
        Factual $E_i[|L_i'-L_i|+|U_i'-U_i|]$ (start: 0.64)& & 0.42 & 0.43 & 0.27 & 0.37 & & \\
        Bound tightness $E_i[\exp(L_i-U_i)]$ (start: 1) & & 0.88 & 0.89 & 0.9 & 0.97 & & \\
        \bottomrule
    \end{tabular}
    \label{tab:smokers1}
    \end{small}
\end{table}

\textbf{LUBM benchmark.}
To verify the soundness and completeness of the reasoning performed by LNN, we used Lehigh University Benchmark (LUBM)~\cite{lubm}, a synthetic OWL reasoning dataset in university domain with 14 benchmark queries. We generated data for $1$ university (102707 triples), parsed the OWL axioms into equivalent LNN constructs resulting in a graph with 257 nodes. After three passes of bidirectional inferences for the network to converge, $14$ queries were added to the network to compute respective answers. We compared LNN results with few symbolic reasoners namely Stardog~\cite{stardog}, Virtuoso~\cite{virtuoso} and Blazegraph~\cite{blazegraph}. Although, all the systems are sound, achieving $100\%$ precision, only Stardog and LNN answers are complete with $100\%$ recall, compared to average recall of $72\%$(Virtuoso) and $78\%$(BlazeGraph).
In another experiment, we evaluated LNN's ability to handle ontological inconsistencies, specifically those arising from incorrect axioms by inserting them into network. During inference, those inconsistencies emerge as bound value contradictions, thus allowing LNN to accurately locate and down-weight corresponding nodes from taking part in further inference (section~\ref{sup:lubm_benchmark} of the supplementary material expands on noise handling).

\textbf{TPTP benchmark.}
We also used a subset of the TPTP benchmark~\cite{sutcliffe2009tptp,trail} to evaluate LNN in generic classical theorem proving. The TPTP (Thousands of Problems for Theorem Provers) is a comprehensive library for Automated Theorem Proving (ATP) systems, with over 22K problems. We evaluated on a subset of the Common Sense Reasoning domain (937 problems) represented in first-order form (FOF), also filtering out functions and equality --- currently not supported by LNN. From the remaining subset of 25 problems, LNN was able to prove all problems within seconds. Despite this small set of TPTP problems, it is noted that none of the recent neural theorem provers (NTP) \cite{evans2018learning,rocktaschel2017end,dong2019neural} have demonstrated success on classical ATP. Moreover, NTP inference is limited to Horn clauses while LNN can support general first-order logical expressions. These promising results show LNN's potential for generic ATP integrated into an end-to-end differentiable learning system.

\section{Conclusions}
We have 1) introduced a new conceptual neuro-symbolic framework, including introducing a new class of weighted real-valued logics, and ways of providing truth bounds which we show can have
probabilistic semantics, 2) demonstrated approaches for learning in the framework including novel loss function minimizing logical contradiction, and a way to provably bypass the need to
perform constrained optimization, and 3) demonstrated approaches for inference/reasoning in the framework, including the upward-downward algorithm which is provably convergent in finite steps, via
preliminary experiments confirming the efficacy of the approach compared to others.  Planned future work includes approaches for rule induction and mixed symbolic/sub-symbolic sub-networks.

\section*{Broader Impact}

As a step in the direction of explainable AI, logical neural networks provide a flexible and well-performing framework for neuro-symbolic learning that can nonetheless be 1) interpreted on account of their 1-to-1 correspondence to systems of logical formulae, 2) audited by examining the chain of inferences computed for a given query, and 3) controlled by human users through the specification of logical constraints.
As a result, LNNs stand to improve the transparency and fairness of modeled tasks, and may serve as a better performing alternative to older methods (e.g.~decision trees and logistic regression) when explainability is a design requirement.

\bibliographystyle{abbrv}
\bibliography{neurips_2020}

\newpage

\appendix

\title{Supplementary material for Logical Neural Networks}

\makeatletter
  \vbox{%
    \hsize\textwidth
    \linewidth\hsize
    \vskip 0.1in
    \@toptitlebar
    \centering
    {\LARGE\bf \@title\par}
    \@bottomtitlebar
    \vskip 0.3in \@minus 0.1in
  }
\makeatother

\section{First-order logical neural networks}
\subsection{Overview}
\label{sec:fol_lnn}
This section introduces an extension to the representation of the LNN to handle formulae expressed in the first-order logic (FOL) language, and supplements footnote 1 (page 1) in the main paper. The FOL language enables representing a domain in terms of objects that exist in that domain and relations that hold between the objects in that domain. Its vocabulary includes \emph{constant} symbols, \emph{predicate} symbols, \emph{functional} symbols, where constants refer to objects and predicates and functions refer to relations. In addition to the logical connectives in propositional logic, it introduces two \emph{quantifiers}, the \emph{universal} and \emph{existential} quantifiers. Furthermore, some FOL representations introduce an \emph{equality} symbol, which is treated as either a predicate or a logical connective. This section discusses the treatment of First-order LNN without function and equality symbols and treatment of these will be provided in future work. Furthermore, since equality is not handled, we necessarily make the \textbf{unique-names assumption}, insisting that every constant symbol refers to a distinct object in the domain.

In First-order LNN, there exist a neuron for each logical connective as before, a connective neuron, and also for each predicate symbol, an atomic neuron. In addition, each neuron keeps a table whose columns are keyed by unique variables appearing in the represented (sub)formulae, and rows keyed by a set of $n$-element tuples of groundings, where the tuple size $n$ corresponds to the arity of the neuron. The content of the tables are the bounds of each grounding when substituted for the variables. Similar to Negation, quantifiers are modeled as pass-through nodes with no parameters, with special operations for each quantifier type. The bounds of a universal quantifier node are set to the bounds of the grounding with the lowest upper bound, so that if all of its groundings are \texttt{True} then the universal statement is also \texttt{True}. And the bounds of the existential quantifier are set to the bounds of the grounding with highest lower bound, so that if this lower bound is True the existential statement will be \texttt{True} when at least one of its groundings is \texttt{True}.

Computation of bounds throughout the network proceeds as before, with each grounding treated separately. This method of extending to FOL is similar to approaches that reduce inference in classical first-order logic to propositional logic. Each grounding in each formulae is treated as a proposition.

\subsection{Inference in first-order logical neural networks}

Section~\ref{sec:fol_lnn} introduced the representation of First-order logical neural networks, where constants, predicates and quantifiers were introduced. Instead of proposition neurons there are predicate neurons, or atomic neurons, and in addition to connective neurons there are also neurons for quantifiers. In inference for First-order LNN all neurons return tables of groundings and their corresponding bounds. Neural activation functions are then modified to perform joins over columns pertaining to shared variables while computing truth value bounds for the associated grounding. Inverse activation functions are modified similarly, but must also reduce results over any columns pertaining to variables absent from the target input's corresponding subformula so as to aggregate the tightest bounds. In the special case that tables are keyed by consecutive integers, these computations are equivalent to elementwise broadcast operations on sparse tensors, where each tensor dimension pertains to a different variable.

Inverse computation for quantifiers eliminates a given key column(s) corresponding to quantified variable(s) by reducing with min or max as appropriate. However, a proper treatment of inverse computation for the existential quantifier is more complicated as it would introduce Skolem functions. With functions currently not handled, 
inverse computation for the existential quantifier only broadcasts its known upper bounds to all key values associated with its column (i.e.~variable) and broadcasts its known lower bounds to a group of new key values identified by each combination of key values associated with any of its columns and vice versa for universal quantifiers.

\subsection{Grounding management}

Neurons in First-order LNN each have a defined set of variable(s) according to their arity, specifying the number of constants in a grounding tuple. The arity of predicate neurons is usually set beforehand (e.g., informed by a knowledge base ingested by the LNN), and can typically include a variety of nullary (propositional), unary, binary and higher-arity predicates. Connective neurons and quantifiers collect variables from their set of input (or operands) in order of appearance during initialization, where these operands can include atomic neurons, other connective neurons and quantifiers. Variables are collected only once from operands that define repeat occurrences of a specific variable in more than one variable position, unless otherwise specified. Logical formulae can also be defined with arbitrary variable placement across its constituent nodes. A variable mapping operation transforms groundings for enabling truth value lookup in neighboring nodes.

Usually, some formulae initially contain only variables (e.g., axioms asserted corresponding to general rules), leading to neurons with no groundings initially. These neurons must obtain their groundings from their ground operands at initialization or during inference. During inference, each ground neuron propagates its groundings to other neurons with shared variables participating in the same parent neuron, and all operands collectively propagate their groundings to the parent neuron. Similarly, a parent neuron propagates groundings acquired elsewhere to its operands. This grounding management process ensures that constants are propagated throughout the LNN graph in order to compute bounds for relevant queries. Note, however, that this naive grounding management process will propagate all constants, including those not relevant for the query, unnecessarily increasing computation. More efficient methods are discussed in Section~\ref{sec:inf_follnn} below.  

Quantifiers can also have variables and groundings if partial quantification is required for only a subset of variables from the underlying operand, although existential quantification is typically performed on a single variable to produce a propositional truth value associated with the quantifier output. For partial quantification the maximum lower bound of groundings from the quantified variable subset is chosen for existential quantification and assigned to a unique grounding consisting of the remainder of the variables, whereas the minimum upper bound is used for universal quantification. For existential partial quantification \texttt{True} groundings for the quantified variable subset form arguments stored under the grounding of the remaining variable subset, so that satisfying groundings can be recalled.

\subsection{Variable binding}

Variable binding assigns specific constant(s) to variables of neurons, typically as part of an inference task, such as answering a query. A variable could be bound in only a subset of occurrences within a logical formulae, although the procedure for producing groundings for inference would typically propagate the binding to all occurrences. It is thus necessary to retain the variable even if bound, in order to interact with other occurrences of the variable in the logical formula to perform join operations. %

\subsection{First-order logical inference}
\label{sec:inf_follnn}

Inference at a connective neuron involves upward and downward pass computations of the associated logical connective for a given set of groundings, whereas inference at a quantifier neuron involves a reduction operation and creation of new groundings in the case of partial quantification. A provided grounding may not be available in all participating operands of an inference operation, where a retrieval attempt would then add the previously unavailable grounding to the operand with \texttt{Unknown} truth value under the open-world assumption. If a proof is offered to a neuron for an unavailable grounding, the proof aggregation would also assume maximally loose starting bounds.

The computational and memory considerations for large knowledge bases with many constants should be taken under consideration, where action may be taken to avoid storing of groundings with unknown bounds. However, inference is a principal means by which groundings are propagated through a logical formula to enable theorem proving, although there are cases where storage can be avoided. In particular, negation can be viewed as a pass-through operation where inference is performed instead on the underlying operand or descendent that is not also a negation. Otherwise, if naively approached, negation may have to populate a grounding list of all \texttt{False} or missing groundings from the underlying operand and store these as \texttt{True} under the closed-world assumption.

An inference context involves input operands and an output operation, where input operands are used in the upward inference pass to calculate a proof for the output, or where all but one input operand and the output are used in the downward inference pass to calculate a proof for the remaining input. If any participant in the inference context has a grounding that is not \texttt{Unknown}, then in real-valued logic it is possible in an inference context to derive a truth value that is also not \texttt{Unknown}. Each participant in the proof generation can thus add its groundings to the set of inference groundings. However, for complex reasoning problems that require long chains of reasoning it may be useful to also add groundings with \texttt{Unknown} states and propagate them in hopes that they fetch proofs from other neurons in the graph. These proofs can then be passed to other neurons in the graph to facilitate theorem proving. This introduces a trade-off between efficient computation, through avoiding propagating many constants, and handling complex reasoning problems, by allowing more constants propagation.

A given inference grounding is used as is for other participant operands with the same variable configuration as the originating operand. In case of disjoint variable(s) not present in the inference grounding, the overlapping variables are first searched for a match with all the disjoint variable values used in conjunction to create an expanded set of inference groundings. If no overlapping variables are present or no match is found, then the overlapping variables could be assigned according to the inference grounding, with the disjoint variable(s) covering its set of all observed combinations.

The set of relevant groundings from a real-valued inference context could become a significant expanded set, especially in the presence of disjoint variables.
However, guided inference could be used to expand a minimal inference grounding set that only involves groundings relevant to a target proof, and manage the trade-off between computation and reasoning capacity. LNN can use a combination of goal-driven reasoning and data-driven reasoning\footnote{Forward-chaining and backward-chaining algorithms are special cases of data-driven and goal-driven reasoning, respectively, where logical formulae are represented using definite clauses.} to obtain a target proof. A backward-chaining-like algorithm is used here as a means of propagating groundings in search of known truth values that can then be used in a forward-chaining-like computation to infer the goal. If-then conditional rules typically require backward inference in the form of \emph{modus tollens} to propagate groundings to the antecedent and \emph{modus ponens} in the forward direction to help calculate the target proof at the consequent. This bidirectional chaining process continues until the target grounding at the consequent is not unknown or until inference does not produce proofs that are any tighter.

In summary, overall computation is characterized similar to the propositional case, with a few minor modifications:
\begin{enumerate}
\item Initialize neurons corresponding to predicates and formula roots with starting truth value bounds for given groundings.
  Usually, all formulae are initialized as \texttt{True} for all associated groundings. Ground atomic neurons with starting bounds represent input data.
\item \label{step:update2} For each formula, evaluate neurons in the forward direction in a pass from leaves to root, storing computed bounds at each node for all groundings, and propagating constants to other nodes in the same formula. Then, backtrack to each leaf using inverse computations to update subformula bounds based on stored bounds (and formula initial bounds), and propagate constants potentially fetched from other neurons in the graph.
\item \label{step:tighten2} Aggregate the tightest bounds computed at leaves for each proposition.
\item Return to step~\ref{step:update2} until bounds for all groundings converge.  Oscillation cannot occur because bounds tightening is monotonic.
\item \label{step:predict2} Inspect computed bounds at specific predicates or formulae, i.e.~those representing the predictions of the model.
\end{enumerate}

\subsection{Acceleration}

As bounds tightening is monotonic, the order of evaluation does not change the final result.
As a result, and in line with traditional theorem provers, computation may be subject to significant acceleration depending on the order that bounds are updated.

In order for such aggregate operations to be tractable, it is necessary to limit the number of key values that participate in computation, leaving other key value combinations in a sparse state, i.e.~with default bounds.
We achieve this by filtering predicates whenever possible to include only content pertaining to specific key values referenced in queries or involved in joins with other tables, prioritizing computation towards smaller such content.
Because many truth values remain uncomputed in this model, the results of quantifiers and other reductions may not be tight, but they are nonetheless sound.
In cases where predicates have known truth values for all key values (i.e.~because they make the closed world assumption), we use different bounds for their sparse value and for the sparse values of connectives involving them, such that a connective's sparse value is its result for its inputs sparse values.

Even minimizing the number of key values participating in computation, it is necessary to guide neural evaluation towards rules that are more likely to produce useful results.
A first opportunity to this effect is to shortcut computation if it fails to yield tighter bounds than were previously stored at a given neuron.
In addition, we exploit the neural graph structure to prioritize evaluation in rules with shorter paths to the query and to visited rules with recently updated bounds.

Tensorization offers another route to accelerate computation, by formulating LNN in terms of weighted adjacency matrices and keeping truth values in multi-dimensional arrays indexed by node and grounding identifiers. The resulting truth value tensor can be sparse so that only entries of existing groundings are stored, and a further batch dimension corresponding to different universes and truth value initializations is also possible.
A set of weighted adjacency matrices can represent the neural graph structure, with one adjacency matrix for each of the different operators, including conjunction, disjunction, and implication. The neuron weighting scheme can also extend to admit negative weights used such that corresponding inputs are logically negated.

\section{Examples of weighted real-valued logics}
\label{sup:weighted_real_value_logics}

Weighted nonlinear logic is in fact only one family of possible logics implemented in LNN.
Notably, the logic introduced in Section~\ref{sec:probability_bounds} for Theorem~\ref{th:1} is a generalization of weighted nonlinear logic in that its lower and upper bounds are computed by different functions.
Another important family of logics satisfying the requirements for LNN are the t-norm logics, which we define presently.

Triangular norms, or \emph{t-norms}, and their related t-conorms and residua are natural choices for LNN activation functions as they already behave correctly for classical inputs and have well known inference properties.
Logics defined in terms of such function are denoted \emph{t-norm logics}.
Common examples of these include
\begin{equation*}
    \begin{array}{llll}
        & \text{G\"odel} & \text{Product} & \text{\L ukasiewicz} \\ \cline{2-4}
        \text{T-norm} & \min\{x, y\} & x \cdot y & \max\{0, x + y - 1\} \rule{0pt}{10pt} \\
        \text{T-conorm} & \max\{x, y\} & x + y - x \cdot y & \min\{1, x + y\} \\
        \text{Residuum} & y \text{ if } x > y, \text{ else } 1 & \frac{y}{x} \text{ if } x > y, \text{ else } 1 & \min\{1, 1 - x + y\}
    \end{array}
\end{equation*}
Of these, only \L ukasiewicz logic offers the familiar \((x \rightarrow y) = (\neg x \oplus y)\) identity, while only G\"odel logic offers the \((x \otimes x) = (x \oplus x) = x\) identities.

\subsection{Weighted \L ukasiewicz logic}
Weighted \L ukasiewicz logic is exactly weighted nonlinear logic with \(f(x) = \max\{0,\min\{1,x\}\}\).
The binary and \(n\)-ary \emph{weighted \L ukasiewicz t-norms}, used for logical AND, are given
\begin{align}
    {}^{\beta}(x_1^{\otimes w_1} \otimes x_2^{\otimes w_2}) & = \max\{0, \min\{1, \beta - w_1 (1 - x_1) + w_2 (1 - x_2)\}\}, \\
    \textstyle {}^{\beta}(\bigotimes_{i \in I} x_i^{\otimes w_i}) & \textstyle = \max\{0, \min\{1, \beta - \sum_{i \in I} w_i (1 - x_i)\}\} \label{eq:lukasiewicz_and_forward}
\intertext{%
for input set~\(I\), nonnegative bias term~\(\beta\), nonnegative weights~\(w_i\), and inputs~\(x_i\) in the \([0,1]\) range.
By the De Morgan laws, the binary and \(n\)-ary \emph{weighted \L ukasiewicz t-conorms}, used for logical OR, are then
}
    {}^{\beta}(x_1^{\oplus w_1} \oplus x_2^{\oplus w_2}) & = \max\{0, \min\{1, 1 - \beta + w_1 x_1 + w_2 x_2\}\} \\
    \textstyle {}^{\beta}(\bigoplus_{i \in I} x_i^{\oplus w_i}) & \textstyle = \max\{0, \min\{1, 1 - \beta + \sum_{i \in I} w_i x_i\}\}.\label{eq:lukasiewicz_or_forward}
\end{align}
In either case, the unweighted \L ukasiewicz norms are obtained when all \(w_i = \beta = 1\); if any of these parameters are omitted, their presumed value is 1.
The exponent notation is chosen because, for integer weights~\(k\), this form of weighting is equivalent to repeating the associated term \(k\) times using the respective unweighted norm, e.g.~\(x^{\oplus 3} = (x \oplus x \oplus x)\).
Bias term~\(\beta\) is written as a leading exponent to permit inline ternary and higher arity-norms, for example \({}^{\beta}(x_1^{\oplus w_1} \oplus x_2^{\oplus w_2} \oplus x_3^{\oplus w_3})\), which require only a single bias term to be fully parameterized.

The \emph{weighted \L ukasiewicz residuum}, used for logical implication, solves
\begin{equation}
    \max \{z : y \geq {}^{\beta/w_y}(x^{\otimes w_x/w_y} \otimes z^{\otimes 1/w_y})\}
\end{equation}
and is given
\begin{align}
    {}^{\beta}(x^{\otimes w_x} \rightarrow y^{\oplus w_y}) & = \max\{0, \min\{1, 1 - \beta + w_x (1 - x) + w_y y\}\} \label{eq:lukasiewicz_implies_forward}\\
    & = {}^{\beta}((1 - x)^{\oplus w_x} \oplus y^{\oplus w_y}). \notag
\end{align}
As for weighted nonlinear logic, note the use of \(\otimes\) in the antecedent weight but \(\oplus\) in the consequent weight, meant to indicate the antecedent has AND-like weighting (scaling its distance from 1) while the consequent has OR-like weighting (scaling its distance from 0).
Similarly, this residuum is most disjunction-like when \(\beta = 1\), most \((x \rightarrow y)\)-like when \(\beta = w_y\), and most \((\neg y \rightarrow \neg x)\)-like when \(\beta = w_x\); that is to say, \(\beta = w_y\) yields exactly the residuum of \(x^{\otimes w_x/w_y} \otimes z^{\otimes 1/w_y}\) (with no specified bias term of its own), while \(\beta = w_x\) yields exactly the residuum of \((\neg y)^{\otimes w_y/w_x} \otimes z^{\otimes 1/w_x}\).

The \L ukasiewicz norms are commutative if one permutes weights~\(w_i\) along with inputs~\(x_i\), and are associative if bias term~\(\beta \leq \min \{1, w_i : i \in I\}\).
Further, they return classical results, i.e.~results in the set \(\{0, 1\}\), for classical inputs under the condition that \(1 \leq \beta \leq \min \{w_i : i \in I\}\).
This clearly requires \(\beta = 1\) to obtain both associative and classical behavior, though neither is a requirement for LNN.
Indeed, constraining \(\beta \leq w_i\) is problematic if we would like \(w_i\) to be able to goes to 0, effectively removing \(i\) from input set~\(I\), whereupon the constraint should no longer apply.
Section~\ref{sec:learning} presents a means of relaxing such constraints so as to facilitate learning.

\subsection{Weighted G\"odel logic}
G\"odel logic uses min and max for its t-norm and t-conorm, respectively.
It is difficult in general to define weighted versions of these functions---for example, forms like \(x^{\oplus 3} = (x \oplus x \oplus x)\) are not meaningful due to min and max's idempotence---though several well-studied options exist, notably Fagin's method \cite{fagin2000formula}.
This document, however, uses a technique borrowed from \cite{hajek2013metamathematics} to derive a weighted min from another t-norm, specifically the weighted \L ukasiewicz t-norm.
H\'ajek shows that, for \emph{any} continuous t-norm, one may define \emph{weak conjunction}
\begin{equation}
    (x \land y) = (x \otimes (x \rightarrow y)) = \min\{x, y\}.
\end{equation}
Using weighted \L ukasiewicz logic and a specifically crafted configuration of weights\footnote{%
    This configuration of weights is the only one in which \(x\) may be swapped with \(y\), \(w_x\) with \(w_y\), and \(\beta_x\) with \(\beta_y\) without affecting the result; accordingly, it is the most reasonable adaptation of the above formulae to define weak conjunction for the weighted \L ukasiewicz t-norm.
}
\begin{equation}
    ({}^{\beta_x}(x^{\otimes w_x}) \land {}^{\beta_y}(y^{\otimes w_y})) = {}^{\beta_x}(x^{\otimes w_x} \otimes (x^{\otimes w_x/\kappa} \rightarrow y^{\oplus w_y/\kappa})^{\otimes \kappa})
\end{equation}
for \(\kappa = \beta_x - \beta_y + w_y\), one obtains binary and n-ary \emph{weighted G\"odel t-norms}, defined
\begin{align}
    ({}^{\beta_1}(x_1^{\otimes w_1}) \land {}^{\beta_2}(x_2^{\otimes w_2})) & = \max\{0, \min\{1, \beta_1 - w_1 (1 - x_1), \beta_2 - w_2 (1 - x_2)\}\}, \\
    \textstyle (\bigwedge_{i \in I} {}^{\beta_i}(x_i^{\otimes w_i})) & \textstyle = \max\{0, \min\{1, \beta_i - w_i (1 - x_i) : i \in I\}\}
\intertext{%
which is just the min of the unary weighted \L ukasiewicz t-norm applied to each argument, and which, if all \(\beta_i = \beta\), happens to be very similar to ((Sung 1998)).
Again using the De Morgan laws, the related binary and n-ary \emph{weighted G\"odel t-conorms} are then
}
    ({}^{\beta_1}(x_1^{\oplus w_1}) \lor {}^{\beta_2}(x_2^{\oplus w_2})) & = \min\{1, \max\{0, 1 - \beta_1 + w_1 x_1, 1 - \beta_2 + w_2 x_2\}\}, \\
    \textstyle (\bigvee_{i \in I} {}^{\beta_i}(x_i^{\oplus w_i})) & \textstyle = \min\{1, \max\{0, 1 - \beta_i + w_i x_i : i \in I\}\}
\end{align}

The \emph{weighted G\"odel residuum} now solves
\begin{equation}
    \max \{z : y \geq {}^{\beta_{xy}/w_y}(x^{\otimes w_x/w_y}) \land {}^{\beta_y/w_y}(z^{\otimes 1/w_y})\}
\end{equation}
for \(\beta_{xy} = \max\{0, \beta_x + \beta_y - 1\}\) and is given
\begin{align}
    ({}^{\beta_x}(x^{\otimes w_x}) \Rightarrow {}^{\beta_y}(y^{\oplus w_y})) & = {}^{\beta_y}(y^{\oplus w_y}) \text{ if } {}^{\beta_x}(x^{\otimes w_x}) > {}^{\beta_y}(y^{\oplus w_y}), \text{ else } 1
\intertext{%
where operands are again unary weighted \L ukasiewicz norms, or specifically
}
    {}^{\beta_x}(x^{\otimes w_x}) & = \max\{0, \min\{1, \beta_x - w_x (1 - x)\}\}, \\
    {}^{\beta_y}(y^{\oplus w_y}) & = \max\{0, \min\{1, 1 - \beta_y + w_y y\}\}.
\end{align}

The weighted G\"odel norms are commutative if one permutes both weights~\(w_i\) and biases~\(\beta_i\) along with inputs~\(x_i\) and, similar to the weighted \L ukasiewicz norms, are associative if \(\beta_i \leq \min\{1,w_i\}\) for each \(i \in I\).
Likewise, they behave classically for classical input if \(1 \leq \beta_i \leq w_i\) for each \(i \in I\).

\subsection{Parameter semantics}
Weights~\(w_i\) need not sum to 1; accordingly, they are best interpreted as absolute importance as opposed to relative importance.
As mentioned above, for conjunctions, increased weight amplifies the respective input's distance from 1, while for disjunctions, increased weight amplifies the respective input's distance from 0.
Decreased weight has the opposite effect, to the point that inputs with zero weight have no effect on the result at all.

Bias term~\(\beta\) is best interpreted as continuously varying the ``difficulty'' of satisfying the operation.
In weighted \L ukasiewicz logic, this can so much as translate from one logical connective to another, e.g.~from logical AND to logical OR.
Constraints imposed on \(\beta\) and \(w_i\) can guarantee that the operation performed at each neuron matches the corresponding connective in the represented formula, e.g., when inputs are assumed to be within a given distance of 1 or 0, as further discussed in Section~\ref{sec:learning}.

\section{The Upward-Downward algorithm (continued)}
\label{sec:recurrent_algorithm}
This section supplements its counterpart in the main paper on page 5, and provides more complete algorithm descriptions. Inference tasks in LNN involve generating proofs of truth values for specified output nodes, given truth values at specified input nodes. Note that a node could both be an input and an output node so that initial proofs can be provided and updated through inference to obtain an aggregate proof as output for the same node. Proof generation operates over a system of formulae that is represented as a mapping of the corresponding syntax tree to a directed graph. An upward pass described in Algorithm \ref{algo-upwardpass} calculates a proof at a formula using values from its input terms and subformulae, which corresponds to normal forward computation for neurons. A downward pass detailed in Algorithm \ref{algo-downwardpass} proves truth values for each input operand based on the other operand truth values and the enclosing formula's truth value.

\begin{algorithm*}[htpb!]
\caption{Upward pass to infer formula truth value bounds from subformulae bounds.\label{algo-upwardpass*}}%
\small%
\SetKwComment{tcp}{\#{ }}{}%
\SetAlgoLined\SetArgSty{}%
\SetFuncArgSty{}
\SetKwProg{Fn}{function}{\string:}{}%
\SetKwFunction{FRecurs}{upwardPass}%
\Fn(\tcp*[f]{recursive upward pass inference function}){\FRecurs{formula $z$}}{
\If(\tcp*[f]{$z$ is an atom or $n$-ary predicate}){$z=P(v_0,\ldots,v_n)$}{
\If(\tcp*[f]{uninitialized truth value}){\FuncSty{uninitialized}$(z)$}{
$(L_{z}, U_{z})$:=$(0, 1)$\tcp*[f]{unknown under open-world assumption}
}

\Return{$(L_{z}, U_{z})$}\tcp*[f]{existing or default truth value bounds}
}
\BlankLine
\For(\tcp*[f]{every input operand $x_i$ of $z$}){subformula $x_{i\in I}$}{
$(L_{x_i}', U_{x_i}')$:=\FuncSty{upwardPass}$(x_i)$\tcp*[f]{obtain operand bounds, recurses to leaves}

$(L_{x_i}, U_{x_i})$:=\FuncSty{aggregate}$(x_i, (L_{x_i}', U_{x_i}'))$\tcp*[f]{aggregate new proof}
}
\BlankLine
\If(\tcp*[f]{$z$ performs negation}){$z=\neg x$}{
\Return{$(1-U_x,\;1-L_x)$}\tcp*[f]{single input operand $x$}
}
\ElseIf(\tcp*[f]{$z$ performs implication}){$z={}^{\beta}(x^{\otimes w_x} \rightarrow y^{\oplus w_y})$}{
\Return{$({}^{\beta}(U_x^{\otimes w_x} \rightarrow L_y^{\oplus w_y}),\;{}^{\beta}(L_x^{\otimes w_x} \rightarrow U_y^{\oplus w_y}))$}\tcp*[f]{two input operands $x,y$}
}
\ElseIf(\tcp*[f]{$z$ performs conjunction}){$z={}^{\beta}(\bigotimes_{i \in I} x_i^{\otimes w_i})$}{
\Return{$({}^{\beta}(\bigotimes_{i \in I} L_{x_i}^{\otimes w_i}),\;{}^{\beta}(\bigotimes_{i \in I} U_{x_i}^{\otimes w_i}))$}\tcp*[f]{multi-input conjunction}
}
\ElseIf(\tcp*[f]{$z$ performs disjunction}){$z={}^{\beta}(\bigoplus_{i \in I} x_i^{\oplus w_i})$}{
\Return{$({}^{\beta}(\bigoplus_{i \in I} L_{x_i}^{\oplus w_i}),\;{}^{\beta}(\bigoplus_{i \in I} U_{x_i}^{\oplus w_i}))$}\tcp*[f]{multi-input disjunction}
}
\ElseIf(\tcp*[f]{$z$ is universal quantifier over groundings $G$}){$z=\forall_{g\in G}x(g)$}{
\Return{$(\min_{g\in G} L_{x(g)},\;\min_{g\in G} U_{x(g)})$}\tcp*[f]{universal quantification over $G$}
}
\ElseIf(\tcp*[f]{$z$ is existential quantifier over groundings $G$}){$z=\exists_{g\in G}x(g)$}{
\Return{$(\max_{g\in G} L_{x(g)},\;\max_{g\in G} U_{x(g)})$}\tcp*[f]{existential quantification over $G$}
}
}
\BlankLine
\SetKwFunction{FRecurs}{aggregate}%
\Fn(\tcp*[f]{aggregate offered proof $(L_z', U_z')$ for $z$}){\FRecurs{formula $z$, $(L_z', U_z')$}}{
\Return{$(\max(L_z,L_z'),\; \min(U_z,U_z'))$}\tcp*[f]{monotonically tighten existing $(L_z, U_z)$ bounds}
}
\end{algorithm*}

\begin{algorithm*}[htpb!]
\caption{Downward pass to infer subformula truth value bounds.\label{algo-downwardpass*}}%
\small%
\SetKwComment{tcp}{{ }{ }{ }{ }{ }\#{ }}{}%
\SetAlgoLined\SetArgSty{}%
\SetFuncArgSty{}
\SetKwProg{Fn}{function}{\string:}{}%
\SetKwFunction{FRecurs}{downwardPass}%
\Fn(\tcp*[f]{recursive downward pass inference function}){\FRecurs{formula $z$}}{
\For(\tcp*[f]{every input operand $x_j$ of $z$}){subformula $x_{j\in I}$}{
\BlankLine
\If(\tcp*[f]{$z$ performs negation}){$z=\neg x$}{
$(L_x', U_x')$:=$(1-U_z,\;1-L_z)$\tcp*[f]{proof for single input operand $x$}
}
\ElseIf(\tcp*[f]{$z$ performs implication}){$z={}^{\beta}(x^{\otimes w_x} \rightarrow y^{\oplus w_y})$}{
$L_x'$:=$\left\{\begin{array}{l@{\qquad}l}
        {}^{\beta/w_x}(U_z^{\otimes 1/w_x} \rightarrow L_y^{\oplus w_y/w_x})  & \text{if } U_z < 1, \\
        0 & \text{otherwise.}
    \end{array}\right.$\tcp*[f]{left input $x$ lower bound}
    
$U_x'$:=$\left\{\begin{array}{l@{\qquad}l}
        {}^{\beta/w_x}(L_z^{\otimes 1/w_x} \rightarrow U_y^{\oplus w_y/w_x})  & \text{if } L_z > 0, \\
        1 & \text{otherwise}
    \end{array}\right.$\tcp*[f]{left input $x$ upper bound}
    
$L_y'$:=$\left\{\begin{array}{l@{\qquad}l}
        {}^{\beta/w_y}(L_x^{\otimes w_x/w_y} \otimes L_z^{\otimes 1/w_y}) & \text{if } L_z > 0, \\
        0 & \text{otherwise}
    \end{array}\right.$\tcp*[f]{right input $y$ lower bound}
    
$U_y'$:=$\left\{\begin{array}{l@{\qquad}l}
        {}^{\beta/w_y}(U_x^{\otimes w_x/w_y} \otimes U_z^{\otimes 1/w_y})  & \text{if } U_z < 1, \\
        1 & \text{otherwise}
    \end{array}\right.$\tcp*[f]{right input $y$ upper bound}
}
\ElseIf(\tcp*[f]{$z$ performs conjunction}){$z={}^{\beta}(\bigotimes_{i \in I} x_i^{\otimes w_i})$}{
$L_{x_j}'$:=$\left\{\begin{array}{l@{\qquad}l}
        {}^{\beta/w_j}((\bigotimes_{i \neq j} U_{x_i}^{\otimes w_i/w_j}) \rightarrow L_z^{\oplus 1/w_j}) & \text{if } L_z > 0, \\
        0 & \mathrm{otherwise,} \end{array}\right.$\tcp*[f]{lower bound proof}
        
$U_{x_j}'$:=$\left\{\begin{array}{l@{\qquad}l}
        {}^{\beta/w_j}((\bigotimes_{i \neq j} L_{x_i}^{\otimes w_i/w_j}) \rightarrow U_z^{\oplus 1/w_j}) & \text{if } U_z < 1, \\
        1 & \mathrm{otherwise,}\end{array}\right.$\tcp*[f]{upper bound proof}
}
\ElseIf(\tcp*[f]{$z$ performs disjunction}){$z={}^{\beta}(\bigoplus_{i \in I} x_i^{\oplus w_i})$}{
$L_{x_j}'$:=$\left\{\begin{array}{l@{\qquad}l}
        {}^{\beta/w_j}((\bigotimes_{i \neq j} (\neg U_{x_i})^{\otimes w_i/w_j}) \otimes L_z^{\otimes 1/w_j}) & \text{if } L_z > 0, \\
        0 & \mathrm{otherwise,}\end{array}\right.$\tcp*[f]{lower bound proof}

$U_{x_j}'$:=$\left\{\begin{array}{l@{\qquad}l}
        {}^{\beta/w_j}((\bigotimes_{i \neq j} (\neg L_{x_i})^{\otimes w_i/w_j}) \otimes U_z^{\otimes 1/w_j}) & \text{if } U_z < 1, \\
        1 & \mathrm{otherwise.} \end{array}\right.$\tcp*[f]{upper bound proof}
}
\ElseIf(\tcp*[f]{$z$ is universal quantifier over groundings $G$}){$z=\forall_{g\in G}x(g)$}{
\For(\tcp*[f]{every $n$-ary grounding tuple $g$}){$g\in G$}{
$(L_{x(g)}', U_{x(g)}'):=(L_z,\;1)$\tcp*[f]{proof of minimum lower bound}
}
}

\ElseIf(\tcp*[f]{$z$ is existential quantifier over groundings $G$}){$z=\exists_{g\in G}x(g)$}{
\For(\tcp*[f]{every $n$-ary grounding tuple $g$}){$g\in G$}{
$(L_{x(g)}', U_{x(g)}'):=(0,\;U_z)$\tcp*[f]{proof of maximum upper bound}
}
}
\BlankLine
$(L_{x_i}, U_{x_i})$:=\FuncSty{aggregate}$(x_i, (L_{x_i}', U_{x_i}'))$\tcp*[f]{aggregate new proof}
\BlankLine
\FuncSty{downwardPass}$(x_i)$\tcp*[f]{recurses to leaves, stop before cyclic loop}
}
}
\end{algorithm*}

\begin{algorithm*}[htpb!]
\caption{Recurrent inference procedure with recursive directional graph traversal.\label{algo-inference*}}%
\small%
\SetKwComment{tcp}{\#{ }}{}%
\SetAlgoLined\SetArgSty{}%
\SetFuncArgSty{}
\SetKwProg{Fn}{function}{\string:}{}%
\SetKwFunction{FRecurs}{inference}%
\Fn(\tcp*[f]{iterative function with input formulae}){\FRecurs{formulae $\mathbf{z}$}}{
\While(\tcp*[f]{new proofs being generated}){$\sum(|\delta L_z| + |\delta U_z|) > \epsilon$}{
\For(\tcp*[f]{root nodes of system of formulae}){$r\in$ \FuncSty{roots}($\mathbf{z}$)}{
$(L_{r}', U_{r}')$:=\FuncSty{upwardPass}$(r)$\tcp*[f]{leaves-to-root traversal}

$(L_{r}, U_{r})$:=\FuncSty{aggregate}$(r, (L_{r}', U_{r}'))$\tcp*[f]{root aggregates new proof}

\FuncSty{downwardPass}$(r)$\tcp*[f]{root-to-leaf traversal}
}
}
}
\end{algorithm*}

Inference propagates information through edges of the formula syntax tree until convergence where no new proofs are generated, as shown in Algorithm \ref{algo-inference}. Information flow in connected components can be optimized through sequential traversal over adjacent nodes such that inference has access to the newest proofs recently calculated for neighboring formulae. Preorder and postorder traversals avoid cycles within an inference epoch so that only a single upward or downward calculation is performed for a node within an epoch. Sequential traversal with upward passes starts from propositions, known truth values, or syntax tree leaves, while downward passes move information from outer formulae to inner terms.

LNN converts to a directed acyclic graph through cycle-avoidance, and over alternating upward and downward passes the graph is unrolled according to the traversal sequence to resemble a finite impulse recurrent network. However, given the monotonic tightening of proof aggregation, the network is more constrained than the dynamic behavior exhibited by recurrent neural networks. Downward traversal can mirror the upward pass, although transient directional edges are effectively introduced by downward inference so separate proofs are offered to each neuron input according to the functional inverse calculation with analytically transformed weights.

Given both upward and downward computations for each connective, overall computation is characterized as follows:
\begin{enumerate}
\item Initialize neurons corresponding to propositions and formula roots with starting truth value bounds.
  Usually, all formulae are initialized as true.
  Propositions with starting bounds represent input data.
\item \label{step:update} For each formula, evaluate neurons in the forward direction in a pass from leaves to root, storing computed bounds at each node.
  Then, backtrack to each leaf using inverse computations to update subformula bounds based on stored bounds (and formula initial bounds).
\item \label{step:tighten} Aggregate the tightest bounds computed at leaves for each proposition.
\item Return to step \ref{step:update} until bounds converge.  Oscillation cannot occur because bounds tightening is monotonic.
\item \label{step:predict} Inspect computed bounds at specific propositions or formulae, i.e.~those representing the predictions of the model.
\end{enumerate}

As suggested in step \ref{step:tighten}, one may simply use min and max to aggregate upper and lower bounds proved for each proposition, though smoothed versions of these may be preferred to spread gradient information over multiple proofs.
Alternately, when targeting classical logic, one can use conjunction and disjunction (themselves possibly smoothed) to aggregate proposition bounds.
When doing so, there is an opportunity to reuse propositions' weights from their respective proofs, so as to limit the effect of proofs in which the proposition only plays a minor role.

As suggested in step \ref{step:predict}, prediction results are obtained by inspecting the outputs of one or more neurons, similar to what would be done for a conventional neural network.
Different, however, is the fact that different neurons may serve as inputs and results for different queries, indeed with a result for one query possibly used as an input for another.
In addition, one may arbitrarily extend an existing LNN model with neurons representing new formulae to serve as a novel query.

\subsection{Proof of Theorem~1}

We shall now proof Theorem~1, or specifically that, given weighted nonlinear logic with monotonic \(\neg\), \(\oplus\), and \(f\), Algorithm~\ref{algo-inference} converges to within \(\epsilon\) in finite time for the propositional case.
\begin{proof}
All operations in weighted nonlinear logic are implemented in terms of \(\neg\), \(\oplus\), and \(f\) and are thus also monotonic functions of their inputs.
Truth value bounds aggregation is likewise monotonic, always taking the tightest available bounds, as per Algorithms~\ref{algo-upwardpass}, \ref{algo-downwardpass}, and \ref{algo-inference}.
Because lower bounds can only increase, but have a maximum value of 1, and upper bounds can only decrease, but have a minimum value of 0, each sequence of updates for a given bound must be constitute a Cauchy sequence; otherwise, there would have to exist some step size~\(\delta > 0\) by which some truth value bound updates an infinite number of times, but this would clearly push the bound past its limit of 1 or 0.
Accordingly, after some finite number of iterations of Algorithm~\ref{algo-inference}, all truth value bounds will be within \(\frac{\epsilon}{n}\) of the end point of their Cauchy sequence.
For \(n\) total lower and upper bounds, the sum of all such deviations will be at most \(\epsilon\).
\end{proof}
This proof does not apply to FOL because, unlike the propositional case, FOL can introduce lower and upper bounds at new predicate groundings throughout evaluation.
For example, a successor function can produce an infinite number of predicate groundings; while each of these constitutes its own Cauchy sequence and thus necessarily converges independently of the others, the entire system may never converge.
This result is expected given the well known undecideable nature of FOL.

\section{Proof of Theorem 2}
\label{sec:th2_proof}

To prove Theorem 2, let us first give an expanded description of the LNN variant from the first paragraph of Section 5.
It is mathematically equivalent to that paragraph, and the expanded bound-update equations for various logical connectives will facilitate the proof.

Consider a bipartite graph $G = \langle V_1\cup V_2,E \rangle$, where each node in $V_1$ represents a formula and each node in $V_2$ represents a connective.
Nodes in $V_2$ have the following types: NOT, IMPLIES, AND-2, OR-2, AND-3, OR-3, $\cdots$.
A NOT node always has a degree of 2.
An IMPLIES node always has a degree of 3.
An AND-$k$ or OR-$k$ node always has a degree of $k+1$.
For a node $v\in V_1\cup V_2$, let $d_v$ denote its degree and let $n_{v,1},\cdots,n_{v,d_v}$ denote its neighbors.

Define two functions $L:V_1 \rightarrow \left[ 0, 1 \right]$ and $U:V_1 \rightarrow \left[ 0, 1 \right]$.
Each node $v\in V_2$ is annotated with $2d_v$ functions:
\begin{align}
  \tilde{L}_{L,U,v,1} &= f_{v,1}\left( L\left(n_{v,2}\right),\cdots,L\left(n_{v,d_v}\right),U\left(n_{v,2}\right),\cdots,U\left(n_{v,d_v}\right) \right)\nonumber\\
  \tilde{U}_{L,U,v,1} &= g_{v,1}\left( L\left(n_{v,2}\right),\cdots,L\left(n_{v,d_v}\right),U\left(n_{v,2}\right),\cdots,U\left(n_{v,d_v}\right) \right)\nonumber\\
  \tilde{L}_{L,U,v,2} &= f_{v,2}\left( L\left(n_{v,1}\right),L\left(n_{v,3}\right),\cdots,L\left(n_{v,d_v}\right),U\left(n_{v,1}\right),U\left(n_{v,3}\right),\cdots,U\left(n_{v,d_v}\right) \right)\nonumber\\
  \tilde{U}_{L,U,v,2} &= g_{v,2}\left( L\left(n_{v,1}\right),L\left(n_{v,3}\right),\cdots,L\left(n_{v,d_v}\right),U\left(n_{v,1}\right),U\left(n_{v,3}\right),\cdots,U\left(n_{v,d_v}\right) \right)\nonumber\\
  \cdots & \nonumber\\
  \tilde{L}_{L,U,v,d_v} &= f_{v,d_v}\left( L\left(n_{v,1}\right),\cdots,L\left(n_{v,d_v-1}\right),U\left(n_{v,1}\right),\cdots,U\left(n_{v,d_v-1}\right) \right)\nonumber\\
  \tilde{U}_{L,U,v,d_v} &= g_{v,d_v}\left( L\left(n_{v,1}\right),\cdots,L\left(n_{v,d_v-1}\right),U\left(n_{v,1}\right),\cdots,U\left(n_{v,d_v-1}\right) \right)\nonumber
\end{align}

Without loss of generality, let's assume an indexing scheme for neighbors of $v\in V_2$: if $v$ has type AND-$k$, $n_{v,k+1}=n_{v,1}\land\cdots\land n_{v,d_k}$; if $v$ has type OR-$k$, $n_{v,k+1}=n_{v,1}\lor\cdots\lor n_{v,k}$; if $v$ has type IMPLIES, $n_{v,3}=n_{v,1} \rightarrow n_{v,2}$.

With different choices of $\tilde{L}$'s and $\tilde{U}$'s, LNN can implement different flavors of real-valued logic.
The variant in Section 5 uses the following choice of $\tilde{L}$'s and $\tilde{U}$'s, grouped by types of node $v\in V_2$.

\begin{itemize}
\item
  NOT

\begin{align}
  \tilde{L}_{L,U,v,1} &= 1-U\left(n_{v,2}\right) \label{eq:notl1}\\
  \tilde{U}_{L,U,v,1} &= 1-L\left(n_{v,2}\right) \label{eq:notu1}\\
  \tilde{L}_{L,U,v,2} &= 1-U\left(n_{v,1}\right) \label{eq:notl2}\\
  \tilde{U}_{L,U,v,2} &= 1-L\left(n_{v,1}\right) \label{eq:notu2}
\end{align}

\item
  AND-$k$

For $i=1,\cdots,k$:
\begin{align}
  \tilde{L}_{L,U,v,i} &= L\left(n_{v,k+1}\right) \label{eq:andl1}\\
  \tilde{U}_{L,U,v,i} &= \min\left(1,U\left(n_{v,k+1}\right)+\sum_{1\leq j\leq k,j\neq i}\left(1-L\left(n_{v,j}\right)\right)\right) \label{eq:andu1}
\end{align}

\begin{align}
  \tilde{L}_{L,U,v,k+1} &= \max\left(0,1-\sum_{j=1}^{k}\left(1-L\left(n_{v,j}\right)\right)\right) \label{eq:andl2}\\
  \tilde{U}_{L,U,v,k+1} &= \min_{j=1}^{k}U\left(n_{v,j}\right) \label{eq:andu2}
\end{align}

\item
  OR-$k$

For $i=1,\cdots,k$:
\begin{align}
  \tilde{L}_{L,U,v,i} &= \max\left(0,L\left(n_{v,k+1}\right)-\sum_{1\leq j\leq k,j\neq i}U\left(n_{v,j}\right)\right) \label{eq:orl1}\\
  \tilde{U}_{L,U,v,i} &= U\left(n_{v,k+1}\right) \label{eq:oru1}
\end{align}

\begin{align}
  \tilde{L}_{L,U,v,k+1} &= \max_{j=1}^{k}L\left(n_{v,j}\right) \label{eq:orl2}\\
  \tilde{U}_{L,U,v,k+1} &= \min\left(1,\sum_{j=1}^{k}U\left(n_{v,j}\right)\right) \label{eq:oru2}
\end{align}

\item
  IMPLIES

\begin{align}
  \tilde{L}_{L,U,v,1} &= 1-U\left(n_{v,3}\right) \label{eq:impl1}\\
  \tilde{U}_{L,U,v,1} &= \min\left(1,1+U\left(n_{v,2}\right)-L\left(n_{v,3}\right)\right) \label{eq:impu1}\\
  \tilde{L}_{L,U,v,2} &= \max\left(0,L\left(n_{v,1}\right)+L\left(n_{v,3}\right)-1\right) \label{eq:impl2}\\
  \tilde{U}_{L,U,v,2} &= U\left(n_{v,3}\right) \label{eq:impu2}\\
  \tilde{L}_{L,U,v,3} &= \max\left(1-U\left(n_{v,1}\right),L\left(n_{v,2}\right)\right) \label{eq:impl3}\\
  \tilde{U}_{L,U,v,3} &= \min\left(1,1-L\left(n_{v,1}\right)+U\left(n_{v,2}\right)\right) \label{eq:impu3}
\end{align}
  
\end{itemize}

With the same notations as in Section 5, initial knowledge of LNN inference is specified by a set of formulas $V_0\subseteq V_1$ and two functions $L_0:V_0 \rightarrow \left[ 0, 1 \right]$ and $U_0:V_0 \rightarrow \left[ 0, 1 \right]$.
The query formula $\sigma$ is either an existing node in $V_1$ or a formula that is composed of existing nodes in $V_1$.
Without loss of generality, let's assume $\sigma\in V_1$, because otherwise we can first expand $G$ by adding connectives until $\sigma$ is included in $V_1$

Consider node $v$ and one of its neighbors $n_{v,i}$, let $j_{v,i}$ denote $v$'s index among $n_{v,i}$'s neighbors.
In other words, $n_{n_{v,i},j_{v,i}} \equiv v$.

LNN inference, i.e., Algorithm 3 in the main paper, can be written as the following pseudo code, where the selection of $v$ and $i$ in the loop is determined by the upward and downward passes of Algorithms 1 and 2 in the main paper.

\begin{center}
\begin{tabular}{l}
  \texttt{\upshape $L\left(v\right) \gets L_0\left(v\right)$, $\forall v\in V_0$;}\\
  \texttt{\upshape $U\left(v\right) \gets U_0\left(v\right)$, $\forall v\in V_0$;}\\
  \texttt{\upshape $L\left(v\right) \gets 0$, $\forall v\in V_1\setminus V_0$;}\\
  \texttt{\upshape $U\left(v\right) \gets 1$, $\forall v\in V_1\setminus V_0$;}\\
  \texttt{\upshape do until convergence:}\\
  $\quad$ \texttt{select $v\in V_1$ and index $i\in\left\{1,\cdots,d_v\right\}$};\\
  $\quad$ \texttt{\upshape $L\left(v\right) \gets \max\left(L\left(v\right), \tilde{L}_{L,U,n_{v,i},j_{v,i}} \right)$, $\forall v\in V_1$;}\\
  $\quad$ \texttt{\upshape $U\left(v\right) \gets \min\left(U\left(v\right), \tilde{U}_{L,U,n_{v,i},j_{v,i}} \right)$, $\forall v\in V_1$;}\\
  \texttt{\upshape return $L\left(\sigma\right)$ and $U\left(\sigma\right)$;}\\
\end{tabular}
\end{center}

To prove Theorem 2, we will start with following lemma which serves as a steppingstone.

\newtheorem{lemma}{Lemma}
\begin{lemma}\label{lemma:1}
  For any functions $L:V_1 \rightarrow \left[ 0, 1 \right]$ and $U:V_1 \rightarrow \left[ 0, 1 \right]$, and for any $\phi\in V_1$ and any $i\in\left\{1,\cdots,d_\phi\right\}$, define $\Gamma=\left\{\left(v,L\left(v\right),U\left(v\right)\right)\mid v\in V_1\right\}$ and define $\Gamma_{\phi,i}^\prime =\left(\Gamma\setminus\left\{\left(\phi,L\left(\phi\right),U\left(\phi\right)\right)\right\}\right)\cup\left\{\left(\phi,\max\left(L\left(\phi\right),\tilde{L}_{L,U,n_{\phi,i},j_{\phi,i}}\right),\min\left(U\left(\phi\right),\tilde{U}_{L,U,n_{\phi,i},j_{\phi,i}}\right) \right)\right\}$. The following equality holds:
  \begin{equation}
    P_\Gamma = P_{\Gamma_{\phi,i}^\prime}
  \end{equation}
\end{lemma}

\begin{proof}[Proof of Lemma~\ref{lemma:1}]
  $\Gamma$ and $\Gamma_{\phi,i}^\prime$ differ by exactly one sentence: the former contains $\left(\phi,L\left(\phi\right),U\left(\phi\right)\right)$ while the latter contains $\left(\phi,\max\left(L\left(\phi\right),\tilde{L}_{L,U,n_{\phi,i},j_{\phi,i}}\right),\min\left(U\left(\phi\right),\tilde{U}_{L,U,n_{\phi,i},j_{\phi,i}}\right) \right)$.
  Since $L\left(\phi\right) \leq \max\left(L\left(\phi\right),\tilde{L}_{L,U,n_{\phi,i},j_{\phi,i}}\right)$ and $U\left(\phi\right) \geq \min\left(U\left(\phi\right),\tilde{U}_{L,U,n_{\phi,i},j_{\phi,i}}\right)$, by definition any model of $\Gamma_{\phi,i}^\prime$ must be a model of $\Gamma$.
  In other words, we have $P_{\Gamma_{\phi,i}^\prime} \subseteq P_\Gamma$.
  Therefore, in order to prove Lemma~\ref{lemma:1}, we only need to show $P_\Gamma \subseteq P_{\Gamma_{\phi,i}^\prime}$.
  It is trivially true if $P_\Gamma=\emptyset$, and therefore it suffices to show that $p \in P_{\Gamma_{\phi,i}^\prime}$ for any $p \in P_\Gamma$.
  By definition, it is equivalent to show the following two inequalities for any $p \in P_\Gamma$:
  \begin{align}
    p\left(S_\phi\right) & \leq \tilde{U}_{L,U,n_{\phi,i},j_{\phi,i}} \label{eq:lemupper} \\
    p\left(S_\phi\right) & \geq \tilde{L}_{L,U,n_{\phi,i},j_{\phi,i}} \label{eq:lemlower}
  \end{align}

  Because the right-hand side of (\ref{eq:lemupper}) and (\ref{eq:lemlower}) may come from any of (\ref{eq:notl1})--(\ref{eq:impu3}), we'll now prove (\ref{eq:lemupper}) or (\ref{eq:lemlower}) for all eighteen possibilities:
  \begin{itemize}

  \item
    The right-hand side of (\ref{eq:lemlower}) comes from (\ref{eq:notl1}) or (\ref{eq:notl2}).
    By definition, $n_{\phi,i}$ is a node of type NOT, and let $\gamma\triangleq n_{n_{\phi,i},3-j_{\phi,i}}$ denote the other neighbor of $n_{\phi,i}$.
    By definition, we have $\gamma=\neg\phi$, $\tilde{L}_{L,U,n_{\phi,i},j_{\phi,i}} = 1-U\left(\gamma\right)$, and $p\left(S_\gamma\right)\leq U\left(\gamma\right)$.
    Therefore,
    \begin{equation}
      p\left(S_\phi\right) = p\left(S_{\neg\gamma}\right) = 1-p\left(S_\gamma\right) \geq 1-U\left(\gamma\right) = \tilde{L}_{L,U,n_{\phi,i},j_{\phi,i}}
    \end{equation}
    Hence (\ref{eq:lemlower}) is true in this scenario.

  \item
    The right-hand side of (\ref{eq:lemupper}) comes from (\ref{eq:notu1}) or (\ref{eq:notu2}).
    By definition, $n_{\phi,i}$ is a node of type NOT, and let $\gamma\triangleq n_{n_{\phi,i},3-j_{\phi,i}}$ denote the other neighbor of $n_{\phi,i}$.
    By definition, we have $\gamma=\neg\phi$, $\tilde{U}_{L,U,n_{\phi,i},j_{\phi,i}} = 1-L\left(\gamma\right)$, and $p\left(S_\gamma\right)\geq L\left(\gamma\right)$.
    Therefore,
    \begin{equation}
      p\left(S_\phi\right) = p\left(S_{\neg\gamma}\right) = 1-p\left(S_\gamma\right) \leq 1-L\left(\gamma\right) = \tilde{U}_{L,U,n_{\phi,i},j_{\phi,i}}
    \end{equation}
    Hence (\ref{eq:lemupper}) is true in this scenario.

  \item
    The right-hand side of (\ref{eq:lemlower}) comes from (\ref{eq:andl1}).
    By definition, $n_{\phi,i}$ is a node of type AND-$k$, and let $\gamma\triangleq n_{n_{\phi,i},k+1}$ and $\xi_1,\cdots,\xi_{k-1}$ denote the other neighbors of $n_{\phi,i}$.
    By definition, we have $\gamma=\phi\land\xi_1\land\cdots\land\xi_{k-1}$, $\tilde{L}_{L,U,n_{\phi,i},j_{\phi,i}} = L\left(\gamma\right)$, and $p\left(S_\gamma\right) \geq L\left(\gamma\right)$.
    It is straightforward to verify that $S_\gamma \subseteq S_\phi$.
    Therefore,
    \begin{equation}
      p\left(S_\phi\right) \geq p\left(S_\gamma\right) \geq L\left(\gamma\right) = \tilde{L}_{L,U,n_{\phi,i},j_{\phi,i}}
    \end{equation}
    Hence (\ref{eq:lemlower}) is true in this scenario.

  \item
    The right-hand side of (\ref{eq:lemupper}) comes from (\ref{eq:andu1}).
    By definition, $n_{\phi,i}$ is a node of type AND-$k$, and let $\gamma\triangleq n_{n_{\phi,i},k+1}$ and $\xi_1,\cdots,\xi_{k-1}$ denote the other neighbors of $n_{\phi,i}$.
    By definition, we have $\gamma=\phi\land\xi_1\land\cdots\land\xi_{k-1}$, $\tilde{U}_{L,U,n_{\phi,i},j_{\phi,i}} = \min\left(1,U\left(\gamma\right)+\left(1-L\left(\xi_1\right)\right)+\cdots+\left(1-L\left(\xi_{k-1}\right)\right)\right)$, $p\left(S_\gamma\right) \leq U\left(\gamma\right)$, and $p\left(S_{\xi_j}\right) \geq L\left(\xi_j\right)$.
    It is straightforward to verify that
    \begin{equation}
      S_\gamma = S_\phi \setminus S_{\neg\xi_1} \setminus \cdots \setminus S_{\neg\xi_{k-1}}
    \end{equation}
    Therefore,
    \begin{equation}
      S_\phi \subseteq S_\gamma \cup S_{\neg\xi_1} \cup \cdots \cup S_{\neg\xi_{k-1}}
    \end{equation}
    Therefore,
    \begin{align}
      p\left(S_\phi\right) & \leq p\left(S_\gamma\right) + p\left(S_{\neg\xi_1}\right) + \cdots + p\left(S_{\neg\xi_{k-1}}\right) \nonumber\\
      & = p\left(S_\gamma\right) + \left(1-p\left(S_{\xi_1}\right)\right) + \cdots + \left(1-p\left(S_{\xi_{k-1}}\right)\right) \nonumber\\
      & \leq U\left(\gamma\right) + \left(1-L\left(\xi_1\right)\right) + \cdots + \left(1-L\left(\xi_{k-1}\right)\right)
    \end{align}
    Since $p\left(S_\phi\right) \leq 1$, it must be true that
    \begin{equation}
      p\left(S_\phi\right) \leq \min\left(1,U\left(\gamma\right)+\left(1-L\left(\xi_1\right)\right)+\cdots+\left(1-L\left(\xi_{k-1}\right)\right)\right) = \tilde{U}_{L,U,n_{\phi,i},j_{\phi,i}}
    \end{equation}
    Hence (\ref{eq:lemupper}) is true in this scenario.

  \item
    The right-hand side of (\ref{eq:lemlower}) comes from (\ref{eq:andl2}).
    By definition, $n_{\phi,i}$ is a node of type AND-$k$, and let $\xi_1,\cdots,\xi_k$ denote the other neighbors of $n_{\phi,i}$.
    By definition, we have $\phi=\xi_1\land\cdots\land\xi_k$, $\tilde{L}_{L,U,n_{\phi,i},j_{\phi,i}} = \max\left(0,1-\left(1-L\left(\xi_1\right)\right)-\cdots-\left(1-L\left(\xi_k\right)\right)\right)$, and $p\left(S_{\xi_j}\right) \geq L\left(\xi_j\right)$.
    It is straightforward to verify that
    \begin{equation}
      S_{\neg\phi} = S_{\neg\xi_1} \cup \cdots \cup S_{\neg\xi_k}
    \end{equation}
    Therefore,
    \begin{align}
      p\left(S_{\neg\phi}\right) & \leq p\left(S_{\neg\xi_1}\right) + \cdots + p\left(S_{\neg\xi_k}\right) \nonumber\\
      & \leq \left(1-L\left(\xi_1\right)\right) + \cdots + \left(1-L\left(\xi_k\right)\right)
    \end{align}
    Therefore,
    \begin{equation}
      p\left(S_\phi\right) = 1-p\left(S_{\neg\phi}\right) \geq 1-\left(1-L\left(\xi_1\right)\right) - \cdots - \left(1-L\left(\xi_k\right)\right)
    \end{equation}
    Since $p\left(S_\phi\right) \geq 0$, it must be true that
    \begin{equation}
      p\left(S_\phi\right) \geq \max\left(0,1-\left(1-L\left(\xi_1\right)\right)-\cdots-\left(1-L\left(\xi_k\right)\right)\right) = \tilde{L}_{L,U,n_{\phi,i},j_{\phi,i}}
    \end{equation}
    Hence (\ref{eq:lemlower}) is true in this scenario.

  \item
    The right-hand side of (\ref{eq:lemupper}) comes from (\ref{eq:andu2}).
    By definition, $n_{\phi,i}$ is a node of type AND-$k$, and let $\xi_1,\cdots,\xi_k$ denote the other neighbors of $n_{\phi,i}$.
    By definition, we have $\phi=\xi_1\land\cdots\land\xi_k$, $\tilde{U}_{L,U,n_{\phi,i},j_{\phi,i}} = \min_{j=1}^{k}U\left(\xi_j\right)$, and $p\left(S_{\xi_j}\right) \leq U\left(\xi_j\right)$.
    It is straightforward to verify that
    \begin{equation}
      S_\phi = S_{\xi_1} \cap \cdots \cap S_{\xi_k}
    \end{equation}
    Therefore,
    \begin{equation}
      p\left(S_\phi\right) \leq \min_{j=1}^{k}p\left(S_{\xi_j}\right) \leq \min_{j=1}^{k}U\left(\xi_j\right) = \tilde{U}_{L,U,n_{\phi,i},j_{\phi,i}}
    \end{equation}
    Hence (\ref{eq:lemupper}) is true in this scenario.

  \item
    The right-hand side of (\ref{eq:lemlower}) comes from (\ref{eq:orl1}).
    By definition, $n_{\phi,i}$ is a node of type OR-$k$, and let $\gamma\triangleq n_{n_{\phi,i},k+1}$ and $\xi_1,\cdots,\xi_{k-1}$ denote the other neighbors of $n_{\phi,i}$.
    By definition, we have $\gamma=\phi\lor\xi_1\lor\cdots\lor\xi_{k-1}$, $\tilde{L}_{L,U,n_{\phi,i},j_{\phi,i}} = \max\left(0,L\left(\gamma\right)-U\left(\xi_1\right)-\cdots-U\left(\xi_{k-1}\right)\right)$, $p\left(S_\gamma\right) \geq L\left(\gamma\right)$, and $p\left(S_{\xi_j}\right) \leq U\left(\xi_j\right)$.
    It is straightforward to verify that
    \begin{equation}
      S_{\neg\gamma} = S_{\neg\phi} \setminus S_{\xi_1} \setminus \cdots \setminus S_{\xi_{k-1}}
    \end{equation}
    Therefore,
    \begin{equation}
      S_{\neg\phi} \subseteq S_{\neg\gamma} \cup S_{\xi_1} \cup \cdots \cup S_{\xi_{k-1}}
    \end{equation}
    Therefore,
    \begin{align}
      p\left(S_{\neg\phi}\right) & \leq p\left(S_{\neg\gamma}\right) + p\left(S_{\xi_1}\right) + \cdots + p\left(S_{\xi_{k-1}}\right) \nonumber\\
      & \leq 1-L\left(\gamma\right) + U\left(\xi_1\right) + \cdots + U\left(\xi_{k-1}\right)
    \end{align}
    Therefore,
    \begin{equation}
      p\left(S_\phi\right) = 1-p\left(S_{\neg\phi}\right) \geq L\left(\gamma\right) - U\left(\xi_1\right) - \cdots - U\left(\xi_{k-1}\right)
    \end{equation}
    Since $p\left(S_\phi\right) \geq 0$, it must be true that
    \begin{equation}
      p\left(S_\phi\right) \geq \max\left(0,L\left(\gamma\right)-U\left(\xi_1\right)-\cdots-U\left(\xi_{k-1}\right)\right) = \tilde{L}_{L,U,n_{\phi,i},j_{\phi,i}}
    \end{equation}
    Hence (\ref{eq:lemlower}) is true in this scenario.

  \item
    The right-hand side of (\ref{eq:lemupper}) comes from (\ref{eq:oru1}).
    By definition, $n_{\phi,i}$ is a node of type OR-$k$, and let $\gamma\triangleq n_{n_{\phi,i},k+1}$ and $\xi_1,\cdots,\xi_{k-1}$ denote the other neighbors of $n_{\phi,i}$.
    By definition, we have $\gamma=\phi\lor\xi_1\lor\cdots\lor\xi_{k-1}$, $\tilde{U}_{L,U,n_{\phi,i},j_{\phi,i}} = U\left(\gamma\right)$, and $p\left(S_\gamma\right) \leq U\left(\gamma\right)$.
    It is straightforward to verify that $S_\phi \subseteq S_\gamma$.
    Therefore,
    \begin{equation}
      p\left(S_\phi\right) \leq p\left(S_\gamma\right) \leq U\left(\gamma\right) = \tilde{U}_{L,U,n_{\phi,i},j_{\phi,i}}
    \end{equation}
    Hence (\ref{eq:lemupper}) is true in this scenario.

  \item
    The right-hand side of (\ref{eq:lemlower}) comes from (\ref{eq:orl2}).
    By definition, $n_{\phi,i}$ is a node of type OR-$k$, and let $\xi_1,\cdots,\xi_k$ denote the other neighbors of $n_{\phi,i}$.
    By definition, we have $\phi=\xi_1\lor\cdots\lor\xi_k$, $\tilde{L}_{L,U,n_{\phi,i},j_{\phi,i}} = \max_{j=1}^{k}L\left(\xi_j\right)$, and $p\left(S_{\xi_j}\right) \geq L\left(\xi_j\right)$.
    It is straightforward to verify that
    \begin{equation}
      S_\phi = S_{\xi_1} \cup \cdots \cup S_{\xi_k}
    \end{equation}
    Therefore,
    \begin{equation}
      p\left(S_\phi\right) \geq \max_{j=1}^{k}p\left(S_{\xi_j}\right) \geq \max_{j=1}^{k}L\left(\xi_j\right) = \tilde{L}_{L,U,n_{\phi,i},j_{\phi,i}}
    \end{equation}
    Hence (\ref{eq:lemlower}) is true in this scenario.

  \item
    The right-hand side of (\ref{eq:lemupper}) comes from (\ref{eq:oru2}).
    By definition, $n_{\phi,i}$ is a node of type OR-$k$, and let $\xi_1,\cdots,\xi_k$ denote the other neighbors of $n_{\phi,i}$.
    By definition, we have $\phi=\xi_1\lor\cdots\lor\xi_k$, $\tilde{U}_{L,U,n_{\phi,i},j_{\phi,i}} = \min\left(1,U\left(\xi_1\right)+\cdots+U\left(\xi_k\right)\right)$, and $p\left(S_{\xi_j}\right) \leq U\left(\xi_j\right)$.
    It is straightforward to verify that
    \begin{equation}
      S_\phi = S_{\xi_1} \cup \cdots \cup S_{\xi_k}
    \end{equation}
    Therefore,
    \begin{equation}
      p\left(S_\phi\right) \leq p\left(S_{\xi_1}\right) + \cdots + p\left(S_{\xi_k}\right) \leq U\left(\xi_1\right) + \cdots + U\left(\xi_k\right)
    \end{equation}
    Since $p\left(S_\phi\right) \leq 1$, it must be true that
    \begin{equation}
      p\left(S_\phi\right) \leq \min\left(1,U\left(\xi_1\right)+\cdots+U\left(\xi_k\right)\right) = \tilde{U}_{L,U,n_{\phi,i},j_{\phi,i}}
    \end{equation}
    Hence (\ref{eq:lemupper}) is true in this scenario.

  \item
    The right-hand side of (\ref{eq:lemlower}) comes from (\ref{eq:impl1}).
    By definition, $n_{\phi,i}$ is a node of type IMPLIES, and let $\gamma\triangleq n_{n_{\phi,i},3}$ and $\xi\triangleq n_{n_{\phi,i},2}$ denote the other neighbors of $n_{\phi,i}$.
    By definition, we have $\gamma=\phi\rightarrow\xi$, $\tilde{L}_{L,U,n_{\phi,i},j_{\phi,i}} = 1-U\left(\gamma\right)$, and $p\left(S_\gamma\right) \leq U\left(\gamma\right)$.
    It is straightforward to verify that $S_{\neg\phi} \subseteq S_\gamma$.
    Therefore,
    \begin{equation}
      p\left(S_\phi\right) = 1-p\left(S_{\neg\phi}\right) \geq 1-p\left(S_\gamma\right) \geq 1-U\left(\gamma\right) = \tilde{L}_{L,U,n_{\phi,i},j_{\phi,i}}
    \end{equation}
    Hence (\ref{eq:lemlower}) is true in this scenario.

  \item
    The right-hand side of (\ref{eq:lemupper}) comes from (\ref{eq:impu1}).
    By definition, $n_{\phi,i}$ is a node of type IMPLIES, and let $\gamma\triangleq n_{n_{\phi,i},3}$ and $\xi\triangleq n_{n_{\phi,i},2}$ denote the other neighbors of $n_{\phi,i}$.
    By definition, we have $\gamma=\phi\rightarrow\xi$, $\tilde{U}_{L,U,n_{\phi,i},j_{\phi,i}} = \min\left(1,1+U\left(\xi\right)-L\left(\gamma\right)\right)$, $p\left(S_\gamma\right) \geq L\left(\gamma\right)$, and $p\left(S_\xi\right) \leq U\left(\xi\right)$.
    It is straightforward to verify that
    \begin{equation}
      S_{\neg\gamma} = S_\phi \setminus S_\xi
    \end{equation}
    Therefore,
    \begin{equation}
      S_\phi \subseteq S_{\neg\gamma} \cup S_\xi
    \end{equation}
    Therefore,
    \begin{equation}
      p\left(S_\phi\right) \leq p\left(S_{\neg\gamma}\right) + p\left(S_\xi\right) \nonumber\\
      = 1 - p\left(S_\gamma\right) + p\left(S_\xi\right) \nonumber\\
      \leq 1-L\left(\gamma\right) + U\left(\xi\right)
    \end{equation}
    Since $p\left(S_\phi\right) \leq 1$, it must be true that
    \begin{equation}
      p\left(S_\phi\right) \leq \min\left(1,1-L\left(\gamma\right)+U\left(\xi\right)\right) = \tilde{U}_{L,U,n_{\phi,i},j_{\phi,i}}
    \end{equation}
    Hence (\ref{eq:lemupper}) is true in this scenario.

  \item
    The right-hand side of (\ref{eq:lemlower}) comes from (\ref{eq:impl2}).
    By definition, $n_{\phi,i}$ is a node of type IMPLIES, and let $\gamma\triangleq n_{n_{\phi,i},3}$ and $\xi\triangleq n_{n_{\phi,i},1}$ denote the other neighbors of $n_{\phi,i}$.
    By definition, we have $\gamma=\xi\rightarrow\phi$, $\tilde{L}_{L,U,n_{\phi,i},j_{\phi,i}} = \max\left(0,L\left(\xi\right)+L\left(\gamma\right)-1\right)$, $p\left(S_\gamma\right) \geq L\left(\gamma\right)$, and $p\left(S_\xi\right) \geq L\left(\xi\right)$.
    It is straightforward to verify that
    \begin{equation}
      S_{\neg\gamma} = S_{\neg\phi} \setminus S_{\neg\xi} 
    \end{equation}
    Therefore,
    \begin{equation}
      S_{\neg\phi} \subseteq S_{\neg\xi} \cup S_{\neg\gamma}
    \end{equation}
    Therefore,
    \begin{align}
      p\left(S_{\neg\phi}\right) & \leq p\left(S_{\neg\xi}\right) + p\left(S_{\neg\gamma}\right) \nonumber\\
      & = 1-p\left(S_\xi\right) + 1-p\left(S_\gamma\right) \nonumber\\
      & \leq 2-L\left(\xi\right)-L\left(\gamma\right)
    \end{align}
    Therefore,
    \begin{equation}
      p\left(S_\phi\right) = 1-p\left(S_{\neg\phi}\right) \geq L\left(\xi\right)+L\left(\gamma\right)-1
    \end{equation}
    Since $p\left(S_\phi\right) \geq 0$, it must be true that
    \begin{equation}
      p\left(S_\phi\right) \geq \max\left(0,L\left(\xi\right)+L\left(\gamma\right)-1\right) = \tilde{L}_{L,U,n_{\phi,i},j_{\phi,i}}
    \end{equation}
    Hence (\ref{eq:lemlower}) is true in this scenario.

  \item
    The right-hand side of (\ref{eq:lemupper}) comes from (\ref{eq:impu2}).
    By definition, $n_{\phi,i}$ is a node of type IMPLIES, and let $\gamma\triangleq n_{n_{\phi,i},3}$ and $\xi\triangleq n_{n_{\phi,i},1}$ denote the other neighbors of $n_{\phi,i}$.
    By definition, we have $\gamma=\xi\rightarrow\phi$, $\tilde{U}_{L,U,n_{\phi,i},j_{\phi,i}} = U\left(\gamma\right)$, and $p\left(S_\gamma\right) \leq U\left(\gamma\right)$.
    It is straightforward to verify that $S_\phi \subseteq S_\gamma$.
    Therefore,
    \begin{equation}
      p\left(S_\phi\right) \leq p\left(S_\gamma\right) \leq U\left(\gamma\right) = \tilde{U}_{L,U,n_{\phi,i},j_{\phi,i}}
    \end{equation}
    Hence (\ref{eq:lemupper}) is true in this scenario.

  \item
    The right-hand side of (\ref{eq:lemlower}) comes from (\ref{eq:impl3}).
    By definition, $n_{\phi,i}$ is a node of type IMPLIES, and let $\gamma\triangleq n_{n_{\phi,i},2}$ and $\xi\triangleq n_{n_{\phi,i},1}$ denote the other neighbors of $n_{\phi,i}$.
    By definition, we have $\phi=\xi\rightarrow\gamma$, $\tilde{L}_{L,U,n_{\phi,i},j_{\phi,i}} = \max\left(1-U\left(\xi\right),L\left(\gamma\right)\right)$, $p\left(S_\xi\right) \leq U\left(\xi\right)$, and $p\left(S_\gamma\right) \geq L\left(\gamma\right)$.
    It is straightforward to verify that
    \begin{equation}
      S_\phi = S_{\neg\xi} \cup S_\gamma
    \end{equation}
    Therefore,
    \begin{align}
      p\left(S_\phi\right) & \geq \max\left(p\left(S_{\neg\xi}\right), p\left(S_\gamma\right)\right) \nonumber\\
      & = \max\left(1-p\left(S_\xi\right), p\left(S_\gamma\right)\right) \nonumber\\
      & \geq \max\left(1-U\left(\xi\right),L\left(\gamma\right)\right) \nonumber\\
      & = \tilde{L}_{L,U,n_{\phi,i},j_{\phi,i}}
    \end{align}
    Hence (\ref{eq:lemlower}) is true in this scenario.

  \item
    The right-hand side of (\ref{eq:lemupper}) comes from (\ref{eq:impu3}).
    By definition, $n_{\phi,i}$ is a node of type IMPLIES, and let $\gamma\triangleq n_{n_{\phi,i},2}$ and $\xi\triangleq n_{n_{\phi,i},1}$ denote the other neighbors of $n_{\phi,i}$.
    By definition, we have $\phi=\xi\rightarrow\gamma$, $\tilde{U}_{L,U,n_{\phi,i},j_{\phi,i}} = \min\left(1,1-L\left(\xi\right)+U\left(\gamma\right)\right)$, $p\left(S_\xi\right) \geq L\left(\xi\right)$, and $p\left(S_\gamma\right) \leq U\left(\gamma\right)$.
    It is straightforward to verify that
    \begin{equation}
      S_\phi = S_{\neg\xi} \cup S_\gamma
    \end{equation}
    Therefore,
    \begin{align}
      p\left(S_\phi\right) & \leq p\left(S_{\neg\xi}\right) + p\left(S_\gamma\right) \nonumber\\
      & = 1-p\left(S_\xi\right) + p\left(S_\gamma\right) \nonumber\\
      & \leq 1-L\left(\xi\right) + U\left(\gamma\right)
    \end{align}
    Since $p\left(S_\phi\right) \leq 1$, it must be true that
    \begin{equation}
      p\left(S_\phi\right) \leq \min\left(1,1-L\left(\xi\right)+U\left(\gamma\right)\right) = \tilde{U}_{L,U,n_{\phi,i},j_{\phi,i}}
    \end{equation}
    Hence (\ref{eq:lemupper}) is true in this scenario.
  \end{itemize}
\end{proof}

Now we are ready to present the main proof.

\begin{proof}[Proof of Theorem 2]
  Let $L^k$ and $U^k$ denote the $L$ and $U$ functions in the pseudo code after $k$ iterations.
  Define $\Gamma^k=\left\{\left(v,L^k\left(v\right),U^k\left(v\right)\right)\mid v\in V_1\right\}$ for $k=0,1,\cdots$.
  By the pseudo code, it is straightforward to verify that $P_{\Gamma^0} = P_{\Gamma_0}$.
  By Lemma~\ref{lemma:1}, we have $P_{\Gamma^k} = P_{\Gamma^{k+1}}$ for any $k$.
  Therefore, after convergence it must be true that $P_{\Gamma^k} = P_{\Gamma_0}$.

  By definition, for any $p\in P_{\Gamma^k}$ after convergence, we have
  \begin{equation}
    L_\sigma = L^k\left(\sigma\right) \leq p\left(S_\sigma\right) \leq U^k\left(\sigma\right) = U_\sigma
  \end{equation}
  Replacing $P_{\Gamma^k}$ with $P_{\Gamma_0}$, we get the two inequalities in Theorem 2.
\end{proof}

\section{Learning with constraints}
g
Here we present a slightly expanded discussion of the constrained optimization problem discussed in Section~\ref{sec:learning}.

\subsection{Constraints}
\label{sec:constraints}
Constraints on neural parameters are derived from the truth tables of the operations they intend to model and from established ranges for ``true'' and ``false'' values.
Given a threshold of truth~\(\frac{1}{2} < \alpha \leq 1\), a continuous truth value is considered true if it is greater than \(\alpha\) and false if it is less than \(1 - \alpha\).
Accordingly, the truth table for, e.g., binary AND suggests a set of constraints given
\begin{equation*}\begin{array}{ccccccccl@{}l@{}l@{}l}
p  & q  & p \land q &             & p          & q          & p \land q  &             & \beta - w_p \cdot (1 - p)      & {}- w_q \cdot (1 - q) \\[1pt]
\cline{1-3}\cline{5-7}\cline{9-10}
\F & \F & \F        &             & 1 - \alpha & 1 - \alpha & 1 - \alpha &             & \beta - w_p \cdot \alpha       & {}- w_q \cdot \alpha       & {}\leq 1 - \alpha \rule{0pt}{10pt} \\
\F & \T & \F        & \rightarrow & 1 - \alpha & 1          & 1 - \alpha & \rightarrow & \beta - w_p \cdot \alpha       &                            & {}\leq 1 - \alpha \\
\T & \F & \F        &             & 1          & 1 - \alpha & 1 - \alpha &             & \beta                          & {}- w_q \cdot \alpha       & {}\leq 1 - \alpha \\
\T & \T & \T        &             & \alpha     & \alpha     & \alpha     &             & \beta - w_p \cdot (1 - \alpha) & {}- w_q \cdot (1 - \alpha) & {}\geq \alpha
\end{array}\end{equation*}
More generally, \(n\)-ary conjunctions have constraints of the form
\begin{alignat}{4}
& \forall i \in I,\quad &           &&                            w_i &                    && \geq 0 \notag \\
& \forall i \in I,\quad & \beta -{} &&                            w_i & \cdot \alpha       && \leq 1 - \alpha \label{eqn:and-0*} \\
&                       & \beta -{} && \textstyle \sum_{i \in I}  w_i & \cdot (1 - \alpha) && \geq \alpha \label{eqn:and-1}
\intertext{while \(n\)-ary disjunctions have constraints of the form}
& \forall i \in I,\quad &               &&                            w_i &                    && \geq 0 \notag \\
& \forall i \in I,\quad & 1 - \beta +{} &&                            w_i & \cdot \alpha       && \geq \alpha \label{eqn:or-1*} \\
&                       & 1 - \beta +{} && \textstyle \sum_{i \in I}  w_i & \cdot (1 - \alpha) && \leq 1 - \alpha
\end{alignat}
The identity \(p^{(w_p)} \rightarrow^{\beta} q^{(w_q)} = (\neg p)^{(w_p)} \oplus^{\beta} q^{(w_q)}\) permits implications to use the same constraints as disjunctions and, in fact, the above two sets of constraints are equivalent under the De Morgan laws.
A consequence of these constraints is that LNN evaluation is guaranteed to behave classically, i.e.~to yield results at every neuron within the established ranges for true and false, if all of their inputs are themselves within these ranges.

\subsection{Slack variables}

It is easy to see that weights~\(w_i\) cannot equal 0 under the above constraints, though this is a desirable outcome of optimization as it effectively removes the affected input.
To permit this, it is necessary to introduce a slack variable for each weight, allowing its respective constraints in \eqref{eqn:and-0*} or \eqref{eqn:or-1*} to be violated as the weight drops to 0:
\begin{alignat}{2}
& \forall i \in I,\quad &                                s_i & \geq 0 \notag \\
& \forall i \in I,\quad &     \beta - w_i \cdot \alpha - s_i & \leq 1 - \alpha \tag{\ref{eqn:and-0*}*} \\
& \forall i \in I,\quad & 1 - \beta + w_i \cdot \alpha + s_i & \geq \alpha \tag{\ref{eqn:or-1*}*}
\end{alignat}
If \(w_i = 0\) is understood as \(i \notin I\), this update remains consistent with the original constraint if either \(s_i = 0\) or \(w_i = 0\).
One can encourage optimization to favor such parameterizations by included in the loss function a penalty term that scales with \(s_i w_i\).
The coefficient on this penalty term controls how classical learned operations must be, with exact classical behavior restored if optimization reduces the penalty term to 0.

\subsection{Optimization problem}

The optimization problem from before is then updated
\begin{align*}
    \textstyle \min_{B, W, S}\quad & \textstyle E(B, W) + \mathrlap{\sum_{k \in N} \max\{0, L_{B,W,k} - U_{B,W,k}\} + \sum_{k \in N} \mathbf{s_k} \cdot \mathbf{w_k}} \\
    \text{s.t.}\quad & \forall k \in N,\; i \in I_k, & \alpha \cdot w_{ik} - s_{ik} - \beta_k + 1 & \geq \alpha, & w_{ik}, s_{ik} & \geq 0 \\
    & \forall k \in N,                      & \textstyle \sum_{i \in I_k} (1 - \alpha) \cdot w_{ik} - \beta_k + 1 & \leq 1 - \alpha, & \beta_k & \geq 0
\end{align*}
for loss function~\(E\), bias vector~\(B\), weight matrix~\(W\), slack matris~\(S\), neuron index set~\(N\), and inferred lower and upper bounds~\(L_{B,W,k}\) and \(U_{B,W,k}\) at each neuron.
In addition, it may be of interest to square either or both of the contradiction penalty terms~\(\max\{0, L_{B,W,k} - U_{B,W,k}\}\) or the slack penalty terms~\(s_i w_i\).

Depending on the specific problem being solved, different loss functions~\(E(B, W)\) may be used.
For example, an LNN configured to predict a binary outcome may use mean squared error as usual.
Alternately, it is possible to use the contradiction penalty to build arbitrarily complex logical loss functions by introducing new formulae into model that become contradictory in the event of undesirable inference behavior.
Understandably, the parameters of specifically these introduced formulae should not be tuned but instead left in a default state (e.g.~all 1), so optimization cannot turn the logical loss function off.

\section{Tailored activation function}

\subsection{Motivation}

In weighted fuzzy logic, a logical operation on inputs is carried out by a neuron activation function, in almost the same way that artificial neurons compute on their inputs. The main difference is that the weights and bias terms are constrained in such a way that the semantics of the operation are maintained. For example, Equation~\eqref{eqn:or-1} describes how at least one input to a disjunction neuron may be \texttt{True} with the rest \texttt{False}, results in the activation function returning a value less than or equal to \(\alpha\), representing classically \texttt{True}.

The constraining of weights and biases, ensures that while parameters can be adjusted (partially), the neuron retains its logical character and its interpretability.

Controlling the allowed weights and biases of the inputs enforces the correct logical operation after applying the static activation function. These constraints define what the activation function is going to do the weighted sum.

A more natural way to control the output of the activation function is to change the activation function itself and in this way enforce the right logical output. This approach has no constraints on the weights and all the enforcement is effected by having an activation function that depends on the weights themselves. We call this the `tailored activation function' approach. It builds on the weighted fuzzy logic approach and is in-fact equivalent to weighted \L ukasiewicz logic in the \(\alpha=1\), equally-weighted case.

There are many advantages to the tailored activation function approach. To unpack this let us re-look at the problems with the constrained approach:

\begin{itemize}
\item The constrained approach dictates that the activation function must be \(\alpha\)-preserving to provide interpretability on the domain. This fixes two static points that the activation function must go through, prematurely giving static truth meaning to the values of the weighted sum.
\item Additional slack parameters are required to relax constraints as weights approach zero --- hindering interpretability and encouraging over-fitting.
\item Parameter updates require a constraint satisfaction algorithm --- such as Frank Wolfe or Projected Gradient Descent.
\item Constraint satisfaction and neuron operations are expensive, with slacks required for each weight and logical operations defined \(\forall i\) \(w_i, s_i\) in \text{(\ref{eqn:and-0*}*), (\ref{eqn:or-1*}*)}.
\item Unbounded weights become unwieldy and less interpretable.
\item There are no gradients in the clamped regions without gradient transparent clamping, as in Appendix~\ref{apndx:Gradient-transparent clamping}.
\item The \(w_i \cdot (1 - x_i)\) form of a weighted conjunction in \eqref{eq:lukasiewicz_and_forward} prevents maximally \texttt{True} inputs \((x_i=1\)) from offering gradients --- this is problematic for contradictory losses that exist when given a \texttt{False} conjunction that has all \texttt{True} inputs. This complication extends to disjunction and implication neurons accordingly.
\item \(\alpha\) is not interpretable especially in choosing its value.
\item \(\beta\) is not easily interpretable.
\item \(\beta\) must be learnt for each neuron --- which encourages over-fitting.
\item downward inference from functional inverses compute classically incorrect bounds.
\end{itemize}

\subsection{Bounds}

Here we present the expanded table of LNN bounds presented in section~\ref{sec:model} to better elaborate on the tailored activation formulation.

The interpretation of lower and upper bound truth values correspond to one of 4 primary or 3 secondary states that a neuron can be in:

\begin{table}[!htbp]
\centering
\caption{Primary and secondary truth value bound states}
\begin{tabular}{ l l l l | l l l }
  \toprule
  \multicolumn{4}{c|}{Primary} & \multicolumn{3}{c}{Secondary}\\
  State &Lower & &Upper &State &Lower &Upper \\
  \midrule
  Unknown &\([ 0, 1 - \alpha ]\) & &\([ \alpha, 1 ]\) & \(\sim\)Unknown &\([ 0, 0.5 ]\) &\([ 0.5, 1 ]\) \\
  True &\([ \alpha, 1 ]\) & &\([ \alpha, 1 ]\) &\(\sim\)True &\(( 0.5, \alpha )\) &\(( 0.5, 1 ]\) \\
  False &\([ 0, 1 - \alpha )\) & &\([ 0, 1 - \alpha ]\) &\(\sim\)False &\([ 0, 0.5 )\) &\(( 1 - \alpha, 0.5 )\) \\
  Contradiction &Lower &\(>\) &Upper \\
  \bottomrule
\end{tabular}
\label{tab:bound-states-secondary}
\end{table}

Neurons with a \texttt{True} or \texttt{False} state may be considered classically-true or classically-false respectively. A \texttt{$\sim$True} or \texttt{$\sim$False} state is interpreted as being more-true-than-not or more-false-than-not, while still remaining open to a classically-true or classically-false convergence provided an upper bound above $\alpha$ or lower bound below $1-\alpha$ respectively. The \texttt{Unknown} state, albeit non-classical, may converge to a classical state, whereas \texttt{$\sim$Unknown} may converge to be \texttt{$\sim$ True} or \texttt{$\sim$False} but not to a classical state. A \texttt{Contradictory} state represents a disagreement in assertions of the truth between connected rule and fact neurons, with intra-classical contradictions typically being ignored.

Neurons are typically introduced into an LNN in an \texttt{Unknown} state, with allowance for rule or fact truth values to be directly assigned if known. When offered new bounds at inference, lower bounds monotonically tighten upward and upper bounds monotonically tighten downward.

\subsection{Definition}

One of the key insights of the tailored activation function approach is to define dynamic False, Fuzzy and True regions in the domain. Both approaches define static regions in the range. The constrained approach imposes static regions on the domain, which is why constraints become necessary.

The tailored activation function rather keeps track of the dynamic semantically-derived boundary points between the regions (i.e. from the classical truth table). The boundary between False and Fuzzy is called $x_F$ and the boundary between Fuzzy and True is called $x_T$. Then the tailored activation function is required to be monotonic and to go through these two dynamic points $(x_F, 1-\alpha)$ and $(x_T,\alpha)$. Monotonicity is not a new requirement and neither is the requirement of going through two points. What is new, is that the two points are dynamic and that the classically \texttt{True} region is more than a single point. This latter advantage is significant, since smooth activation functions such as the sigmoid function may be used without worrying about reaching a maximally true value, i.e.~$1$, to represent classically true ($\alpha < 1$).

The biggest advantage is that there are no constraints on the weights; the activation function, by definition, adjusts as the weights change, maintaining the logical semantics. The second major advantage is that there are useful gradients everywhere.

The price to pay for this, of course, is that these truth points depend on the weights of the operands, therefore every node requires its own activation function. 

Weighted inputs to the activation function are neural-network-like, in that they require a dot product of weight and bounds vectors of an $n$-ary input, $f_{\mathbf{w}}(\mathbf{w}\cdot \mathbf{x})$.   

In total there are four points of interest:

\begin{subequations}
\label{tailored:fourpoints}
\begin{align}
(x_{min}, y_{min}) & = (\beta -\sum_{i \in I}  w_i, 0)\\
(x_F, y_F) & = (\beta - \alpha w_{max} , 1-\alpha) \label{tailored:x_F} \\
(x_T, y_T) & = (\beta - \sum_{i \in I}  w_i \cdot (1 - \alpha) , \alpha)\label{tailored:x_T} \\
(x_{max}, y_{max}) &= (\beta , 1)
\end{align}
\end{subequations}

The monotonicity requirement does produce one non-trivial constraint related to the weights in combination with \(\alpha\), namely that $x_F \leq x_T$. Equality can be optionally ruled out, $x_F\ne x_T$ if a continuous activation function is desired. This ensures that there is a non-empty Fuzzy region interpolating between the regions of classicality. Enforcing this constraint, is easily achieved from the outset by choosing an appropriate \(\alpha\):

\begin{align}
    \frac{\sum_{i \in I} w_i}{\sum_{i \in I} w_i+w_{max}} < \alpha \le 1 \label{tailored:a_constraints}
\end{align}

The constraints of $x_{min}\le x_F$ and $x_T\le x_{max}$ are automatically enforced by definition. 

\begin{align}
    \forall i \in I, & \phantom{\sum_{i \in I}} 0 \le w_i \label{tailored:w_constraints} \\
    \forall i \in I, & \phantom{\sum_{i \in I}} w_i \le 1 \label{tailored:w_simplify}
\end{align}

Note, however, that \(\alpha\) has lower bound~\(n / (n + 1)\) for \(n\) the largest number of operands of any connective in the knowledge graph. To maximise learning gradients, this global \(\alpha\) can therefore be initialised to \(\alpha = \frac{1}{2}\) (the maximum classical range) and update alpha accordingly.

As an extension, it may be useful to have an \(\alpha\) learnable per node, allowing each node to keep a relative interpretation of the truth. This would then require that the constraint \eqref{tailored:a_constraints} be enforced at every parameter update, though the exploration of this is future work.

The required constraints \eqref{tailored:w_constraints} can be achieved by clamping at zero\footnote{though allowing weights to become negative may be useful as a form of negation of the operand. This, however, would come with extra complications. Getting this right in future work would have many advantages, including a generic neuron that can learn many gate configurations --- equivalent to a logical operation --- solely from the weights}.
There are other optional orthogonal constraints on the weights, \eqref{tailored:w_simplify}, that may be useful for simplifying interpretability, which we do employ: scaling all input weights by $w_{max}$ --- giving some effect to updates calling for an increase beyond one. 
Alternatively, we may simply clamp at zero and one --- ignoring gradients that take weights out of the range.

An analysis of the tailored activation function, reveals that the output is independent of $\beta$. Therefore it can be eliminated or for convenience we analytically define it as:
\begin{align*}
\beta_C = \sum_{i \in I} w_i
\end{align*}
simplifying \eqref{tailored:fourpoints} to:

\begin{subequations}
\begin{align}
(x_{min}, y_{min}) & = (0, 0)\\
(x_F, y_F) & = (\beta_C - \alpha w_{max} , 1-\alpha)\\
(x_T, y_T) & = (\sum_{i \in I}  w_i \cdot \alpha, \alpha)\\
(x_{max}, y_{max}) &= (\beta_C , 1)
\end{align}
\end{subequations}

This simplification anchors $(x_{min},x_{max}) =(0,\beta_C)$ while weights fluctuate and ensures functional inverses behaves equivalently $(y_{min},y_{max}) =(0,\beta_C)$. 

\subsection{Satisfying the classical truth table}

\subsubsection{Design choice concerning \texorpdfstring{$\mathbf{x_F}$}{x\_F}}

Both the weighted fuzzy logic and tailored activation approaches identify the boundaries in a similar way, namely, by consulting the classical truth table. However, having introduced weights to operands, there is an interesting design choice that arises: when weights are equal and truth values are classical, the classical truth table should be perfectly obeyed, but when the weights are unequal there is a choice as to where to place the boundary. Should classicality still be imposed? For example, the False-Fuzzy boundary point, $x_F$, of conjunction could be defined by insisting that the classical truth table hold when the lowest weighted operand is \texttt{False} and the rest \texttt{True} (i.e.~when $w_{min}$'s operand is False --- derived from \ref{eqn:and-0*}). But then the undesirable behaviour of a very low weighted operand having undue influence is observed. Rather, it is proposed to pay attention to only the largest weight, $w_{max}$.

Firstly, it is true that non-classical behaviour is exhibited for a \texttt{False}, lowly weighted input, but then this is desirable. In the constrained approach, this is achieved by learning slacks to switch off constraints or not, Equation~\text{\ref{eqn:and-0*}*}. In the tailored activation approach slacks are not needed and the proposed choice of $w_{max}$ forces non-classicality. This fulfils the expectation of what a low weight means, in spite of fewer parameters to over-fit. This is also good for interpretability, even if it suffers a slight reduction in performance. Afterall, LNN's are supposed to be logically-based and if formulae need to be augmented, it is aesthetically more appropriate to do so logically rather than with extra uninterpretable parameters.

Secondly, the tailored-activation approach provides a second justification for $w_{max}$, namely the correct behavior when weights go to zero. Indeed for an $n$-ary conjunction, if all but one weight goes to zero, the identity function is recovered for the remaining operand. When all operands have weights equal to zero, the operator is effectively removed from its parent's formula --- illustrated in Figure~\ref{fig:identity}.

\subsubsection{Classical truth tables of other logical operations beside conjuction}

The procedure of choosing the boundary points of the tailored activation approach to reflect the classical truth table could be repeated for each logical operation. However, it is more elegant to simply rely on the universality result of classical logic, namely that the set consisting of the unary Not and the binary And gates is universal for classical binary computation.

Through universality everything behaves as expected on classical inputs: \((x \rightarrow y) = (\neg x \oplus y)=\neg(x \otimes \neg y)\), and \(((x \rightarrow 0) \rightarrow 0) = x\). The properties of \((x \otimes x)=(x \oplus x)=x\), also hold for classical inputs.

This is contrast to the constrained approached, where implication is defined via the residuum and not the logical definition. The residuum procedure is a sufficient (though not necessary) condition to ensure upward modus-ponens operates as expected. With the tailored activation approach, since classicality is perfectly captured, upward modus-ponens for classical values is guaranteed to work. The expected behaviour for non-classical values must be checked and indeed it is also respected.

\subsubsection{Design choice concerning arity}
The above discussion about universality would suggest that the $n$-ary conjunction should simply be decomposed into nested binary conjunctions, thereby retaining all the familiar classical properties. However, there is no unique way of performing the decomposition and an $n$-ary form of the tailored-activation function is easily definable. This, then becomes a design choice that must be made, namely when encountering an $n$-ary conjunction, how should it be represented?

The disadvantage of the $n$-ary form is that it is not equivalent to the decomposed form, which is the form that is guaranteed to behave classically. Nevertheless, it is the form we experiment with, since it results in a simpler logical tree.

\subsection{Alpha interpretation – size of classical regions}

The choice of hyper-parameter $(\alpha)$ is an interpretable design decision that provides useful functionality for the tailored activation function. It defines the size of the classical regions and influences the rate of learning.

\begin{figure}[!htbp]
    \begin{subfigure}[b]{0.49\textwidth}
	    \centering
	    \includegraphics[width=\textwidth]{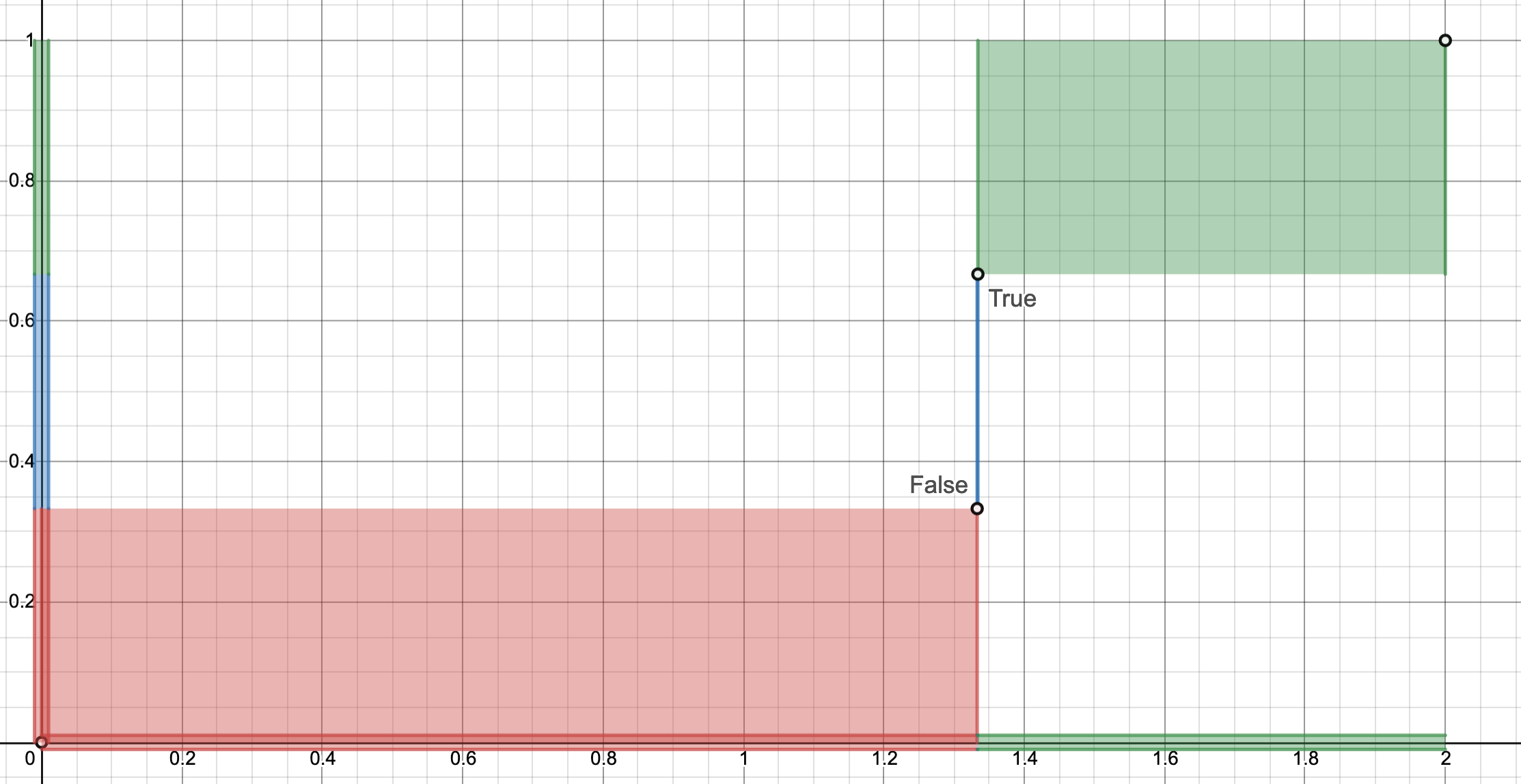}
	\end{subfigure}
	\hfill
	\begin{subfigure}[b]{0.49\textwidth}
	    \centering
		 \includegraphics[width=\textwidth]{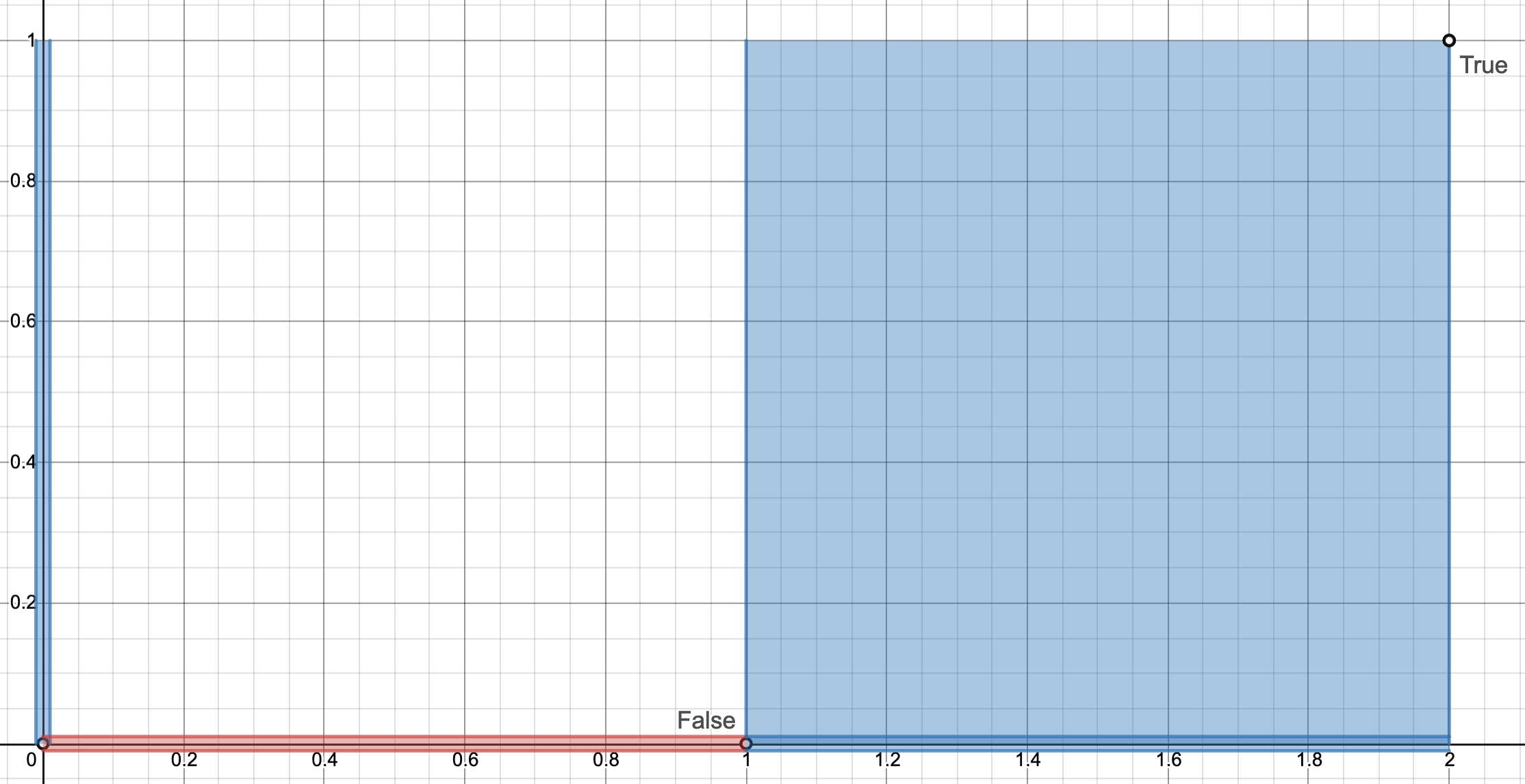}
	\end{subfigure}
    \caption{Regions of classicality as alpha varies; 
    left - minimum alpha, right - maximum alpha}
    \label{fig:regions_alpha}
\end{figure}

At the extrema of equation \eqref{tailored:a_constraints}, the choice of alpha defines where the database designer places their trust --- the facts or the rules. At its maximum \((\alpha=1)\), complete trust is placed in the database rules, with flat gradients in the classical regions ensuring that rules remain immutable regardless of the truth of classical input facts. This may be of interest to regulated environments, where the preservation of rules are demanded. At its minimum \((\alpha>n/(n+1))\), the designer trusts the facts over the rules. This may apply for machine generated rules that are noisy at best and incorrect at worst. Here, the LNN will maximise the gradients and quickly learn to correct rule weights. Reducing alpha further reduces the Fuzzy region and increases the range of classicality, highlighted in blue in figure \ref{fig:regions_alpha}. 

An alpha between the extrema will allow for variable learning, optionally with a global alpha or an alpha per node, where learning as fast as possible further requires node dependent alphas.

A global alpha places constraints on the weights at initialisation, constrained by the number of input arguments, whereas an alpha node requires constraints to be enforced for each weight update. An alpha per node additionally requires a rescaling of alphas between sub-formula since the relative meaning of truth changes between nodes. Further optimisation is also required for updating local alphas, as \eqref{eqn:and-1} may contain an embedded alpha on the left-hand side of the inequality --- i.e. representative of all sub-formulae, which differs from the formula/operator alpha on the right-hand side. To achieve this, an alpha per node potentially requires a two-stage update of weights and alpha's independently of one another while enforcing constraints --- future work.

There is an additional motivation for including an alpha per node that needs to be considered, namely, the training of an LNN with sub-symbolic/real-valued inputs. Under this configuration, a neural network may offer real-valued truth value bounds as facts, with LNN rules typically initialised as \texttt{True}. While the truth of these formulae are ordinarily used to self-supervise the LNN and minimise contradictions, they may also provide classical supervised target bounds to converge towards. With static bounds, weights parameters alone cannot change the truth of facts themselves. It must also be noted that while facts may be directly modifiable in the LNN framework, not all facts should be modifiable and the database designer may choose not to do so --- e.g.~given trusted fact bounds or operating the LNN in a regulated context. A node per alpha can therefore facilitate such a correction, i.e. turning the secondary fuzzy bounds in table \ref{tab:bound-states-secondary} into classical bounds by decreasing \(\alpha\). Additional loss terms may also be required to ensure a balanced alpha per node, i.e.~lowering \(\alpha\) to increase the learning rate and allowing the defuzzification of bounds vs. increasing \(\alpha\) to decrease the relative interpretability of truth between neurons.

\subsection{Activation function independence}
The tailored activation approach is defined using four boundary points. This permits logical operations to be expressed as any linear or non-linear function that intersects these points for upward inference.

\begin{figure}[!htbp]
    \centering
    \begin{subfigure}[t]{.49\textwidth}
        \centering
        \includegraphics[width=\textwidth]{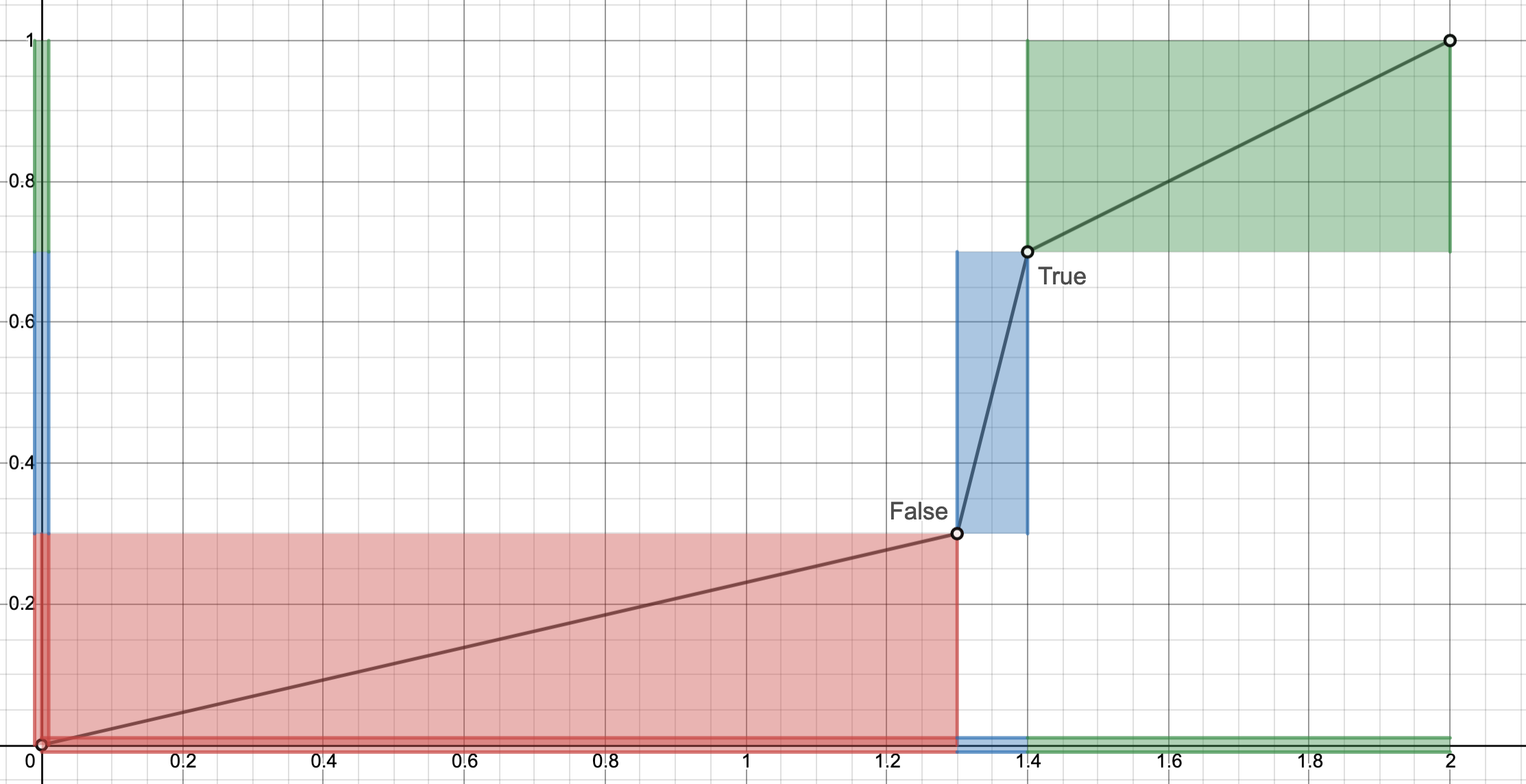}
    \end{subfigure}
    \hfill
    \begin{subfigure}[t]{.49\textwidth}
        \centering
        \includegraphics[width=\textwidth]{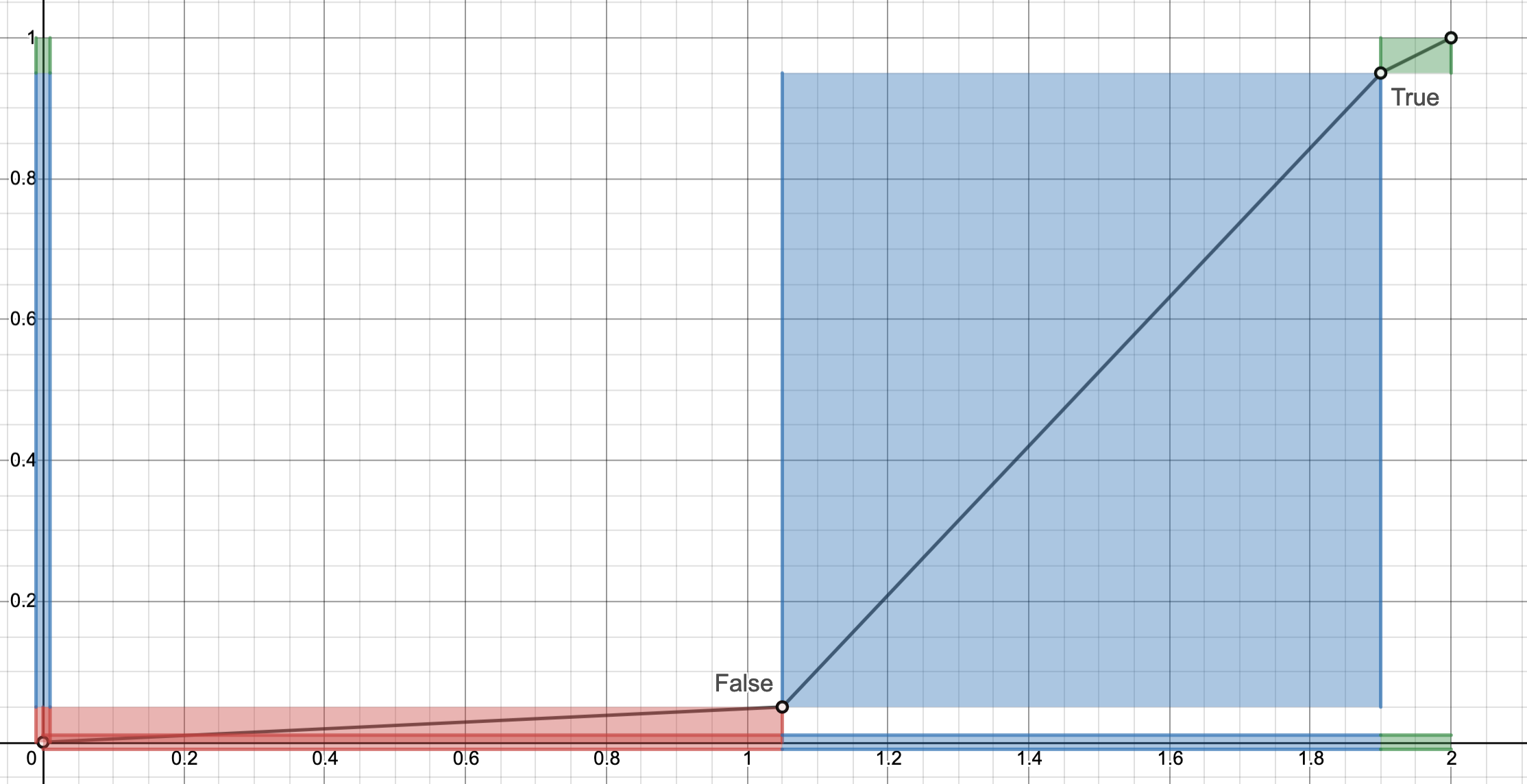}
    \end{subfigure}
    \caption{Piecewise linear activation function; 
    left: alpha = 0.7, right: alpha = 0.95}
    \label{fig:linear_activation}
\end{figure}

The unique piece-wise linear function and functional inverse, is defined as:
\begin{alignat*}{2}
	(x_{min}, y_{min}) &= (0, 0) \\
    (x_F, y_F) &= (\sum_{i \in I} w_i - \alpha w_{max},  1-\alpha) \\
    (x_T, y_T) &= (\sum_{i \in I}  w_i \cdot \alpha,  \alpha) \\
    (x_{max}, y_{max}) &= (\sum_{i \in I} w_i, 1)
\end{alignat*}
\begin{alignat*}{3}
    m_F &= \frac{y_F}{x_F} &&= \frac{1-\alpha}{\sum_{i \in I} w_i - \alpha w_{max}}\\
    m_{Z} &= \frac{y_T - y_F}{x_T - x_F} &&= \frac{2\alpha -1}{\alpha w_{max}-\sum_{i \in I} w_i(1-\alpha)} \\
    m_T &= \frac{y_{max} - y_T}{x_{max} - y_T} &&= \frac{1}{\sum_{i \in I} w_i}
\end{alignat*}
\begin{align}
    f_{\mathbf{w}}(x) &= \left\{\begin{array}{l@{\qquad}l}
        x\cdot m_F & \text{if } x_{min} \leq x\leq x_F,\\
        y_F + (x - x_F) \cdot m_{Z} & \text{if } x_F < x < x_T,\\
        y_T + (x - x_T) \cdot m_T & \text{if } x_T \leq x \leq x_{max}
    \end{array}\right.  \\
    f^{-1}_{\mathbf{w}}(y) &= \left\{\begin{array}{l@{\qquad}l}
        y/m_F & \text{if } y_{min} \leq y\leq y_F,\\
        x_F + (y - y_F) / m_{Z} & \text{if } y_F < y < y_T,\\
        x_T + (y - y_T) / m_T & \text{if } y_T \leq y \leq y_{max}
    \end{array}\right.  %
\end{align}

Figure~\ref{fig:linear_activation} illustrates how a simple linear interpolation between the four boundary points provides interpretability for \(\alpha, \beta, \forall i~w_i\) and a clear definition of classicality in the domain and range.  

Checks for boundary point convergence are also established as parameters fluctuate, ensuring that any division by zero is handled appropriately. This can be especially tricky when working with autograd packages because the existence of such an operation may be allowed --- returning \emph{inf} bounds with \emph{NaN} (not a number) gradients and disconnecting the parameter trace. 

\begin{figure}[!htpb]
    \begin{subfigure}[b]{0.49\textwidth}
	    \centering
	    \includegraphics[width=\textwidth]{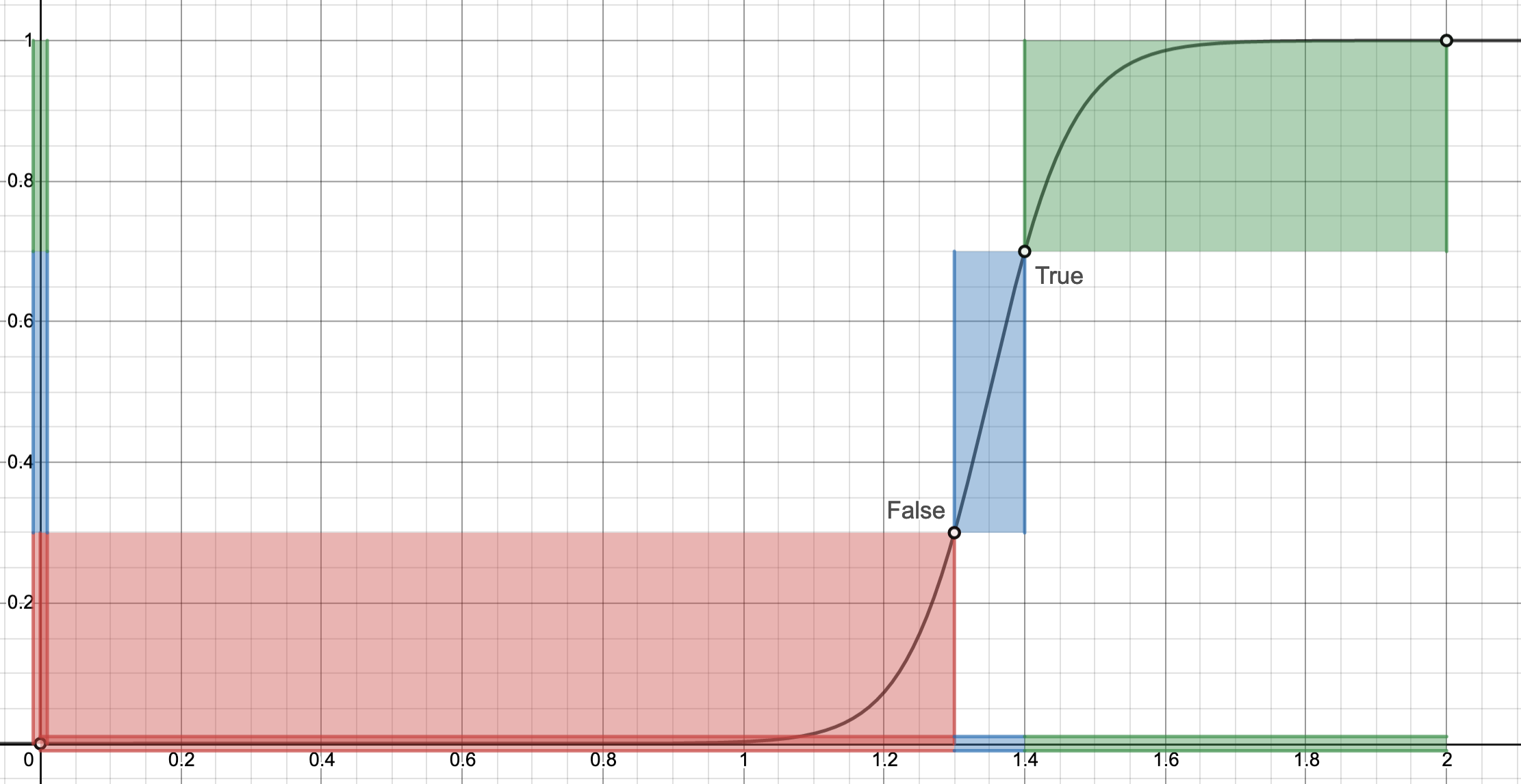}
	\end{subfigure}
	\hfill
	\begin{subfigure}[b]{0.49\textwidth}
	    \centering
		\includegraphics[width=\textwidth]{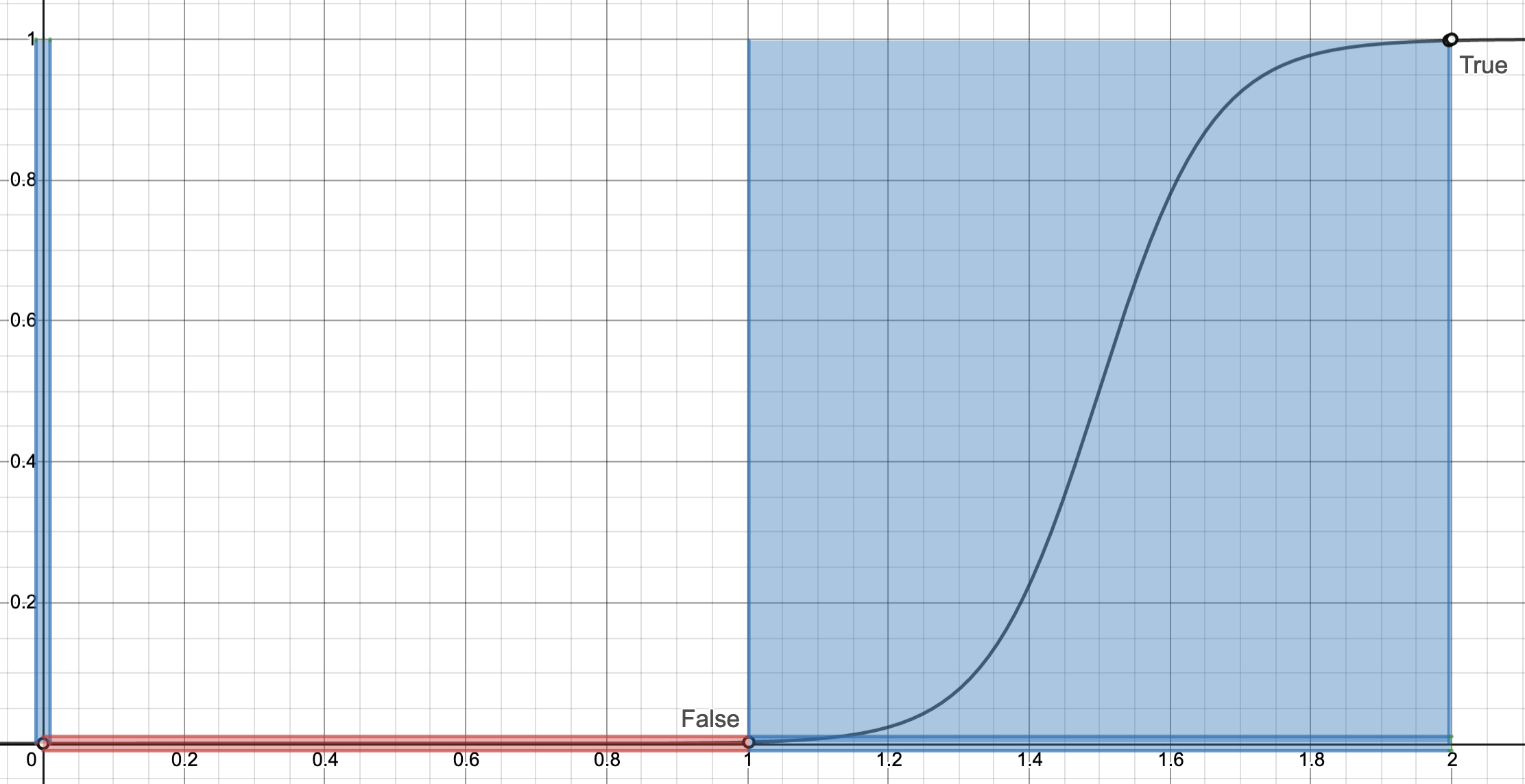}
	\end{subfigure}
    \caption{Nonlinear activation function; 
    left: alpha = 0.7, right: maximum alpha}
    \label{fig:nonlinear_activation}
\end{figure}

Figure \ref{fig:nonlinear_activation} demonstrates a tailored logistic activation function, a scaled and shifted sigmoid. This is one possible non-linear function that satisfies the \texttt{True} and \texttt{False} boundary requirements with smooth gradients everywhere. The function is defined as:

\begin{align*}
    f_{\mathbf{w}}(x) &= \frac{1}{1+e^{(-Ax+B)}} &
    A &= \frac{2\ln\left((1-\alpha)/\alpha\right)}{x_{F}-x_{T}} &
    B &= \ln\left(\frac{\alpha}{1-\alpha}\right)+Ax_F
    \label{logistic:definition}
\end{align*}

While the asymptotic nature of this activation function may prevent its logical operation from reaching maximally \texttt{True} or \texttt{False} points --- \(f_{\mathbf{w}}(x)= \{0,1\}\) --- or even allow it due to floating point round-off, the defined regions of classicality still permit the tailored-LNN to operate as usual. It must be noted that this formulation does not satisfy the boundary points under the extrema, $\alpha=1$. However, no learning would have taken place under this scenario and smooth gradient are therefore not required when operating classically. 

By scaling, shifting \emph{and rotating} the logistic function, we can conceivably rectify the intersection between all boundary points for a given alpha. While the special condition of requiring smoothness without learning, ~$\alpha=1$, may be possible by constructing a piecewise-sigmoid, with flat gradients in the classical ranges.

The incorporation of independent linear and non-linear activation functions therefore demonstrates the robustness of the tailored approach and also permits smooth-optimisation techniques to update LNN parameters --- future work.

\subsection{Proof that zero weighting leads to identity}

Intuitively, as a neuron's input weight begins to drop, so too should its influence on the activation function output. An all-but-one zero-weighted neuron should therefore retrieve only the value of the remaining input, represented by an identity activation function.

\begin{figure}[!htpb]
    \centering
	\includegraphics[width=0.4\textwidth]{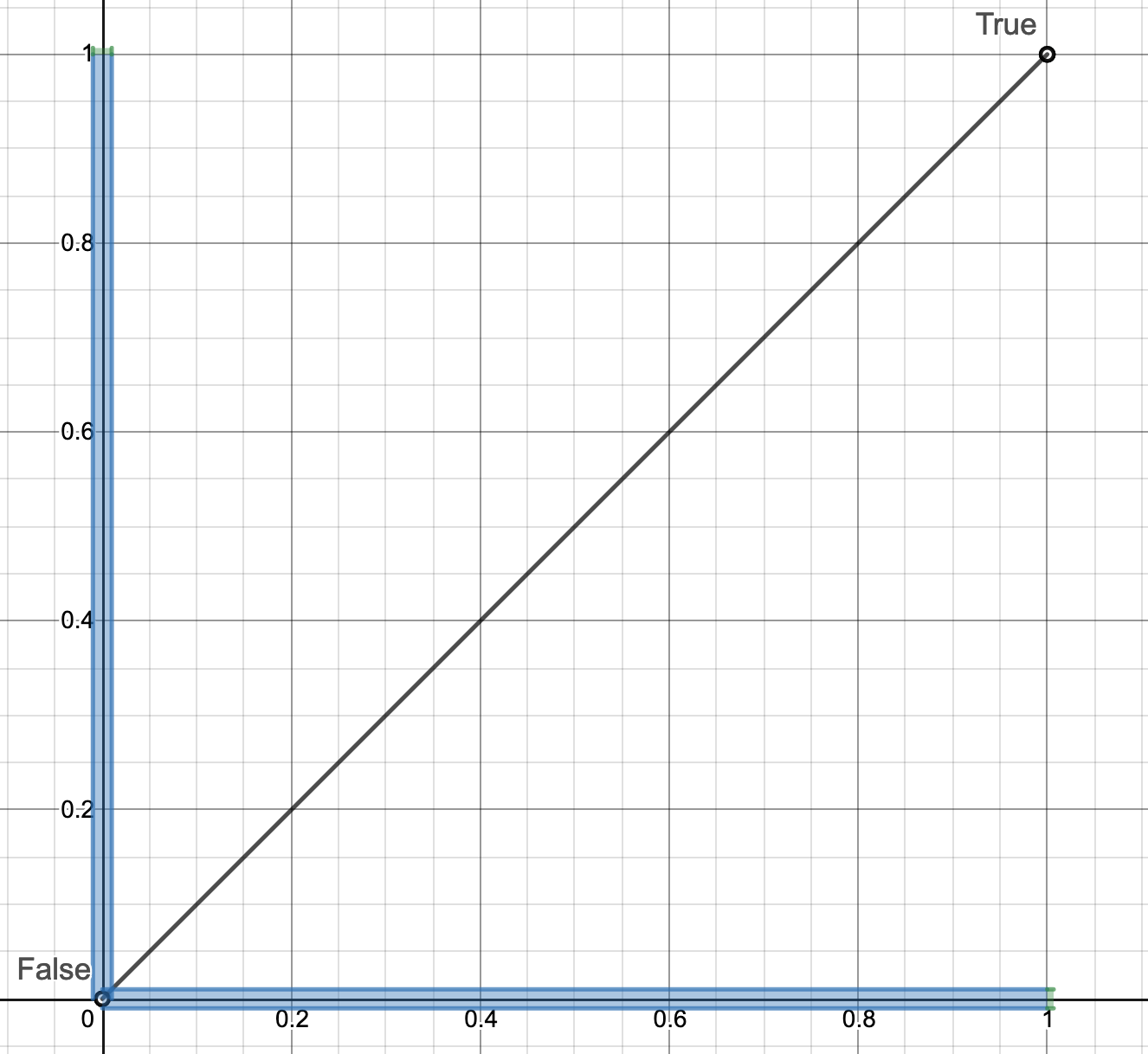}
    \caption{Activation function for a binary conjunction \((\alpha=1)\), with one weight = 0}
    \label{fig:identity}
\end{figure}

Here we provide proof from first principles that the design choice of $x_F$ as a function of the maximum weight, $w_{max}$, is well founded.

Without loss of generality, assume $w_A=0$, we now show that \(w_A A~\&~w_B B=B\) and further visualize it in Figure~\ref{fig:identity}:

\begin{alignat*}{4}
x_F &= \sum_{i \in I} w_i - \alpha w_{max} &&= w_B(1-\alpha) & \hspace{2cm} y_F &= 1-\alpha \\
x_T &= \sum_{i \in I} w_i \cdot \alpha &&= w_B\alpha & \hspace{2cm} y_T &= \alpha \\
x_{max} &= \sum_{i \in I} w_i &&= w_B & \hspace{2cm} y_{max} &= 1
\end{alignat*}
\begin{alignat*}{2}
    m_F &= \frac{y_F}{x_F} &&= \frac{1}{w_B}\\
    m_{Z} &= \frac{y_T - y_F}{x_T - x_F} &&= \frac{1}{w_B}\\
    m_T &= \frac{y_{max} - y_T}{x_{max} - y_T} &&=\frac{1}{w_B}
\end{alignat*}

This then leads to:

\begin{align}
    f_{\mathbf{w}}(w_A A + w_B B) &= \left\{\begin{array}{l@{\qquad}l}
        B & \text{if } x_{min} \leq x\leq x_F,\\
        B & \text{if } x_F < x < x_T,\\
        B & \text{if } x_T \leq x \leq x_{max}
    \end{array}\right.
\end{align}

Since the \texttt{False} boundary point is now defined using \(w_{max}\), identity will hold for an $n$-ary neuron where any number of inputs, but one, are zero-weighted. 

Whereas an all-zero weighted input will decrease the identity until it returns an impulse function, at $x=0$. This results in a \texttt{True} conjunction, \texttt{False} disjunction and a \texttt{False} implication (representative of an `if True then False' statement), making the tailored activation function consistent with its weighted logic counterparts in the upward direction. We note that without directly modifying neuron bounds, the tailored function is incapable of removing contradictions that arise in a dead-state (i.e. all-zero weighted inputs). In this state, it is possible to defer to a learnable $\beta$ and thereby remove inherent contradictions that may arise without the propagation of bounds --- e.g.~root disjunction and implication formulae initialised with all rules being \texttt{True} are self-contradictory due to downward inference from a universal quantifier.

\subsection{Downward inference}

Downward inference may be defined in two ways by weighted fuzzy logic, namely, a functional inverse (Section~\ref{sec:inference}) or logical upward inference (equations~\eqref{eq:lukasiewicz_and_forward}, \eqref{eq:lukasiewicz_or_forward}, \eqref{eq:lukasiewicz_implies_forward}). When given the bounds of the output of a neuron as well as the bounds of all the inputs except one, the functional inverse may proclaim the correct bound values associated with the excluded input in question, however, these values may be classically incorrect. Several of these scenarios exist, where the semantic mapping of real-valued bounds to classical states are erroneous. We highlight a few below.

For example, consider a binary conjunction with \texttt{Unknown} $[0,1]$ and \texttt{True} $[1,1]$ arguments, where the conjunction is itself \texttt{False} $[0,0]$. Classicality demands that downward inference under this configuration should result in the \texttt{Unknown} input being set to \texttt{False} --- since only a false and true input results in a false conjunction in the upward direction. The functional inverse, however, demands from the activation function configuration that a conjunction being exactly zero can only arise if all inputs were also exactly zero. All upper bound inputs are therefore pulled down to zero, irrespective of their current state. This incorrectly sets the \texttt{True} argument into a \texttt{Contradictory} state. A False Implication and True disjunction neuron behave similarly. 

Also consider a \texttt{False} $[0,1-\alpha]$ binary conjunction with one pre-existing \texttt{False} $[0,1-\alpha]$ argument, $A$. Bi-direction inference demands that $B$ remains the same by offering strictly \texttt{Unknown} $[0,1]$ bounds --- it should not be able to say anything about the source of falsity given an existing false input. The functional inverse may, however, aggressively tighten bounds beyond the base \texttt{Unknown} $[0,1]$ state and further yet out of the classically \texttt{Unknown} $[1-\alpha,\alpha]$ endpoint as alpha decreases.

To combat these concerns, downward inference can be seen as nothing but upward inference of a certain tautology. Since the tailored activation approach is classically correct, this automatically does the correct calculation for classical values. Beyond classical values, this can still be used but with a sanity check, namely the functional inverse.

To achieve this, the following tautology $B \rightarrow (A\rightarrow (A\& B))$ is conditionally applied, when trying to retrieve the value of $B$ (the truth of $A$ and $A\&B$ are assumed known and $B$ takes the role of `excluded' since we want to infer its value). The method we use `swaps' out the implication and replaces it with equality, whose truth is equivalent to the truth of the converse $(A\rightarrow (A\& B))\rightarrow B$:

\begin{table}[!htbp]
\caption{downward inference tautology}
\centering
\begin{tabular}{c | c | c | c | c | c} 
 \toprule 
 A & B & A\&B & \(A\rightarrow (A\&B)\) & \(B\rightarrow A\rightarrow (A\&B)\) & \(B=A\rightarrow (A\&B)\)\\
 \midrule
 0 & 0 & 0 & 1 & 1 & 0\\ 
 0 & 1 & 0 & 1 & 1 & 1\\ 
 1 & 0 & 0 & 0 & 1 & 1\\ 
 1 & 1 & 1 & 1 & 1 & 1\\ 
 \bottomrule
\end{tabular}
\label{table:tailored-downward-inference}
\end{table}

If we have a look at the truth table \ref{table:tailored-downward-inference}, we note that equality does not hold in row 1. Despite the tautology being true for row 2, it is indistinguishable from row 1 and cannot be enforced ($A$ and $A\&B$ are given, $B$ is being calculated hence is treated as not given).

\subsubsection{Learning from nothing}

The tautology in table \ref{table:tailored-downward-inference} has undesired non-classical behavior that is only brought to light when enforcing the non-classical characteristic of not `learning from nothing'. This test verifies that a Cartesian product of neuron inputs are incapable of retrieving tighter bounds from downward inference, given unknown neuron output bounds. It must be noted that a neuron may operate classically despite non-classical inputs, which further contributes to discrepancies.

By way of example, consider a classically \texttt{False} $[1-\alpha,1-\alpha]$ conjunction with $(\alpha=0.8)$ and one \texttt{True} $([\alpha,\alpha])$ argument, $A$. The functional inverse, not having any notion of classicality, offers \texttt{$\sim$False} bounds $B = [0.4, 0.4]$, while the logical inference, correctly enforces classicality by offering False bounds $[0.2, 0.2]$. Notice that only a maximally True argument, $[1,1]$, offers a False, $[0.2, 0.2]$, output for the functional inverse. Whereas the logical inference offers a $[0.1, 0.1]$ False bound. As the truth of the conjunction increases away from \texttt{False} $[1-\alpha,1-\alpha]$ to \texttt{$\sim$Unknown} $[0.5,0.5]$, the functional inverse is correctly non-classical, however, the logical inference is undesirably not so --- since the tautology first has to lift out of the classically-false range $[0.1,0.2]$ before offering a non-classical value to $B$.

This behavior therefore motivates for a logical downward inference that is conditioned on the functional inverse outside of the classical regions. For the above scenario, the functional inverse updates too aggressively for both bounds in the classical regions. While in the non-classical region, the lower bound is too aggressive but the upper bound is correct. Comparatively, the logical inference learns too much for the upper bound but just enough for the lower bound of the non-classical region. A combination of both approaches provide for more consistent bound behavior.

We therefore condition the logical inference tautology on the functional inverse by selecting the looser of the two bounds, while keeping the weight assignment consistent with \eqref{eq:or_downward_0} and \eqref{eq:or_downward_1}:
\begin{align}
    U_{x_j} & = \left\{\begin{array}{l@{\qquad}l}
        U_{x_j} & \text{if } U_{x_j} \geq \alpha\ or\ U_{x_j} \geq f^{-1}_{\mathbf{w}}(\bigotimes_{i \neq j} L_{x_i}, U_{\bigotimes_i x_i}),\\
            f^{-1}_{\mathbf{w}}(\bigotimes_{i \neq j} L_{x_i}, U_{\bigotimes_i x_i}) & \mathrm{otherwise}
    \end{array}\right.  \label{eq:conditioned_log_inv_U}\\
    L_{x_j} & = \left\{\begin{array}{l@{\qquad}l}
        L_{x_j} & \text{if } L_{x_j} \leq 1-\alpha\ or\  L_{x_j} \leq f^{-1}_{\mathbf{w}}(\bigotimes_{i \neq j} U_{x_i}, L_{\bigotimes_i x_i}),\\
        f^{-1}_{\mathbf{w}}(\bigotimes_{i \neq j} U_{x_i}, L_{\bigotimes_i x_i}) & \mathrm{otherwise}
    \end{array}\right.  \label{eq:conditioned_log_inv_L}
\end{align}

To demonstrate this behavior of not `learning from nothing', we test a binary conjunction $(A\&B)$ with predefined input bounds to an \texttt{Unknown} neuron. The computation follows both a upward and a downward pass, offering new bounds to the neuron and thereafter offering argument updates. 

The ground truths in Figure \ref{fig:learn_from_nothing} (on the left) is characterised by an identity function for output argument $B$. The output $B=A\rightarrow (A\&B)$ (on the right) illustrates the effect of logical downward inference conditioned on equations \eqref{eq:conditioned_log_inv_U} and \eqref{eq:conditioned_log_inv_L}, as $\alpha$ is varied. Ideally, a neuron should not `learn from nothing' across all real-valued bounds, however learning inside the classical regions doesn't impact the interpretation of the neuron and can be therefore be ignored --- grayed out regions.

\begin{figure}[tbp]
  	\begin{subfigure}{0.5\textwidth}
		\centering
		\includegraphics[width=\textwidth]{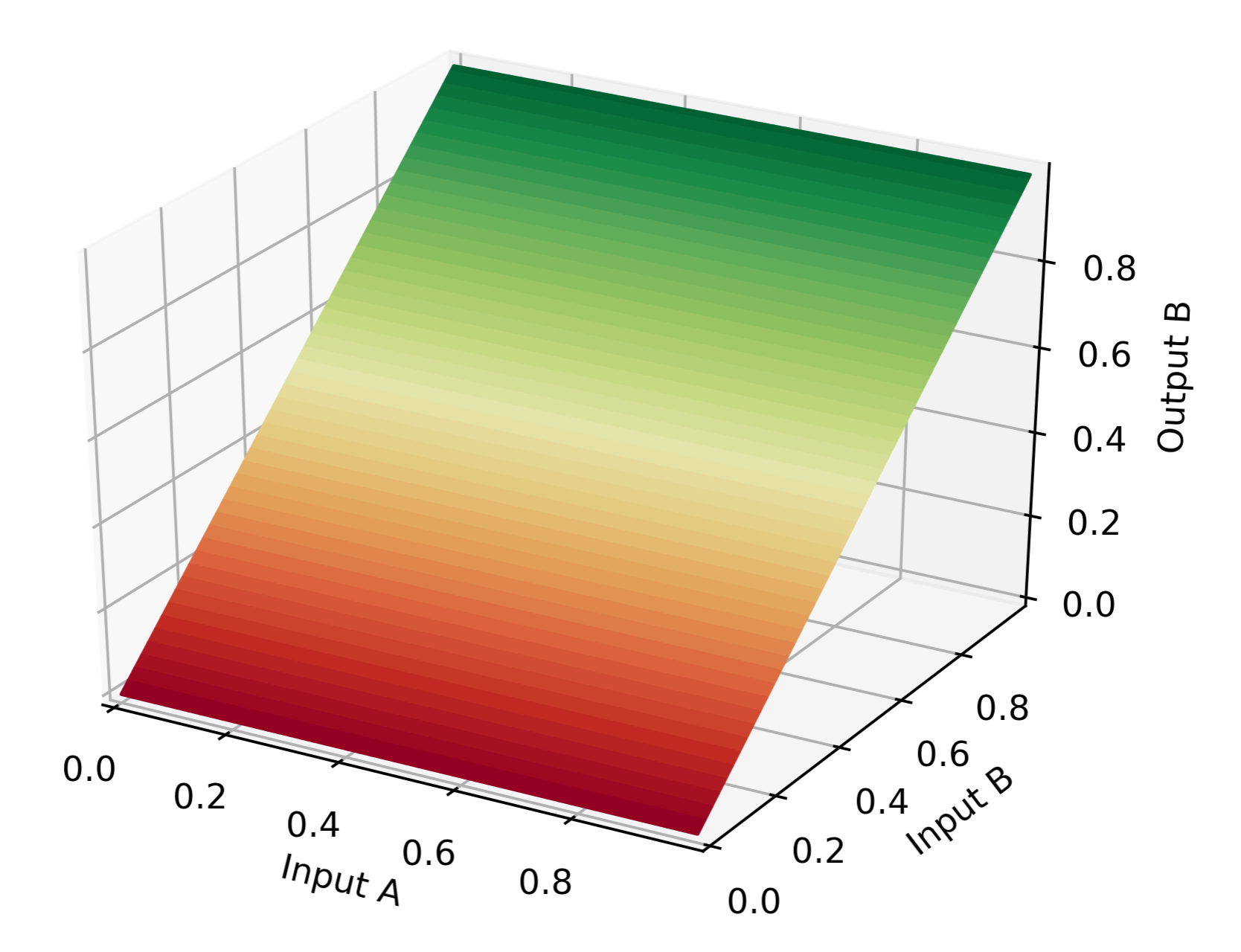}
	\end{subfigure}
	\hfill
	\begin{subfigure}{0.5\textwidth}
		\centering
		\includegraphics[width=\textwidth]{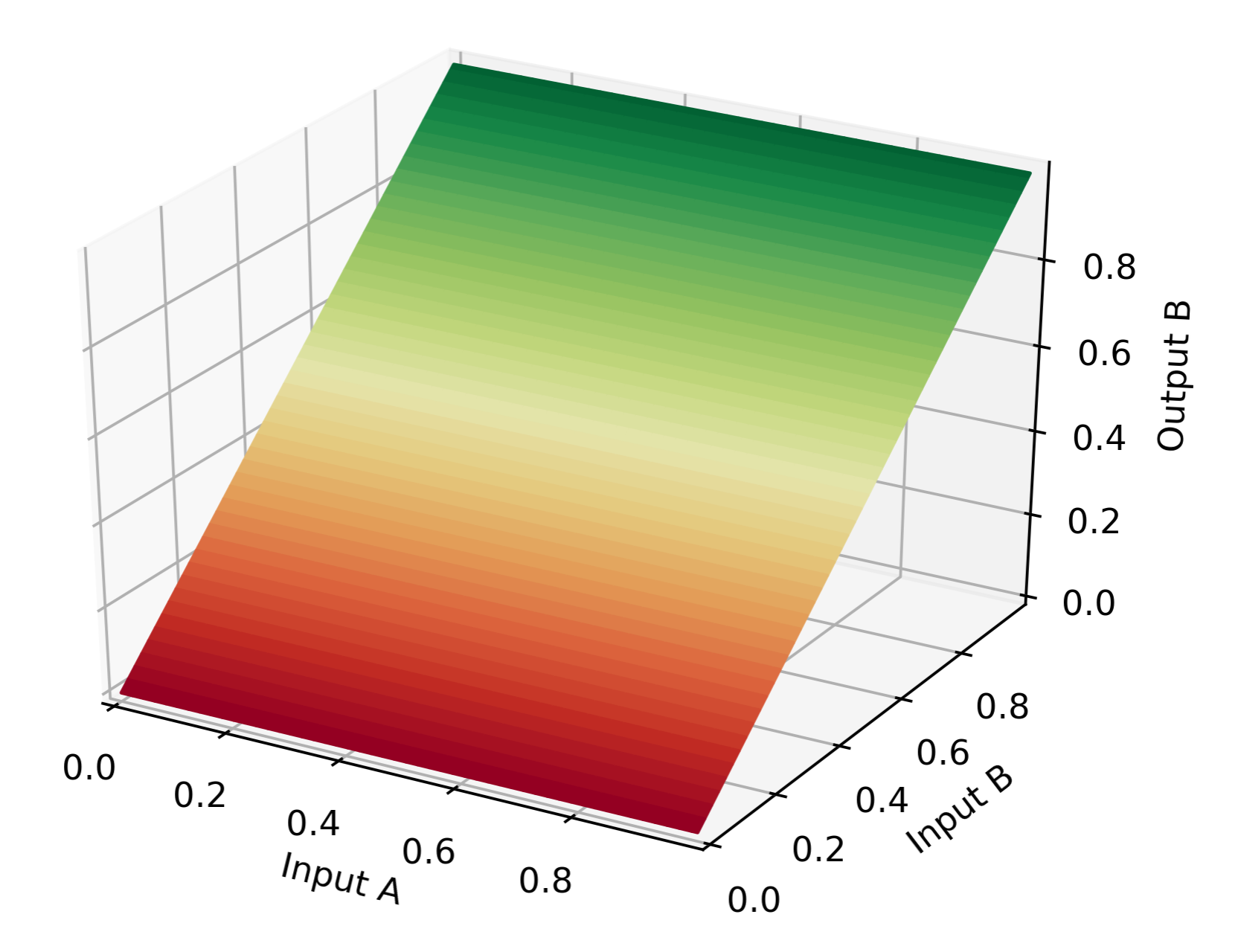}
	\end{subfigure}
	\\
	\begin{subfigure}{0.5\textwidth}
		\centering
		\includegraphics[width=\textwidth]{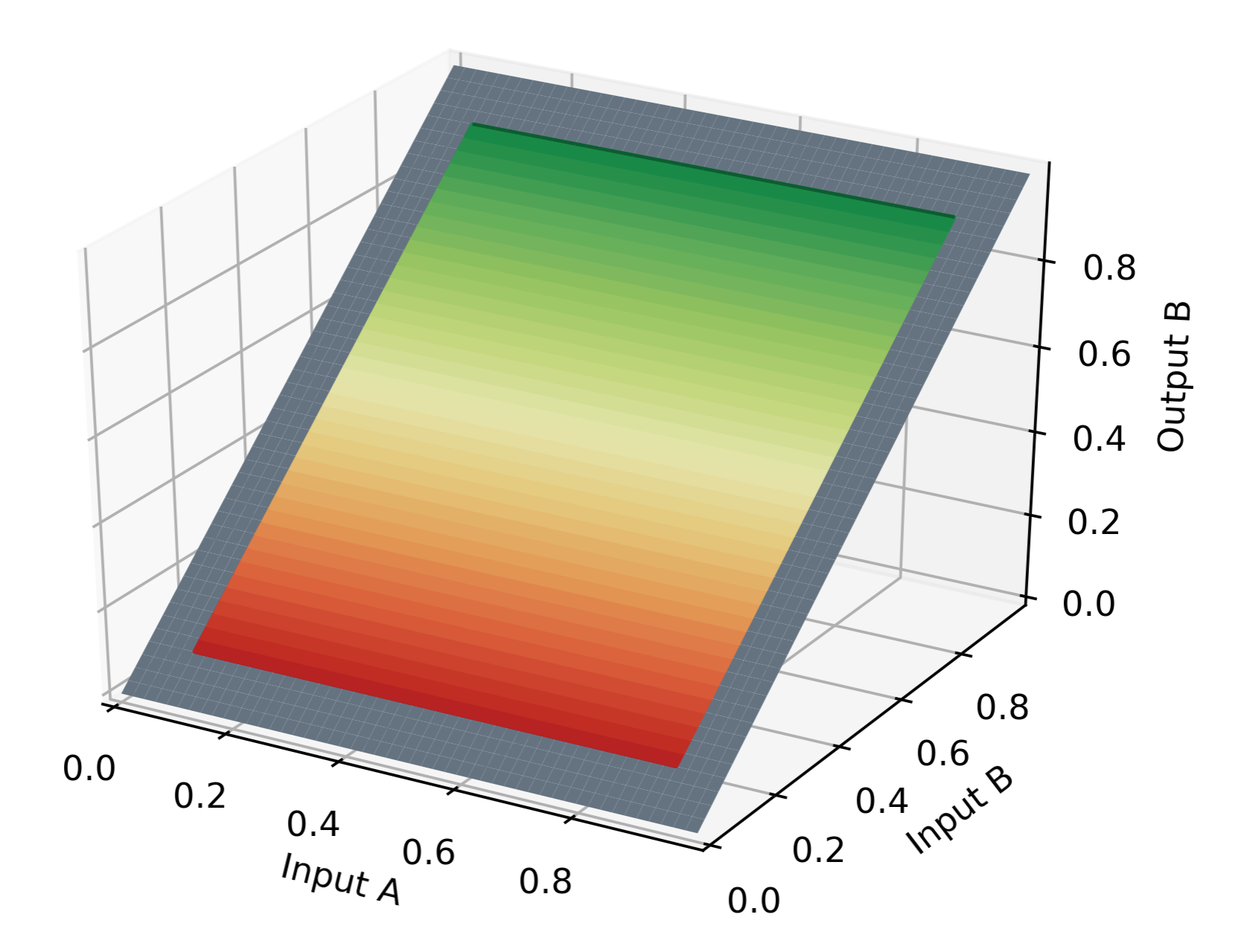}
	\end{subfigure}
	\hfill
	\begin{subfigure}{0.5\textwidth}
		\centering
		\includegraphics[width=\textwidth]{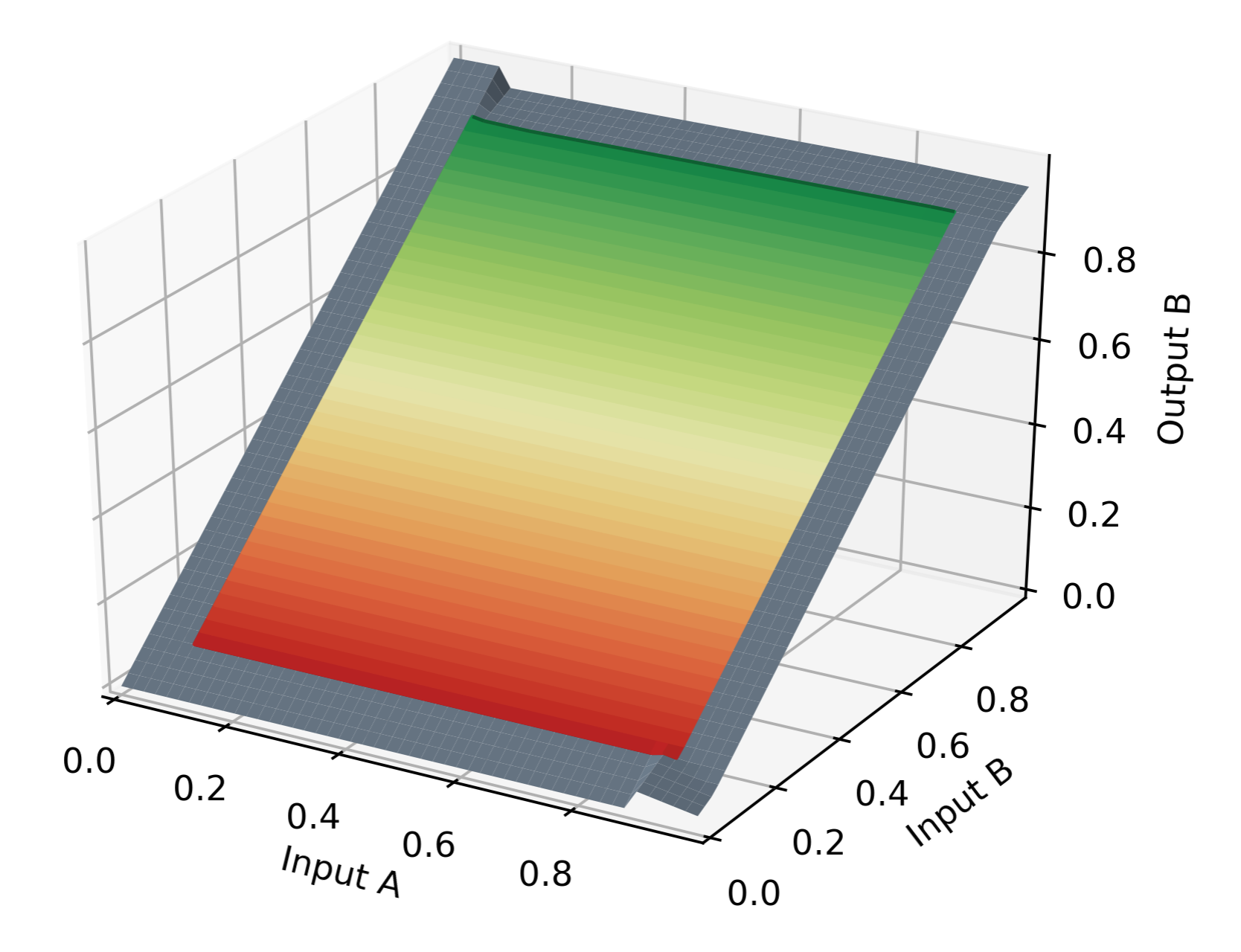}
	\end{subfigure}
	\\
	\begin{subfigure}{0.5\textwidth}
		\centering
		\includegraphics[width=\textwidth]{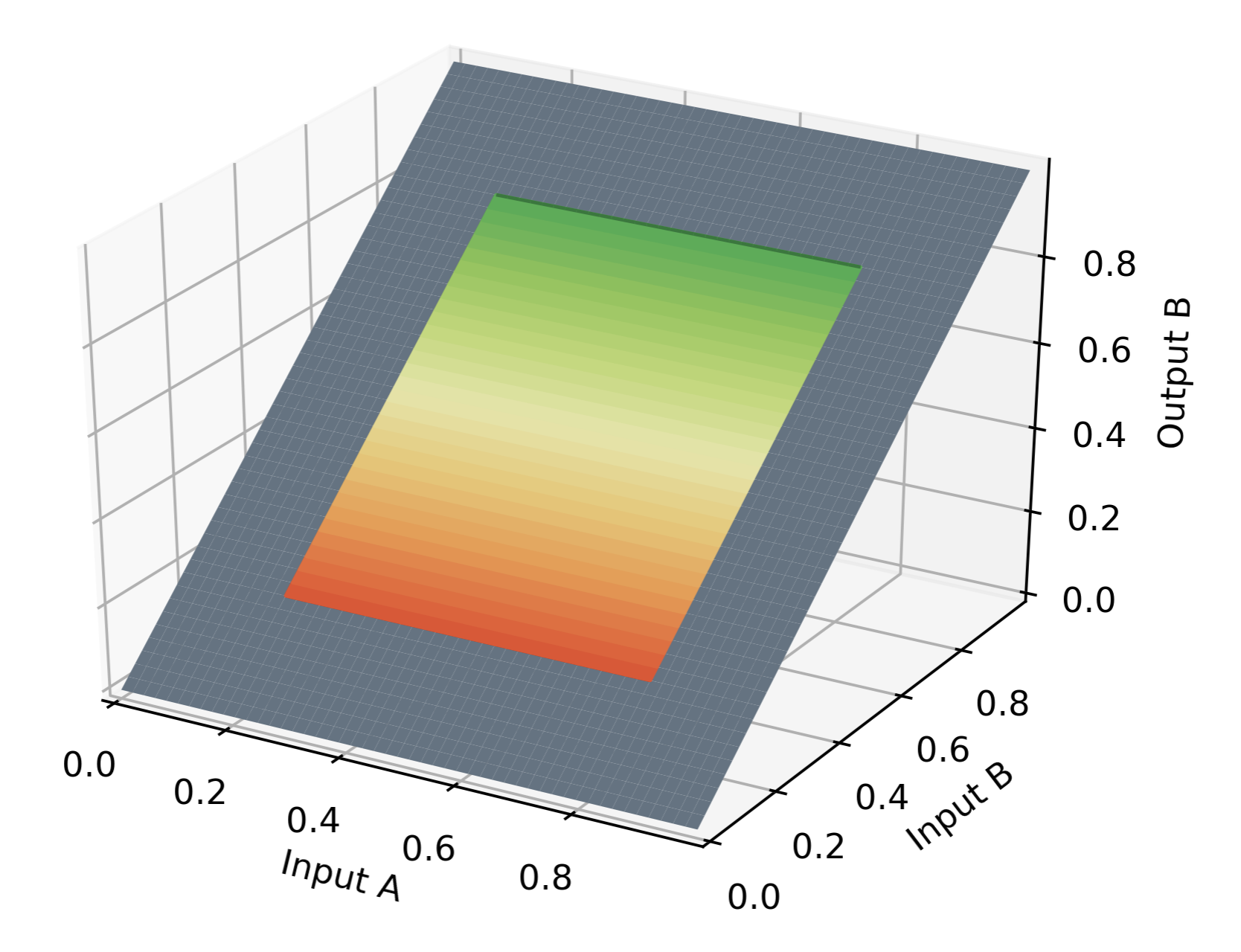}
	\end{subfigure}
	\hfill
	\begin{subfigure}{0.5\textwidth}
		\centering
		\includegraphics[width=\textwidth]{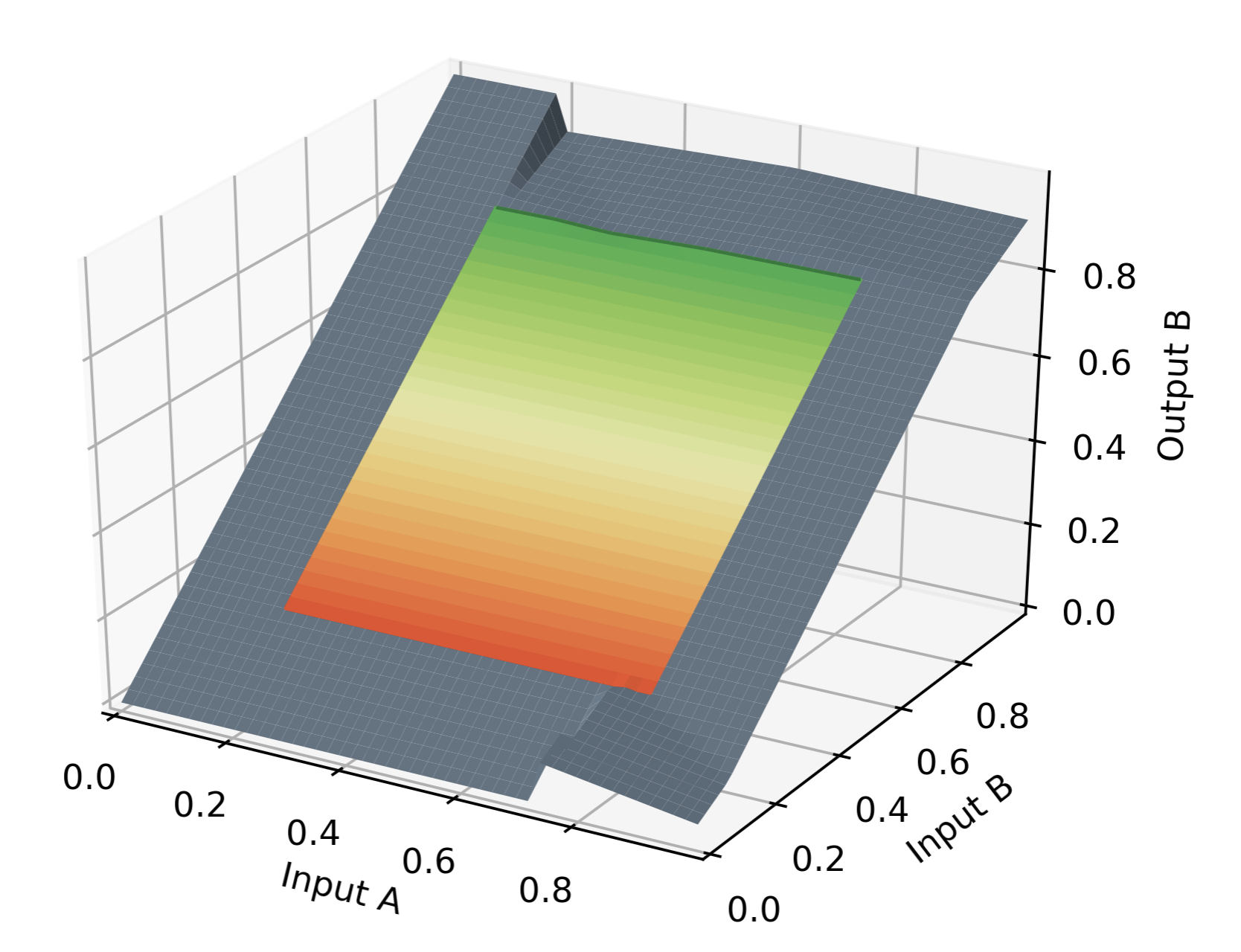}
	\end{subfigure}
	\\
	\begin{subfigure}{0.5\textwidth}
		\centering
		\includegraphics[width=\textwidth]{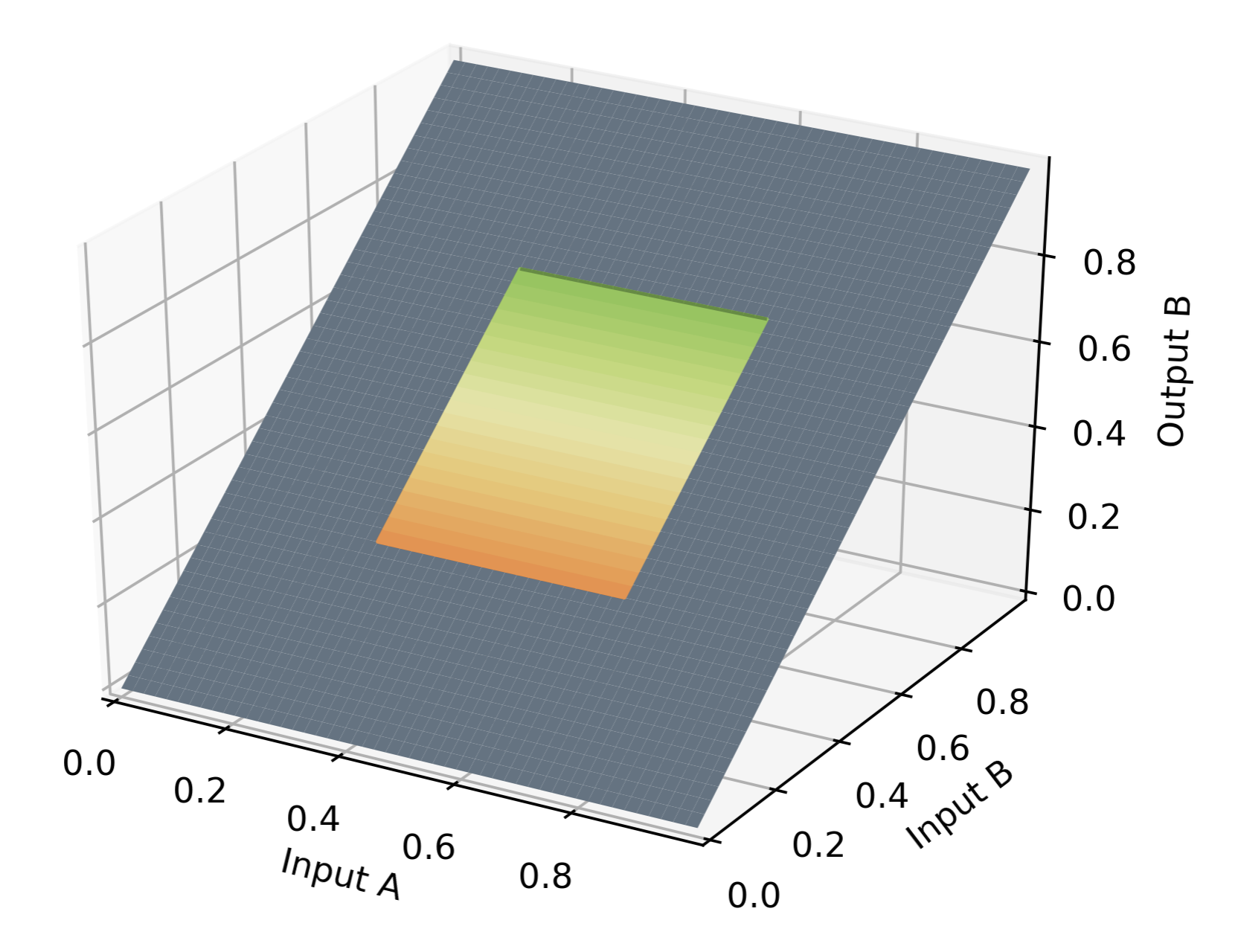}
	\end{subfigure}
	\hfill
	\begin{subfigure}{0.5\textwidth}
		\centering
		\includegraphics[width=\textwidth]{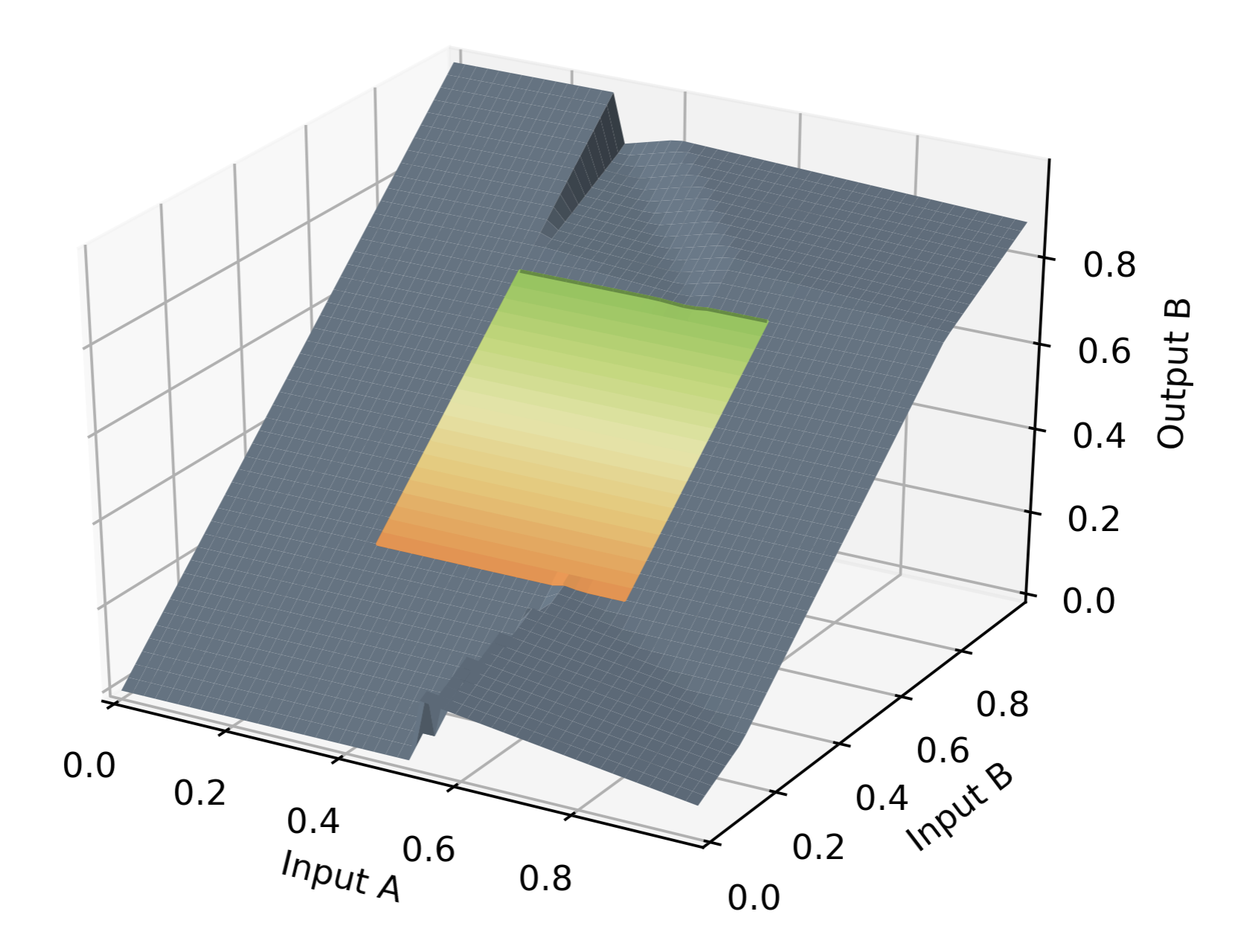}
	\end{subfigure}
  \caption{Testing that a neuron cannot learn from nothing; 
  left: ground truth, right: output from binary conjunction argument, top-to-bottom: $\alpha$ = [1, 0.9, 0.8, 0.7]}
  \label{fig:learn_from_nothing}
\end{figure}

An additional caveat of the conditioned logical inference tautology is that contradictions also require conditioning. While the tautology may aggressively tighten bounds within a classical region, they still ensure a classical operation between regions. An intra-classical contradiction (e.g. where the lower bound is greater than the upper bound, but both bounds are greater than $\alpha$ or less than $1-\alpha$) can therefore no longer be considered as such. Note that this does not break the interpretability of classical regions --- while the bounds may disagree on the relative truth value, both agree on the same classical state. Lower and upper bounds within the ranges of \([0\rightarrow1-\alpha, 0\rightarrow \text{lower bound}]\) or \([\text{upper bound}\rightarrow1, \alpha\rightarrow1]\) are therefore no longer viewed as contradictions and the relevant loss terms are updated accordingly.

\subsubsection{Modus ponens}
It is worth noting that downward-inference is a generalization of classical modus ponens (or modus tollens). What is interesting with downward inference as implemented above is that downward inference is nothing but upward inference. This is surprising from a classical logic theory point of view, since modus ponens is a meta-rule, where as above, downward inference is recast as upward inference. It seems to indicate that the other machinery that we have included, namely, bounds, upward inference, aggregation, tailored-activations and so on, contain sufficient `meta' structure to contain modus ponens naturally. This is a significant aesthetic vote-of-confidence for LNNs.

\subsection{Learning}

Weights are interpretably constrained by the tailored activation function and regularisation is no longer necessary. Losses from the predictive accuracy (supervised learning) and logical consistency checks (self-supervised learning) may be individually or simultaneously calculated. 

It may often be the case that only a subset of root formulae or subformulae labels are available to learn from, if at all. Theorem proving is one such example, where grounded facts and formulae are assumed to be \texttt{True}. Self-supervised learning can then be used to propagate bounds from facts to formulae and decrease logical inconsistencies by propagating bounds downwards. These inference steps may be computed until bounds converge. As bounds are updated, automatic differentiation packages may be used to back-propagate gradients and update weights to correct logical inconsistencies. Weight updates may be determined by, but are not limited to, standard gradient descent methods. 

After computing an inference step for a single neuron or across the entire network, parameters may be updated. The standard approach is to compute upward and downward inference until convergence, to accumulate gradients and thereafter do network-wide updates --- where partial derivatives identify the source of losses. An alternative is to compute inference in both direction until convergence while simultaneously updating weights, i.e. upon detection of a \texttt{Contradiction}, inference is arrested and parameters are updated from the contradiction source to the leaves. This prevents or constrains the propagation of contradictions further into the network.

In contrast to normal neural network training where parameter updates are network-wide, the LNN allows independent formulae, disconnected subgraphs or partitioned subgraphs to learn parameter updates independently, i.e. subgraph inference and parameter convergence may prevent select nodes from future updates and thereby reduce computational idling and improve the convergence rate.

Given classical inputs, the expected classical outputs for a self-supervised LNN is equivalent to that of a supervised LNN, albeit with a possibly longer training time. In stark contrast to ordinary machine learning approaches, the LNN and it's dynamic activation function direct the network neurons towards logical self-consistency, without requiring many labeled data points to do so.

\section{Gradient-transparent clamping}
\label{apndx:Gradient-transparent clamping}
This section describes gradient-transparent clamping used in the Smokers and Friends and LUBM learning and noise experiments. Real-valued logic typically bounds truth values to a defined range, which necessitates the use of clamping operations on the results. Automatic differentiation provides clamps that can perform the value clamping, but it also nullifies the associated gradient which can disable learning of involved parameters. Gradient-transparent clamping addresses this by fully recovering gradients outside of the allowable region by utilizing computations detached from the gradient-tracking computational graph to assert the clamping while keeping the original gradient information.

\subsection{Clamping in neural networks}
\subsubsection{Smooth bounded activations}
Neuron activation functions based on the logistic function, including sigmoid $(1+e^{-x})^{-1}$ and tanh $(e^x-e^{-x})/(e^x+e^{-x})$ ensure outputs are in a bounded range, typically [-1, 1] or [0, 1], while ensuring differentiability across the entire domain. Bounded neuron activations ensure that values in the neural network don't grow too large and that there is a degree of interpretability at the output of neurons such as for binary classifiers or real-valued logic.

Smooth bounded activation functions $f()$ like sigmoid and tanh have two-sided saturation where gradients $\lim\limits_{x \to \infty}\partial f(x)/\partial x=0$ tend to zero in the extremes. The vanishing gradient problem affects learning with these bounded activation functions where neurons are saturated or where the chain rule in deep networks produce a negligible product of small gradients. The negative effect of vanishing gradients is that it can significantly attenuate signals that gradient descent depends upon for learning, effectively shutting down learning at some neurons. Residual connection is a solution in deep neural networks that skips over a stride of layers to shorten the number of interacting gradients and reduce the vanishing gradient problem. Another approach is to choose an activation function that does not have small gradients, like the ReLU.

\subsubsection{ReLU}
A rectified linear unit (ReLU) $\max(0,x)$ has a one-sided saturation of zero where $x<0$ with an associated gradient of 0, and a linearity with gradient 1 otherwise. ReLU can address the vanishing and exploding gradient problems since its gradients can only be 0 or 1. However, this activation is affected by the ``dying ReLU'' problem where, if a preactivation distribution is limited to the negative domain, then the input gradients never propagate through the ReLU such that it cannot learn. Leaky ReLU and various forms of parameterized ReLU have been proposed to address the absence of negative domain gradients.

The lower bound on the ReLU output can be shown to be useful when neurons are viewed as concept detectors that give a degree of presence of a feature through the linearity, but only a uniform statement about the absence of a feature through the rectification.
The ReLU lower bound can also possibly prevent correlated neurons by not propagating negative values corresponding to degrees of absence of a learnt feature. The ability to stop signal propagation could also reduce noise and extraneous information replication in the network.

The computational benefits of ReLU during learning is attributed to the gradient sparsity introduced by its zero gradients, which means that gradient calculation only needs to operate on a subset of the neurons. In addition, the two possible gradients of ReLU are constant scalars and no involved computation is required to determine its gradients.

\subsubsection{Information and gradient dropout}
Neural networks are universal function approximators that can embody functions of any complexity if provided with adequate capacity and a large enough set of neurons and interconnections between the neurons. Determination of the exact capacity required for arbitrary functions could be an intractable problem, so normally neural networks are given excess capacity or grow larger over a series of hyperparameter optimizations.

The problem of overfitting is where a trained neural network can not reach its objective performance on previously unseen data, so it can fail to generalize. Redundancy or excess capacity in neural networks allow for overfitting where the input training data could be memorized to optimize the stated objective. Various regularization methods like data augmentation and dropout have been proposed to prevent overfitting, which introduce random variations in the input data in the case of data augmentation or randomly remove neurons during training epochs.

Bounded activation functions that can saturate to an output and gradient of 0, such as sigmoid and ReLU could be considered to be performing a combination of information and gradient dropout. This can act as a regularization method that effectively removes capacity or disconnects certain neurons under the conditions of 0 saturation. The empirical success of dropout and ReLU has been attributed to the regularizing effects of the sparsity these methods introduce.

\subsection{Gradient-transparent clamping}
\subsubsection{Clamping in automatic differentiation}
The lower bound on ReLU $\max(0,x)$ performed through the clamp, min or max functions of automatic differentiation systems typically disconnects $x$ from this node in the computational graph and replaces it with a new node valued 0 with no prior computational history attached. This means that the loss gradient could be $\partial f(x)/\partial x=0$ even though $x$ has been involved in the determination of the clamping outcome. This gradient statement says that any change to $x$ will leave the output unchanged, which is no longer the case if an update to $x$ is sufficiently large to bring it back to the allowable domain.

Severing the argument when clamping so that it is no longer represented in the computation could introduce inaccuracies in the interpretation of calculated gradients. If the subdifferential for 0 is set to the right-hand discontinuity, so that 0 gives the 1 gradient of the linearity, then it can state that the ReLU will decrease to a negative value for gradient descent at a linear rate. So if $x=0$ the projected result of gradient descent with a step of $s=-0.1$ could be $x'=x+s=-0.1$ given a linear gradient, although this update could not materialize since it is already clamped at its starting value. So non-zero gradients provided do not give guidance on when the clamp could be encountered, which shows that even the normal application of clamping could result in incorrect gradient interpretations especially if function linearity is assumed.

Gradients provided by automatic differentiation both inside and outside of the clamped region are then subject to interpretation, and there is reliance on the learning update mechanisms to manage step sizes and momentum to navigate the pathological loss landscape introduced by clamping. Otherwise clamping can inform learning with gradients that indicate the loss can be changed by updating a parameter, even though clamping is then turned on at a fraction of the parameter step size. Clamping could also indicate that no learning is possible when the output is currently clamped, even though clamping turns off for a small parameter step in the right direction such that it does learn.

\subsubsection{Information dropout and gradient transparency}
Bounded activation functions that saturate at 0 can be said to perform regularization under certain preactivation conditions by stopping the propagation of dynamic information and gradients. Approaches to information and gradient dropout are especially relevant to neural networks with redundant capacity as a means of introducing sparsity to preventing overfitting and to address the vanishing gradient problem.

In the case of LNN as a one-to-one mapping of the syntax tree of a logical program, there is a significant difference to normal neural networks where hidden neurons do not necessarily relate to external concepts and where hidden layer widths can be arbitrarily sized. The problems of overfitting and dealing with redundancy are thus not of primary concern with LNN, although logical neurons and subnetworks could be copied to expand parameter capacity. The motivation for introducing gradient sparsity for normal neural networks through ReLU does not readily apply to LNN.

The vanishing gradient problem can be altogether avoided in LNN for certain loss terms that can be arrested neuron-wise so that a shorter gradient explanation can be calculated. However, it can be more complex for combined losses with supervised task-specific objectives calculated simultaneously over different batch dimensions representing alternate universes of truth value assignments. Smooth activations should then still be avoided even for LNN to address the vanishing gradient problem.

Information dropout through clamping in real-valued logic serves to keep truth values within the interpretable permissible range of the logic. \L ukasiewicz logic applies clamping aggressively to ensure logical soundness, so the logic operation output values need to be bounded. However, the associated gradient outside of the clamped region need not necessarily be bounded as well. In fact, the severing of the gradient outside the clamp prevents learning for a significant part of the operating range of the logical neuron activation.

Gradient-transparent clamping provides information dropout or value clamping while also leaving gradients unaffected so that any gradient calculation effectively ignores clamping. The benefit is that the primary purpose of clamping is still fulfilled, namely clamping output values, but that the full output domain retains gradients as if no clamp was applied. The ``dying ReLU'' problem can also be solved with gradient-transparent clamping, since learning can receive gradients across the full operating range such that gradient descent always has the optimization direction available for a parameter.

The perceived errors this would introduce would be of exactly the same nature as those made by also clamping the output value. In particular, a gradient inside the bounds could indicate that the output will change even when the clamp then immediately applies into the update step, yet this was not reflected in the gradient. Similarly, a gradient outside the bounds given by gradient-transparent clamping could indicate an output change, but then the clamp still applies after the update step. The learning optimizer should manage the step sizes and momentum to ensure robust learning under these scenarios, both with and without gradient-transparency.

\subsubsection{Gradient-transparent clamping}
Automatic differentiation normally severs a computed node when clamping the output and substitutes it with a new node representing the applied bound. Removing the computed node also removes trace of its involvement in the current output, which is undesirable as it was indeed part of the calculation. The idea is then to retain the input to the clamping operation as part of the resultant node in the computation graph. This necessitates the definition of a new clamping function that performs value clamping while keeping gradients as if no clamp was applied.

For a lower bound clamp at $x_{\min}$ the value clamping is performed when $x<x_{\min}$ to give an output of $x_{\min}$. This can be calculated as $x - \min(0,x^*-x_{\min})$ where $x^*$ denotes a copy of $x$ that has been detached from the computational graph and thus carries no history. By obtaining the disconnected value $x^*$ the value clamp can still be applied without the destructive interference otherwise caused if a tracked copy of the same node was used. Automatic differentiation tools normally do allow for detached value replication, and there are various capabilities that allow for value copy and reuse. The superscript $(1)$ denotes that gradients in the clamped regions are unscaled, as opposed to the linear scaling performed by gradient supplantation that will be explained later.
\begin{align}
    \min(x_{\max}, x)^{(1)}&=x - \max(0,x^*-x_{\max}),\\
    \max(x_{\min}, x)^{(1)}&=x - \min(0,x^*-x_{\min}).
\end{align}

Clamping can then be replaced with gradient-transparent clamping to recover gradients across the entire output range while still providing the intended value clamping. The provision here is that the parameter update strategy should consider the possibility of no output change for an insufficient step size, despite the gradient in the value-clamped region stating that an output change was expected. Again, this downside is akin to the gradient interpretation difficulties faced when clamping values and being near the boundary on the inside of the allowable region and stepping outside, in which case an output change was also expected but did not realize when the clamping immediately applies.

\subsection{Clamping in fuzzy logic}
\subsubsection{Real-unit interval}
The continuous t-norms of fuzzy logic perform binary operations on the real-unit interval [0, 1] that represent infinite-valued logic where truth values can be interpreted as ambiguous mixtures between \texttt{True} and \texttt{False}. Clamping the result of fuzzy logic operations like those of \L ukasiewicz logic is necessary to ensure truth values remain in the interpretable range of [0, 1]. The min() and max() functions clamp the computed values in all \L ukasiewicz logic operators, otherwise the result can be outside of the permissible truth value range of [0, 1]. For the \L ukasiewicz conjunction upward pass the clamping then bounds the output and provides the associated clamped gradients as follows.
\begin{align}
    \textstyle \bigotimes^{\beta}_{i \in I} x_i^{\otimes w_i} & \textstyle = \max(0, \min(1, \beta - \sum_{i \in I} w_i (1 - x_i))),\\
    \frac{\partial \left(\bigotimes^{\beta}_{i \in I} x_i^{\otimes w_i}\right)}{\partial \beta} & = \left\{\begin{array}{l@{\qquad}l}
  1 & \mathrm{if}~0\le\bigotimes^{\beta}_{i \in I} x_i^{\otimes w_i}\le 1,\\
  0 & \mathrm{otherwise},
    \end{array}\right.\\
    \frac{\partial \left(\bigotimes^{\beta}_{i \in I} x_i^{\otimes w_i}\right)}{\partial w_i} & = \left\{\begin{array}{l@{\qquad}l}
  (x_i - 1) & \mathrm{if}~0\le\bigotimes^{\beta}_{i \in I} x_i^{\otimes w_i}\le 1,\\
  0 & \mathrm{otherwise}.
    \end{array}\right.
\end{align}
Note that for a \texttt{True} input $x_i=1$ the corresponding gradient for the output in terms of $w_i$ is $(x_i-1)=0$, which means that it provides no gradients for updating $w_i$ even when no upper bound clamping is applied. The bias $\beta$ will have to be adjusted instead to obtain the desired output, such as a \texttt{False} output when all inputs are \texttt{True}.

\subsubsection{Gradient-transparent clamping}
Note that the gradients are non-zero only inside the real-unit interval, so any clamping that is applied normally nullifies the gradient information so that learning receives no useful gradient information. This is problematic given that a major part of the operating range of the logical conjunction undergoes clamping, depending on the parameters. The solution to this problem is to only perform value clamping but leave the gradients untouched, through gradient-transparent clamping. In this case the output gradient for $\beta$ is always 1 and for $w_i$ it is always $(x_i-1)$, both inside and outside the clamped region.

The parameterization of the proposed logical neurons is especially well-behaved and bounded, with $0\le w_i\le1$ and $0\le\beta\le 1+\sum_i w_i$ normally applying without loss of functionality. This is because any parameter configuration from these constraints can allow the useful functional range of bounded ReLU to be accessed. Consequently, the learning optimizer can set robust update step sizes and limit parameter updates, which means that the presence of transparent gradients from parameters in clamped situations requiring larger update steps can be handled feasibly.

The contradiction loss $\sum_j \max(0,L_j(\mathbf{\beta},W) - U_j(\mathbf{\beta},W))$ also involves clamping with the intent that a loss term should be activated only when the lower bound $L_j$ is higher than upper bound $U_j$. If there are no contradictions, we would not want to make any parameter updates based on this loss. So the clamping can be performed such that the gradients are also clamped and a zero gradient set when there are no contradictions. This means normal clamping can be performed, especially for loss terms in cases where no learning should be performed.

\subsubsection{Downward pass}
Downward pass at a weighted \L ukasiewicz conjunction can be determined by firstly unclamping its output, which involves changing a clamped lower bound at $L_{\otimes}=0$ to the minimum unclamped value $L_{\otimes}=\beta - \sum_{i \in I} w_i$ and similarly changing a clamped upper bound at $U_{\otimes}=1$ to the maximum unclamped value $U_{\otimes}=\beta$. Unclamping and extending the bounds where necessary ensures that all downward pass explanations can be generated.
\begin{align}
    \textstyle \bigotimes^{\beta'}_{i \in I} x_i^{\otimes w_i} & \textstyle = \max(0, \min(1, \beta - \sum_{i \in I} w_i (1 - x_i)))\\
     & \textstyle = \max\left(0, \min\left(1, \bigotimes^{\beta}_{i \in I} x_i^{\otimes w_i}\right)\right).
\end{align}
Unclamping $x'=\min(x_{\max}, x)$ is simply $\max(x', x)$ and similarly for $x'=\max(x_{\min}, x)$ we have $\min(x', x)$ as unclamping operation. Here the $x$ values are recalculated with gradient-tracking, although during clamping they are detached in the subtraction. Note that disabling gradient-tracking is not required for unclamping as it was for clamping, since there is no subtraction or its destructive interference that can nullify gradients. The unclamping of the bounded output to its extremes can then continue as
\begin{align}
    \textstyle L_{\otimes} & \textstyle = \min\left(\bigotimes^{\beta'}_{i \in I} x_i^{\otimes w_i},  \beta - \sum_{i \in I} w_i\right),\\
    \textstyle U_{\otimes} & \textstyle = \max\left(\bigotimes^{\beta'}_{i \in I} x_i^{\otimes w_i}, \beta\right).
\end{align}

The functional inverse can then be determined with known values for the conjunction output and all but one of the inputs $i$, where backward inference can then calculate a value for the remaining subject input $j$. The calculated proof for the target input is then also clamped to the real-unit interval before proof aggregation is performed.
\begin{align}
    \textstyle \bigotimes^{\beta}_{i \in I} x_i^{\otimes w_i} & \textstyle = \beta - \sum_{i \in I} w_i (1 - x_i)\\
    \textstyle w_j (1 - x_j) & \textstyle = \beta - \sum_{i \in I\setminus j} w_i (1 - x_i) - \bigotimes^{\beta}_{i \in I} x_i^{\otimes w_i}\\
    \textstyle (1 - x_j) & \textstyle = \beta / w_j - \sum_{i \in I\setminus j} w_i (1 - x_i) / w_j - \left.\left( \bigotimes^{\beta}_{i \in I} x_i^{\otimes w_i}\right ) \right/ w_j\\
    \textstyle x_j & \textstyle = \max\left(0, \min\left(1, 1 - \beta / w_j + \sum_{i \in I\setminus j} w_i (1 - x_i) / w_j + \left.\left( \bigotimes^{\beta}_{i \in I} x_i^{\otimes w_i}\right ) \right/ w_j\right)\right).\label{eqn:b-back1}
\end{align}
The conjunctive syllogism $(p\otimes \neg(p\otimes q)) \rightarrow \neg q$ provides logical inference to determine a backward inference result equivalent to the functional inverse since
\begin{align}
    \textstyle x_j & \textstyle = 1 - \beta / w_j + \left(1 - \left(1-\sum_{i \in I\setminus j} w_i (1 - x_i)\right)\right) / w_j + \left.\left( \bigotimes^{\beta}_{i \in I} x_i^{\otimes w_i}\right ) \right/ w_j\\
     & \textstyle = 1 - \beta / w_j + \left.\left(1 - \bigotimes^{1}_{i \in I\setminus j} x_i^{\otimes w_i}\right) \right/ w_j + \left.\left( \bigotimes^{\beta}_{i \in I} x_i^{\otimes w_i}\right ) \right/ w_j\\
     & \textstyle = {^{\beta/w_j}}\left( \left(\bigotimes^{1}_{i \in I\setminus j} x_i^{\otimes w_i}\right)^{\otimes 1/w_j} \rightarrow \left(\bigotimes^{\beta}_{i \in I} x_i^{\otimes w_i}\right)^{\oplus 1/w_j} \right).
\end{align}
where the weighted \L ukasiewicz implication is generally defined as
\begin{align}
    {}^{\beta}(x^{\otimes w_x} \rightarrow y^{\oplus w_y}) & = \max(0, \min(1, 1 - \beta + w_x (1 - x) + w_y y)) \\
    & = {}^{\beta}((1 - x)^{\oplus w_x} \oplus y^{\oplus w_y}). \notag
\end{align}
Note that there is a negation of the partial conjunction which also involves a swapping of its lower and upper bounds in the backward inference calculation.
The unclamped gradients obtained over the entire operating range with gradient-transparent clamping calculates as follows from \ref{eqn:b-back1}
\begin{align}
    \frac{\partial x_j}{\partial \beta} & = -1/w_j, \\
    \frac{\partial x_j}{\partial w_i} & = \frac{1-x_i}{w_j}, \\
    \frac{\partial x_j}{\partial w_j} & = -\frac{\beta}{w_j^2} - \sum_{i \in I\setminus j} \frac{w_i (1 - x_i)}{w_j^2} - \frac{\textstyle \bigotimes^{\beta}_{i \in I} x_i^{\otimes w_i}}{w_j^2} \frac{\partial \left( \bigotimes^{\beta}_{i \in I} x_i^{\otimes w_i}\right)}{\partial w_j}.
\end{align}
For weights $0\le w_j \le 1$ smaller than 1 all backward inference gradients $\lim\limits_{w_j \to 0}\partial x_j/\partial \beta=\infty$, $\lim\limits_{w_j \to 0}\partial x_j/\partial w_i=\infty$, and $\lim\limits_{w_j \to 0}\partial x_j/\partial w_j=\infty$ tend to become large as the weights become smaller. Gradient clipping can  deal with these large gradients when performing learning updates, or reverting to clamping gradients as per usual can also be considered.

\subsubsection{Gradient supplantation}\label{apndx:gradsup}
The introduction of a threshold-of-truth $\alpha$ allows for non-zero gradients in the regions where clamping will normally be performed. A tailored piecewise linear activation function provides a range of $[0,1-\alpha]$ to express the clamped \texttt{False} domain of a logical conjunction, thus it has a positive non-zero gradient associated with the value-region of classical \texttt{False}. The gradient states that a positive change in the function input will result in a positive increase in the output truth value, even though an insufficient step-size could still result in a classical \texttt{False} output.
Adjusting the threshold-of-truth $2/3\le \alpha\le 1$ can change the gradient magnitude in the classical regions, so that a more conservative gradient can be obtained for these regions that could make learning updates more accurate. 

In contrast, gradient-transparent clamping utilizes the existing gradients of the same magnitude as the unclamped region, so it offers more aggressive gradients in the classical regions compared to piecewise linear activations. An approach of gradient supplantation in gradient-transparent clamping could ensure equivalent gradient magnitudes to piecewise linear activations, but without the need to support the classical symmetry in range $[\alpha, 1]$ associated with \texttt{True}. Output values also do not have to be relaxed by $\alpha$ so that \texttt{False} is still only at an output of 0, but arbitrary gradients can be provided in the clamped regions with gradient-transparent clamping.

Basic gradient supplantation alters gradient-tracked values by scaling their gradients with a provided scalar $a$ under specified conditions. In gradient-transparent clamping the addition of gradient supplantation can scale the effective gradient where value clamping has been applied. Bounded ReLU can then exhibit an arbitrarily scaled gradient in its rectified region to allow for more accurate learning updates, since smaller gradients can be chosen for clamped regions. The scaling-based gradient supplantation uses indicator or boolean condition functions readily available in automatic differentiation libraries, in addition to value detachment $x^*$ from the computational graph for gradient-tracked $x$, and the adapted gradient-transparent functions are thus applied as
\begin{align}
    \min(x_{\max}, x)^{(a)}&=x(x\le x_{\max})+(ax-ax^*+x_{\max})(x_{\max}< x),\\
    \max(x_{\min}, x)^{(a)}&=(ax-ax^*+x_{\min})(x<x_{\min})+x(x_{\min}\le x).
\end{align}
Normal clamping that involves setting gradients in the clamped regions to zero, would then correspond with zero-scaled gradient-transparent clamping $\min(x_{\max}, x)^{(0)}$ and $\max(x_{\min}, x)^{(0)}$ that uses gradient supplantation with a scaling factor of $a=0$.

\section{Empirical evaluation (continued)}
\label{apndx:EmpiricalEvaluation}
\subsection{Smokers and friends}
This section has further details about the \textbf{Smokers and friends} experiment, to augment the overview given in the main paper in {\bf Empirical evaluation} section (Section 7) of the main paper, on page 7.

\textbf{Dataset.}
We use the Smokers and friends dataset from LTN experiment $\mathcal{K}_{\text{exp2}}$ \cite{serafini2016logic} with the small universe of constants a-h. Initial facts on three predicates, namely smokes $S(x)$, cancer $C(x)$, and friends $F(x,y)$ (open-world) are given.

\textbf{Axioms.}
We use the logical formulae, consisting of 5 plausible axioms, of the Smokers and Friends experiment in LTN experiment $\mathcal{K}_{\text{exp2}}$ \cite{serafini2016logic}.
We repeat the experiments including axioms induced by MLN~\cite{richardson2006markov} (total 8 axioms) on this dataset, seeded with the original 5 axioms. In particular we learn MLN formula structure and associated weights on this data, using Alchemy 2~\cite{richardson2007alchemy} with the command \texttt{./learnstruct -i smoking.mln -t smoking-a-h.db -o smoking-a-h-out.mln} (please find the attached code in the supplementary files).
\begin{table}[htbp!]
    \centering
    \begin{small}
    \caption{Learned parameters for each Smokers and friends experiment, to contrast ablation of logical connective weights, and show the effect of a larger simultaneous axiom set on bounds relaxation.}
    \begin{tabular}{ccccc}
    \toprule
        \textbf{Experiment} & \textbf{Axioms} & \textbf{Axiom bounds} & \textbf{Initial fact bounds} & \textbf{Connective weights}\\
    \midrule
        $P_1^5$ & 5 & $\checkmark$ & $\checkmark$ & $\checkmark$\\
        $P_2^5$ & 5 & $\checkmark$ & $\checkmark$ & \\
        $P_1^8$ & 8 & $\checkmark$ & $\checkmark$ & $\checkmark$\\
        $P_2^8$ & 8 & $\checkmark$ & $\checkmark$ & \\
    \bottomrule
    \end{tabular}
    \label{tab:smokers-specs1}
    \end{small}
\end{table}

\textbf{Experiments.}
We conduct four experiments, two on the smaller 5 axioms, and another two on the larger 8 axioms for the same data. The learning parameterization of the experiments in Table~\ref{tab:smokers-specs1} show ablation of the logical connective weights to ascertain its effect. The objective of the experiment is formulated in the loss function as $(1+\mathtt{contradiction})/(1+\mathtt{factalign}+\mathtt{tightbounds})$, where the goal is to reduce contradiction over the entire network as much as possible by balancing relaxation (through learning) of initial facts and all inferences made. This means that initial fact bounds and axiom bounds form learnable parameters and are updated through gradient descent with the loss gradient w.r.t. each parameter. A larger axiom set will place more simultaneous constraints that could produce more inconsistencies.

After each update a new training epoch is initialized with the updated bounds, followed by the Upward-Downward inference algorithm run until no new inferences are made. Backpropagation occurs over all of these inference passes for a given training epoch to provide a gradient trace to the initial computational graph leaves that include the adjusted fact/axiom bounds and logical connective weights. The learning settings in Table~\ref{tab:smokers-specs2} are all the same for the various experiments, and a relatively small number of epochs are required.
\begin{table}[htbp!]
    \centering
    \begin{small}
    \caption{Learning settings for $P_1^5, P_2^5, P_1^8, P_2^8$, with a decreasing learning rate schedule over 100 epochs from 0.1$\rightarrow$0. Gradient clipping for truth bound and logical connective weights ensure granular updates, and a minimum logical connective weight is kept to avoid division-by-zero during inference.}
    \begin{tabular}{ccccc}
    \toprule
        \textbf{Epochs} & \textbf{Learning rate} (start) & \textbf{Learning rate} (end) & \textbf{Weight} (min) & \textbf{Gradient clipping}\\
    \midrule
        100 & 0.1 & 0 & 0.01 & [-0.1, 0.1]\\
    \bottomrule
    \end{tabular}
    \label{tab:smokers-specs2}
    \end{small}
\end{table}

\begin{table}[htbp!]
    \centering
    \caption{Inferred fact bounds ($[L,U]$ as \%) for LTN $\mathcal{K}_{\text{exp2}}$ (cmp. \cite{serafini2016logic} Table 1) with learnt initial facts. A single number indicates matching bounds $L=U$. The purple shaded regions have groundtruth facts provided, which are adjusted through learning, followed by a final inference to convergence to produce the bounds below.}
    \includegraphics[width=0.9\textwidth,trim={1.7cm 1.9cm 16.1cm 1.8cm},clip]{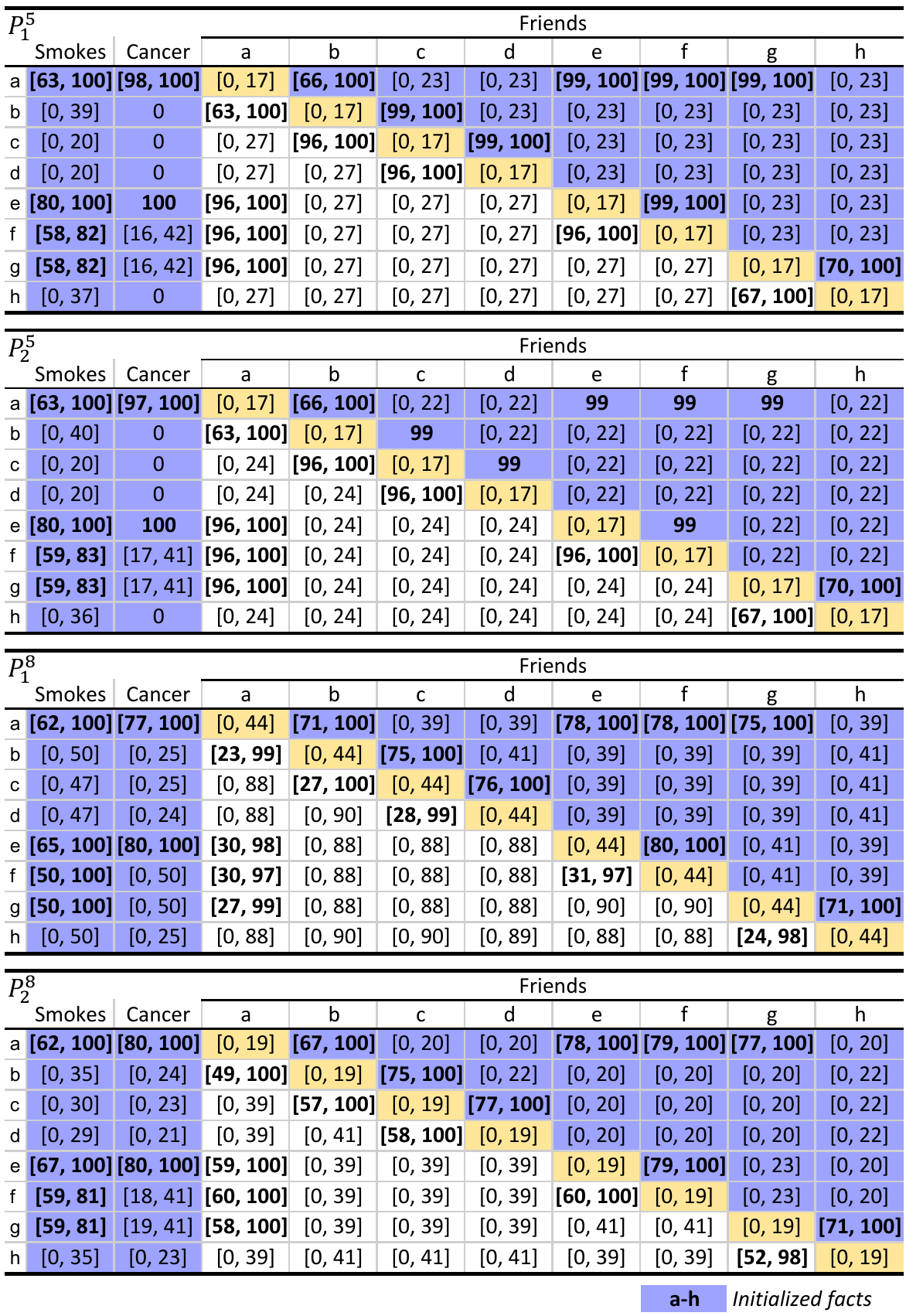}
    \label{tab:smokers-p5}
\end{table}

\textbf{Experimental results.}
Overall results in Tables~\ref{tab:smokers-p5} and~\ref{tab:smokers1*} show learnt initial fact and axiom truth value bounds for the different experiments, also showing additional loss values not given in the main paper.
The LTN degree of satisfiability could possibly be compared to the mean of the LNN truth value bounds, although some comparative deviances are noted especially for $\neg S(x)\kern-0.2em\lor\kern-0.2em\neg F(x,y)\kern-0.2em\lor\kern-0.2em S(y)$ where LTN reports high satisfiability of 0.96 while LNN records bounds of [65, 100]. This is possibly due to LTN not completely inferring friendship symmetry, whereas LNN applies friendship symmetry which is then followed by transitivity of smoking amongst friends that are untrue in some cases.
\begin{table}[htbp!]
    \centering
    \begin{small}
    \caption{Learnt LNN neuron weights (from $P_1^{8}$) and axiom lower bounds (as \%) for LTN experiment $\mathcal{K}_{\text{exp2}}$ (universe: a-h) \cite{serafini2016logic}, compares LTN degree of satisfiability (as \% for 5 axioms) to LNN $P_1^{5}, P_2^{5}$ and repeats in $P_1^{8}, P_2^{8}$ for 8 axioms including 3 induced by MLN~\cite{richardson2006markov}, with corresponding MLN log-probability weights, followed by axiom-wise contradiction counts. Loss function $(1+\mathtt{contradiction})/(1+\mathtt{factalign}+\mathtt{tightbounds})$ (normalized) component values after training show complete removal of contradictions by relaxing facts and inferences.
    Gradient descent ($\alpha=1, \beta=1, w_{\max}=1$ with weight normalization and gradient-transparent clamping) adjusts operand weights ($P_1$), initial axiom/fact bounds ($P_1$, $P_2$). Every training epoch performs inference initialized with updated bounds until convergence after the parameter update.
    }
    \begin{tabular}{@{ }l@{}c@{ }@{ }c@{ }c@{ }@{ }c@{ }c@{ }@{ }c@{ }c@{ }}
        \toprule
        \multicolumn{2}{@{}r@{}}{}&\multicolumn{4}{@{ }c@{ }}{\bf{LNN} $[L, U]$} \textit{100 epochs, lr: 0.1$\rightarrow$0}\\
        \cmidrule{3-4}\cmidrule{5-6}
        \textbf{Smokers and friends} [$\mathcal{K}_{\text{exp2}}$ (a-h)] & \textbf{LTN} & ${P_1^{5}}$ & ${P_2^5}$ & ${P_1^8}$ & ${P_2^8}$ & \textbf{MLN} & $L\kern-0.3em>\kern-0.3emU$\\
        \midrule
        $\exists y F(x,y)$ & 100 & 100 & 100 & 100 & 100 & 6.88 & 0\\
        $\neg F(x,x)$ & 98 & [83, 98] & [83, 98] & [56, 98] & [80, 98] & 0.26 & 0\\
        $\neg F(x,y)^{0.96}\kern-0.2em\lor\kern-0.2em F(y,x)^1$ & 90 & [96, 97] & [97, 100] & [51, 95] & [82, 97] & - & 0\\
        $\neg S(x)^{0.98}\kern-0.2em\lor\kern-0.2em\neg F(x,y)^1\kern-0.2em\lor\kern-0.2em S(y)^{0.97}$ & 96 & [65, 100] & [65, 100] & [65, 100] & [66, 100] & 3.53 & 2\\
        $\neg S(x)^1\kern-0.2em\lor\kern-0.2em C(x)^{0.98}$ & 77 & [57, 100] & [58, 100] & [50, 100] & [60, 100] & -1.35 & 2\\\cmidrule{1-4}
        $\neg F(x,y)^{0.97}\kern-0.2em\lor\kern-0.2em \neg S(y)^1\kern-0.2em\lor\kern-0.2em F(y,x)^{0.96}\kern-0.2em\lor\kern-0.2em S(x)^1$ &  &  &  & 100 & 100 & 6.87 & 0\\
        $\neg F(w,x)^1\kern-0.2em\lor\kern-0.2em\neg F(w,y)^1\kern-0.2em\lor\kern-0.2em \neg F(z,x)^1\kern-0.2em\lor\kern-0.2em C(z)^{0.97}$ &  &  &  & [73, 100] & [70, 100] & 4.33 & 51\\
        $\neg F(w,x)^1\kern-0.2em\lor\kern-0.2em\neg F(y,w)^1\kern-0.2em\lor\kern-0.2em \neg F(z,y)^1\kern-0.2em\lor\kern-0.2em \neg S(y)^{0.99}$ &  &  &  & [80, 100] & [77, 100] & 9.68 & 66\\
        \midrule
        Contradiction (remaining) & & 0 & 0 & 0 & 0 & & \textit{121}\\
        Factual $E_i[|L_i'-L_i|+|U_i'-U_i|]$ (start: 0.64)& & 0.42 & 0.43 & 0.27 & 0.37 & & \\
        Bound tightness $E_i[\exp(L_i-U_i)]$ (start: 1) & & 0.88 & 0.89 & 0.9 & 0.97 & & \\
        \midrule
        \textbf{Loss} (start)& & 1.897 & 1.897 & 46.276 & 46.276 & & \\
        \textbf{Loss} (end)& & 0.435 & 0.432 & 0.461 & 0.429 & & \\
        \bottomrule
    \end{tabular}
    \label{tab:smokers1*}
    \end{small}
\end{table}

The MLN formula weights are not directly comparable to truth values, and requires conversion to probabilities, although this conversion would be dependent on the data specifically. MLN axiom weight 6.88 for $\exists y F(x,y)$ states that a world with $n$ friendless people is $e^{6.88n}$ times less probable than a world where all have friends, other things being equal~\cite{richardson2006markov}, but LNN sets full truth as it does not conflict.
Still, a relative magnitude ranking can be compared, where we see that the last two axioms have relatively high weight or log-probability of being true for the data, in spite of the observation of many contradictions attributed to these axioms. This could possibly due to incorrect closed-world assumptions made about participating predicates.

The primary observation is that LNN can perform gradient descent with the system loss to remove contradiction by relaxing facts and axioms, carefully balancing impact on factual correctness and the usefulness (tightness) of inferences made. If facts deviate too much from groundtruth then factual incorrectness arises, and when an axiom truth is relaxed too much, then less tight and more unknown inferences will be made. Learning of neuron weights can possibly reduce bounds relaxation in $P_1^{8}$ when comparing to $P_2^{8}$ for the last two high-conflict axioms.

\textbf{Clamped gradients.}
We perform an experiment to measure the effect of available gradients across the clamped domain of \L ukasiewicz for the 8 axiom set where we update all parameters, namely $P_1^8$. The results in Figure~\ref{fig:smokers-cgrad} shows graphs for the system loss $(1+\mathtt{contradiction})/(1+\mathtt{factalign}+\mathtt{tightbounds})$ and its denominator components. Gradient supplantation, introduced in Section~\ref{apndx:gradsup}, is used here to introduce arbitrary gradients in the clamped domain. A comparison is made against the baseline of conventional \L ukasiewicz with no gradients in the clamped domain.
\begin{figure}
    \centering
    \begin{small}
    \includegraphics[width=\textwidth,trim={0.3cm 6.0cm 3.4cm 0.3cm},clip]{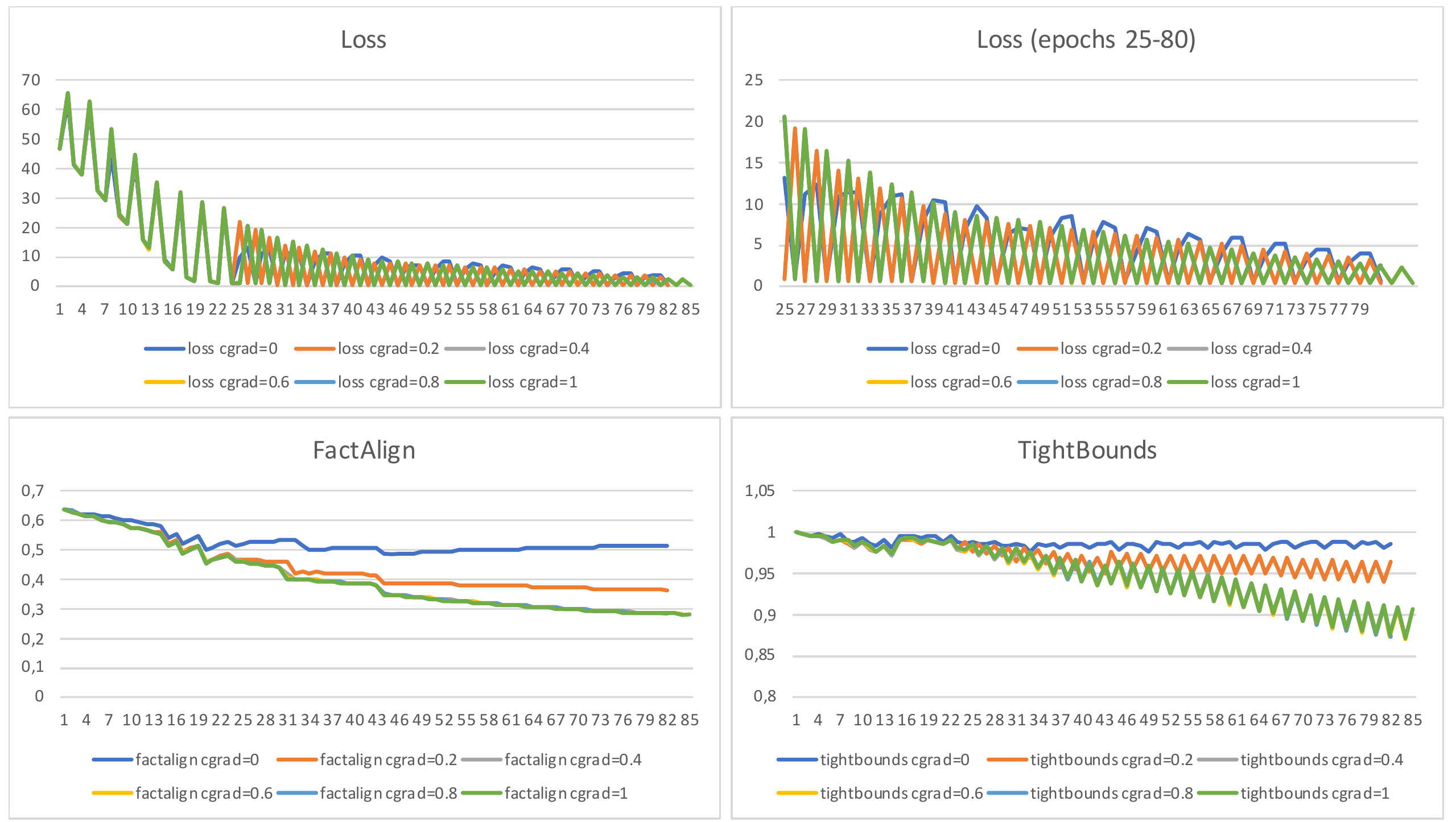}
    \caption{Gradient supplantation $x(0\le x)+c_{\text{grad}}(x-x^*)(x<0)$ (please see Section~\ref{apndx:Gradient-transparent clamping}) with arbitrary \L ukasiewicz conjunction \texttt{False} gradients $c_{\text{grad}}\in \{0, 0.2, 0.4, 0.6, 0.8, 1\}$ are compared in terms of the system loss $(1+\mathtt{contradiction})/(1+\mathtt{factalign}+\mathtt{tightbounds})$ and its denominator components for $P_1^8$. A conventional \L ukasiewicz conjunction \texttt{False} gradient $c_{\text{grad}}=0$ changes fewer parameters, but is still able to remove contradictions for this CNF system although at a slower convergence, yet it exhibits smaller deviance from initial facts and produces more useful inferences. \L ukasiewicz conjunctions with non-zero \texttt{False} gradients appear to contribute to converging slightly faster (at the top of the "flip-flop"), as it possibly has more parameters with gradients available, although the downside is larger relaxation of facts and axioms.}
    \label{fig:smokers-cgrad}
    \end{small}
\end{figure}

We note a "flip-flop" learning convergence pattern for the loss in Figure~\ref{fig:smokers-cgrad}, where slight contradictions register in alternate epochs which are subsequently corrected. This is possibly due to the optimum closely bordering a logically inconsistent state, so as not to relax bounds and axioms too much. Gradients detach for corrected contradictions, which mean that only alternate epochs show non-zero gradients near convergence. \texttt{tightbounds} only measures bounds separation for non-contradictory states, so the "flip-flop" pattern extends to this loss component.  

The pattern is probably just an artifact of the loss function, and the added degrees of freedom for parameter updates gradient-introduction in the clamped \L ukasiewicz domain gives. A standard decreasing learning rate, which also benefits from applied gradient clipping, can ensure some learning convergence at a degree of the real optimum.

A conventional \L ukasiewicz conjunction \texttt{False} gradient $c_{\text{grad}}=0$ converges slower than with non-zero gradients, as there are fewer parameters with gradients. In practice, we implement disjunctions via De Morgan's law in terms of negations and conjunctions. However, for disjunctive clauses and \texttt{True}-polarized axioms, there could still be sufficient gradient-tracked parameters for conventional \L ukasiewicz to change disjunctive clauses from \texttt{False} to \texttt{True} (the clamped region for disjunctions) to remove contradictions in this system of formulae. However, without gradients, one would normally not change a clamped result, e.g. \texttt{True} to \texttt{False} for disjunctions.

In the case of these CNF formulae, gradient introduction across the full operating domain of logical connectives may not be strictly necessary to remove contradictions. However, for more complex systems of formulae and facts a non-zero gradient in the clamped regions of logical connectives may be required to address some contradictions. Sigmoidal and tailored activation functions can also provide full gradient coverage for logical neurons.

\subsection{LUBM benchmark}
\label{sup:lubm_benchmark} 
This section gives additional details on {\bf LUBM benchmark} experiments as reported in {\bf Empirical evaluation} section (Section 7) of the main paper, on pages 7 and 8 (line numbers from 245 to 257). Our main goal in these experiments (using LUBM data) is to demonstrate the soundness and completeness of the reasoning performed by LNN. In addition, we also want to report on preliminary evaluation of our initial exploration towards achieving noise robustness using LNN, i.e., specifically to explore LNN's ability to handle noisy axioms that result in ontological inconsistencies.

{\bf Dataset.} Lehigh University Benchmark (LUBM)~\cite{lubm} is a synthetic OWL reasoning dataset in university domain together with 14 benchmark queries. In our experiments we considered LUBM data generated for $1$ university. This has in total $102707$ facts expressed in triples format in addition to axioms expressing university domain ontology in OWL format. We pre-processed these facts and axioms to convert them into equivalent LNN constructs so they could be consumed directly by our LNN code.

{\bf Soundness and completeness.} The purpose of this experiment is to evaluate the soundness and completeness of the results generated by LNN for $14$ LUBM benchmark queries. The $14$ queries are originally designed to evaluate the accuracy of symbolic reasoners. Accurate answering of these queries require varying degrees of reasoning capabilities from shallow to deep. Thus, accurate answers for all these queries by LNN should assert its deep reasoning capabilities.

In our approach, we built a single LNN for the entire LUBM data, with logical nodes and interconnections corresponding to various predicate nodes, logical connectives and their interactions, as expressed through the axioms. This resulted in LNN with $257$ nodes. The facts expressed are assigned as initial groundings to the corresponding nodes. We then made this initial network to go through multiple passes of forward and backward inferences, to allow all the nodes to gather their inferred groundings together with bound values, through reasoning as per network structure and initial facts. It took $4$ bidirectional passes for the network to converge with respect to the node level inference groundings and the corresponding upper and lower bound values. Then for each query, we attach a small query-specific network to the converged LNN in order to infer its groundings through a single forward pass within that particular query network alone. The final groundings computed at the query nodes correspond to the answers for the queries.

In this experiment, we compared LNN with a few standard symbolic reasoners namely Stardog~\cite{stardog}, Virtuoso~\cite{virtuoso} and Blazegraph~\cite{blazegraph}. A detailed comparison of the results for each of the $14$ benchmark queries is given in Table \ref{lubmReasoning}. LNN achieves $100\%$ precision and recall, thus demonstrating its sound and complete reasoning capabilities on LUBM OWL reasoning task. Among symbolic reasoners, Stardog also achieves $100\%$, while Virtuoso and Blazegraph achieve respective average recall values of only $72\%$ and $78\%$ indicating their lack of complete reasoning capabilities in case of queries that need deep reasoning.

\begin{table}[htpb!]
	\centering
	\caption{Soundness and Completeness for LUBM 1 university showing output sizes for each query }
	\label{lubmReasoning}
	\begin{tabular}{c c c c c}
		\toprule
		\textbf{Query} & \textbf{Virtuoso} & \textbf{Blazegraph} & \textbf{Stardog} & \textbf{LNN}  \\ \midrule
		Q1 & 4 & 4 & 4 & 4 \\
		Q2 & 0 & 0  &0 & 0 \\
		Q3 &6  & 6 & 6 & 6\\
		Q4 & 34 & 34 & 34 & 34 \\
		Q5 & 146 & 719 & 719 & 719 \\ 
		Q6 & 5916 & 5916 & 7790 & 7790\\
		Q7 & 59 & 59 & 67 & 67 \\ 
		Q8 & 5916 & 5916 & 7790 & 7790\\
		Q9 & 103 & 103 & 208 & 208 \\
		Q10 & 0& 0& 4 & 4 \\
		Q11 & 224 & 224& 224& 224\\
		Q12 & 15&0&15&15\\
		Q13& 1&1&1&1\\
		Q14& 5916&5916&5916&5916\\ 
		\midrule
		Precision & 1.0 & 1.0 & 1.0 & 1.0 \\
		Recall & 0.72 & 0.78 & 1.0 & 1.0 \\
		\bottomrule
	\end{tabular}
\end{table}

{\bf Noise handling.} Noisy/incorrect axioms typically result in inconsistent ontology and hamper the ability of the reasoning systems to perform inference. For example, if an axiom makes an incorrect statement about the domain and range for a specific binary variable, that likely would result in inconsistency if there is conflict with domain and range inferred for that variable from the rest of the axioms. Symbolic reasoners typically fail to answer queries when faced with such ontological inconsistencies. Typically, they end up alerting the presence of inconsistency, prompting for correction before proceeding with the inference.

In this experiment we explore the potential use of LNN to handle such ontological inconsistencies. LNNs perform inference by propagating and aggregating truth value bounds at various nodes across the network as per node specific activation. As a result, ontological inconsistencies reflect themselves as truth value bound contradictions at nodes corresponding to incorrect axioms in the network. This is because evidences for truth value bounds computed through different routes of the network are likely to contradict each other when aggregated at the inconsistent nodes. In our preliminary exploration we evaluated the usefulness of this behavior towards identifying the source of inconsistency and to further perform self-correction through down-weighting those parts of the network causing inconsistency, thus effectively switch them off from taking part in further inference.

To evaluate this, we added an incorrect axiom to the original set of LUBM axioms and built LNN using them. The specific noisy axiom we considered makes domain (concept to which subject in the binary predicate belongs) and range (concept to which the object in the binary predicate belongs) incorrect for one of the binary predicates, i.e., it wrongly asserts domain of the subject variable and range of the object variable. For example, we took one of the binary predicates \texttt{degreeFrom} with domain and range as \texttt{Person} and \texttt{University} respectively, and modified range to \texttt{Person}. Note that, in addition to the noisy axiom, we also added an additional (correct) axiom to assert that \texttt{University} and \texttt{Person} concepts are disjoint. This was added to make sure Ontology becomes inconsistent with conflict arising from noisy axiom.

As explained above, during inference such inconsistency is expected to result in bound value contradictions at node corresponding to the incorrect axiom. Once node with contradictory bound values is identified, we down-weight that node to effectively switch that off from taking part in further inferences. The network is then allowed to converge through multiple passes of forward and backward inference (after down-weighting inconsistent node that caused bound value contradictions). The converged network is then used to answer all the $14$ benchmark queries. In our experimental run, we found that the LNN is able to successfully locate and down-weight the inconsistent node, converge further through forward and backward inference passes, and then answer all the $14$ queries with $100\%$ precision and recall. We currently experimented with simple noise axioms that do not require deep reasoning to identify conflicts. As a future work, we plan to build upon this idea to further demonstrate the effectiveness of LNN in identifying and self-correcting various types of ontological inconsistencies.

\subsection{Additional information on TPTP benchmark}
This section gives additional information on the {\bf TPTP benchmark} used in the {\bf Empirical evaluation} section of the main paper. The TPTP (Thousands of Problems for Theorem Provers)~\cite{sutcliffe2009tptp,trail} is a comprehensive library for Automated Theorem Proving (ATP) systems, with over 22K problems. It contains problems in 53 different domains, represented in 4 different formalisms. These are First-order Form (FOF), Clausal Normal Form (CNF), typed higher-order form (THF) and typed first-order form (TFF). There are a total of 22686 problems, 17535 ($77\%$) of which have equality. The FOF subset contains 8630 problems spanning 41 domains, 6811 ($78\%$) of which have equality. An example problem file is named CSR001+1.p, with CSR indicating the Common Sense Reasoning domain, 001 the problem number in this domain, $+$ indicating FOF format and then 1 after $+$ signifying the difficulty level of the problem. Instead of a $+$ symbol, THF problems are identified by $\wedge$, CNF by $-$ and TFF by $=$.

The CSR domain we evaluated on has a total of 1030 problems with 937 of these in FOF. The difficulty level of its problems range from 1 to 100 with majority of them below level 7. After filtering out problems with functions and equality, the 25 remaining all had a difficulty level of 1. From these 25 problems, the LNN was able to prove all theorems.

\end{document}